\documentclass[12pt,a4paper]{article}

\usepackage{jmlr2e}
\usepackage[utf8]{inputenc}
\usepackage{amsmath}
\usepackage{amsfonts}
\usepackage{amssymb}
\usepackage{graphicx}
\usepackage{hyperref}
\usepackage[left=2.5cm,right=2.5cm,top=2.5cm,bottom=2.5cm]{geometry}
\usepackage{tikz}
\usetikzlibrary{shapes,backgrounds}
\usepackage{caption}
\usepackage{subcaption}
\usepackage{placeins}
\usepackage{array}
\usepackage{cleveref}
\usepackage{color,soul}
\usepackage{makecell}

\newcommand{\PreserveBackslash}[1]{\let\temp=\\#1\let\\=\temp}
\newcolumntype{C}[1]{>{\PreserveBackslash\centering}p{#1}}
\newcolumntype{R}[1]{>{\PreserveBackslash\raggedleft}p{#1}}
\newcolumntype{L}[1]{>{\PreserveBackslash\raggedright}p{#1}}

\newcommand{\slug}{\hbox{\kern1.5pt\vrule width2.5pt height6pt depth1.5pt\kern1.5pt}}
\def\xskip{\hskip 7pt plus 3pt minus 4pt}
\newdimen\algindent
\newif\ifitempar \itempartrue 
\def\algindentset#1{\setbox0\hbox{{\bf #1.\kern.25em}}\algindent=\wd0\relax}
\def\algbegin #1 #2{\algindentset{#21}\alg #1 #2} 
\def\aalgbegin #1 #2{\algindentset{#211}\alg #1 #2} 
\def\alg#1(#2). {\medbreak 
  \noindent{\bf#1}({\it#2\/}).\xskip\ignorespaces}
\def\algstep#1.{\ifitempar\smallskip\noindent\else\itempartrue
  \hskip-\parindent\fi
  \hbox to\algindent{\bf\hfil #1.\kern.25em}%
  \hangindent=\algindent\hangafter=1\ignorespaces}

\usepackage{color}

\newtheorem{algo}{Algorithm}

\def\ci{\perp\!\!\!\perp}

\newcommand{\HT}{\mathcal{T}}

\newcommand{\CG}{\mathcal{C}}
\newcommand{\SG}{\mathcal{S}}
\newcommand{\KG}{\mathcal{K}}

\newcommand{\realz}[1]{\hat{\mathbf{#1}}}

\newcommand{\algvar}[1]{\textit{\textbf{#1}}}

\DeclareMathOperator{\Tr}{Tr}

\newcommand{\mfcffix}{\textsc{MFCF\_FIX}}
\newcommand{\mfcffixi}{\textsc{MFCF\_FIX\_ID}}
\newcommand{\mfcfvar}{\textsc{MFCF\_VAR}}
\newcommand{\mfcfvari}{\textsc{MFCF\_VAR\_ID}}
\newcommand{\glasso}{\textsc{GLASSO}}
\newcommand{\shrinkage}{\textsc{SHRINKAGE}}
\newcommand{\mfcf}{\textsc{MFCF}}

\ShortHeadings{Learning Clique Forests}{Previde Massara and Aste}
\firstpageno{1}

\begin{document}

\title{Learning Clique Forests}

\author{\name Guido Previde Massara \email guido.massara.12@ucl.ac.uk \\
       \name Tomaso Aste \email t.aste@ucl.ac.uk \\
       \addr Department of Computer Science\\
       University College London\\
       Gower Street, London, WC1E 6B\\
        UCL Centre for Blockchain Technologies, UCL, London, UK. \\
        Systemic Risk Centre, London School of Economics, London UK.
}

\editor{TBD}

\maketitle

\begin{abstract}

We propose a new algorithm for learning complex networks from data, named Maximally Filtered Clique Forest (\mfcf{}). The \mfcf{} is a preferential attachment scheme where new nodes are added to a growing forest of cliques by means of a \emph{clique expansion operator}. The clique expansion operator is topologically invariant and always transforms clique forests into clique forests. Nodes are added recursively and the selection of the next node and of the attachment points is driven by the increase in a given score (gain function). We show that the clique forests produced by the \mfcf{} are chordal. 
This makes an associated graphical model decomposable and it makes it easy to construct an associated sparse inverse covariance matrix.
The clique expansion operator can be customized so as to control the size of the cliques and the multiplicity of their separators. 
The algorithm is general, and can be constructed with gain functions suited to diverse applications. In this paper we provide examples for the problem of covariance selection for Gaussian graphical models and compare the results obtained by the \mfcf{} with the Graphical Lasso.

\end{abstract}


\section{Introduction}

In this paper we introduce a novel algorithm, the Maximally Filtered Clique Forest (\mfcf{}) which produces \emph{clique forests} representative of the structure of input data sets.  
Clique forests are acyclic graphs that span the set of the cliques of a base graph. The nodes of a clique forest are the maximal complete subsets of the graph (the cliques) and the edge, which we will call \emph{separator}, joining two cliques is made of the elements in the intersection of the two cliques\footnote{There is a further technical requirement,  the \emph{clique intersection property} that will be described in section \ref{sec:definitions}}. 

The clique forest, as a data structure, enjoys a number of useful properties.

\begin{enumerate}

    \item The graph underlying a clique forest is a \emph{chordal}, or \emph{triangulated} graph\footnote{See, for instance, the excellent reference \citet[Theorem 3.1]{blairpeython1992} for a proof. Note that the authors refer to \emph{clique trees}, rather than clique forests, with the understanding that all the arguments and demonstration can be repeated separately for all of the connected components of a clique forest.}.
    Chordal graphs are a particular case of \emph{perfect graphs} \citep{golumbic2004algorithmic}, for which there are known polynomial time solutions for many problems that are much harder for general graphs such as graph coloring, maximum clique, maximum independent set \citep{grotschel2012geometric,vandenberghe2015chordal}.

    \item The structure of the cliques and separators is such that the clique forest is also an example of an \emph{abstract simplicial complex} \citep{munkres2018elements}. We believe this fact (although not exploited in the current paper) might add the tools of Topological Data Analysis \citep{wasserman2018topological} to the analysis of networks modelled as clique forests.
    
    \item The maximum clique size allowed in the forest is in direct correspondence with the \emph{treewidth} of the graph and the treewidth is linked to the computational complexity of the most common inference tasks \citep{bach2001thin}. Limiting the treewidth (or equivalently the maximum size of a clique in the forest) is both a way to limit the complexity of inference and the number of parameters associated to a graphical model, see \citet{bach2001thin}. Decomposable models associated with clique forests are particularly suitable for inference \citep[Chap. 10]{koller2009probabilistic}, since in this case  inference can be carried out exactly and efficiently using variable elimination, belief propagation and the junction-tree algorithm \citep{shafer1990probability,koller2009probabilistic}. In the case of clique forests, the complexity of inference is defined by the size of the maximal cliques in the model, hence the interest in producing ``thin junction trees'' \citep{bach2001thin,chechetka2008efficient}, that are graphs where the maximal size of the cliques is small. Conversely, if the underlying graph is not a junction forest, exact inference is infeasible since it is NP-hard (see \cite{chandrasekaran2012complexity}) and has to be carried out using a number of approximate inference procedures based on sampling (e.g. Gibbs, Markov Chain Monte Carlo, importance sampling) or on deterministic approximate inference (variational approximation, loopy belief propagation, expectation propagation, see \citet{wainwrightjordan2008,koller2009probabilistic,barber2012bayesian}).

\end{enumerate}


When the problem is to infer the \emph{sparse} dependency structure of the dataset, the \mfcf{} clique-forest structure yields to a \emph{decomposable} Markov random field (MRF).
MRFs are \citep{lauritzen1996,2017AnRSA...4..365D,ScutariStrimmer2011} multivariate probability distributions associated with an undirected graph $G(V, E)$ (and are therefore also called undirected graphical models, UGM).  The vertices $v \in V$ are in a one to one relationship with a collection of random variables $X_v, v \in V$; the edges in $E$ join the conditionally dependent random variables. In MRFs the absence of an edge between two variables represents conditional independence. Structure learning is the process of inferring from the data the pattern of missing edges. 

The problem of finding the dependency structure that allows to carry out inference in a system of variables is, in general, NP-complete \citep{karger2001learning,bogdanov2008complexity,chickering1996learning,chickering1994learning}, and some simplifying assumptions are required in practice. 

When a clique forest is the underlying graph of a graphical model \citep{lauritzen1996} the cliques represent a completely interdependent set of random variables, and the separator is the conditioning set that mediates the interaction between two cliques. The graphical model is \emph{decomposable} and has many desirable properties: the maximum likelihood estimators of the parameters are known in closed form  for a number of probability distributions \citep{lauritzen1996}; the calculation of the probability density can be expressed in closed form without the need to perform a demanding calculation of the partition function; and, finally, the clique forest readily supports the application of the \emph{junction tree} algorithm (see, for instance, \citet{koller2009probabilistic}).
As it is common in many types of probabilistic graphical models
\citep{koller2009probabilistic} we make the assumption that the global 
dependence structure is \emph{sparse} and therefore the number of edges in the graph $G$ is small with respect to the total possible number of edges, $|E| \sim \mathcal{O}(|V|) \ll \frac{1}{2} |V| (|V| - 1)$. 
The presence of noise and the low number of samples makes it difficult to establish whether a conditional dependency between two variables is significant, and therefore it should be represented with an edge in the MRF, or whether it is due to spurious effects, in which case the associated edge should not be present in the MRF. Structure learning in the context of this paper is the problem of identifying and estimating significant structural links while filtering spurious interactions. There are many upsides to the use of sparse graphical models \citep{hastie2015statistical}:

\begin{enumerate}
    \item from the end user point of view they display meaningful relationships and provide interpretable predictive models; 
    
    \item  from a computational point of view they require less storage and computing resources than the dense counterparts and are tractable for exact inference or sampling; 
    
    \item If the dependency network is sparse, then the cliques in the network are of small sizes with respect to the total number of variables and such a joint probability distribution can be calibrated from data even when the number of samples is much lower than the number of variables overcoming the ``curse of dimensionality'' issue \citep{lauritzen1996,barfuss2016parsimonious,Gross2018}. 
    
    \item finally, from a modelling point of view there is the consideration that goes under the name of ``bet on sparsity'' \citep[p. 2]{hastie2015statistical}: ``Use a procedure that does well in sparse problems, since no procedure does well in dense problems''. In other words, when the number of samples is low compared to the number of parameters and the problem is not sparse, no estimation of a large number of parameters is feasible (see also \citet{barfuss2016parsimonious}).
\end{enumerate}

In this paper we provide examples for decomposable Gaussian graphical models constructed on  \mfcf{} addressing the problem of covariance selection and we compare the results with the $L_1$-norm regularization via Graphical Lasso. 
The same machinery introduced for  Gaussian graphical models can be directly used for the more general covariance selection problem. 
Indeed, \mfcf{} can be used as $L_0$-seminorm regularization tools  which generate sparse inverse covariances with direct application to modeling with the whole multivariate elliptical family distribution  \citep{aste2020topological}. We chose the problem of covariance selection, and specifically a comparison with the Glasso, because of the combination of extensive and successful practical use, exhaustive literature and availability of reliable software implementations; therefore the comparison between Glasso and \mfcf{} can benefit from the many different contexts, papers, and software.


\subsection{Contributions}

Structure learning has been studied extensively, however, to the best of our knowledge, the emphasis has been so far on testing  pairwise relationships represented by a single edge, rather than on the detection of higher order topological structures \footnote{Some exceptions are based on frequent itemsets \citep{huang2002constructing} and t-cherry trees \citep{szantai2013discovering}.}. 
In this paper we propose a methodology we name \emph{Maximally Filtered Clique Forest} (\mfcf{})  that exploits the equivalence between 
decomposable graphs and clique forests (see \citet[Par. 2.2.3]{lauritzen1996} and \citet[Th. 4.12]{koller2009probabilistic}) 
and models the clique forest directly, rather than the underlying collection of edges. 

The \mfcf{} works by recursively adding vertices to the cliques of a clique forest using a local topological move, \emph{clique expansion}, that involves a clique and a vertex (see Section \ref{sec:clique:expansion}); at each step the algorithm adds the vertex whose addition maximises the gain of a given score function. The vertex can be added to the clique, effectively extending it, or can be added to a subset of the clique, creating two cliques separated by the subset; consistently with the terminology of graphical models, we call such a subset the \emph{separator}. 
The algorithm is generic, in the sense that it does not rely on a specific score function but it assumes that there is a function that takes in input a clique and a vertex and returns a gain in score and the resulting separator, but, in principle, it can take into account the existing graph structure (the assumption of chordality for the underlying graph is extremely helpful in controlling the  dependency structure).
In this paper we provide some examples related to graphical models and covariance selection problem and we specify what requirements must be fulfilled by a score function to fit into the \mfcf{} algorithm. 
The \mfcf{} has broad applicability beyond graphical modeling as a general methodology to construct clique forests.

\subsection{Organisation}

The structure of this paper is as follows: 
in Section  \ref{sec:literature:review} we briefly review some essential background literature and introduce notations and definitions. 
In Section \ref{sec:methodology} we describe the clique expansion operator, the ideas underlying the gain function and the \mfcf{} algorithm. 
In Section \ref{sec:experiments} we describe a specific application to covariance selection models and compare the results against 
alternative methods. 
Section \ref{sec:conclusions} provides discussion and conclusions.
We provide a vast Appendix (Section \ref{sec:appendix}) containing some additional theorems that are useful, if not strictly required for the problem at hand, as well as a detailed description of the experiments performed and some extra results.

\section{Background literature, notations and definitions} \label{sec:literature:review} 

\subsection{Information filtering with networks: a brief review}

Let us briefly review four approaches to associate, from data, a meaningful network structure to a multivariate system of random variables.
Specifically here we briefly introduce five important topics:  
(1) structure learning algorithms for graphical models;
(2) sparse graphical models and covariance selection; 
(3) information filtering networks;
(4) the Triangulated Maximally Filtered Graph algorithm, which the present paper generalises;
(5) other miscellaneous approaches.

\subsubsection{Structure Learning in Graphical Models} 
General approaches to structure learning in Graphical Models can be classified into 
three main categories \citep{zhou2011structure,lauritzen2012,ScutariStrimmer2011,koller2009probabilistic}: 
\emph{score based}, \emph{constraint based}, and \emph{Bayesian methods}.
Let us here briefly introduce and comment  them one by one.

\paragraph{Score based} algorithms perform structure learning by detecting 
edges or other structures that optimize some global function such as likelihood, 
Kullback-Leibler divergence \citep{kullback1951information}, Bayesian Information 
Criterion (BIC) \citep{schwarz1978estimating}, Minimum Description Length 
\citep{rissanen1978modeling} or the likelihood ratio test statistics 
\citep{petitjean2013scaling}. 
In general, the identification of the structure that optimises the score function results in a  
difficult combinatorial optimization problem \citep[Ch. 20]{koller2009probabilistic} and some 
sort of greedy approach should be implemented to produce a sequence of steps that optimize a limited space of solutions.

One of the leading methods in the score based sparse representation of joint 
probability distributions is the \emph{Chow-Liu} trees (CLT). In the 
original paper, \citet{ChowLiu} proposed a mechanism where a $p$-order 
discrete distribution is approximated by the product of a number of second-order 
distributions. The second order distributions are specified by using the minimum 
spanning tree algorithm (MST) by maximizing the total mutual information of the $p-1$ 
edges. For the MST this maximization coincides with the maximum-likelihood estimation for a tree inference structure. 
\citet{ku1969approximating} extended the CLT for discrete probability distributions by allowing the use of marginal probabilities of order greater than two. They used the Kullback-Leibler divergence as a scoring function. 
To control the complexity of the resulting model they did not use full marginal distributions, but the marginals of small sets of variables (that is probabilities defined on small cliques of variables). 
\cite{huang2002constructing}  built a Chow-Liu Tree and 
successively refined it by adding frequent-itemset in large vertices (Large 
vertex Chow-Liu Tree – LNCLT), which are set of vertices that occur frequently 
together. 
The statistical learning theory of Chow-Liu trees is presented in 
detail in \citet{koskilectures}.

When the goal is to learn a decomposable model, the score based algorithms need 
to fulfil the additional chordality constraint: 
in this area there are a number of methods that efficiently explore the 
graphical structure (directed, in the case of Bayesian networks, or undirected in 
the case of log-linear or multivariate Gaussian models) with the help of suitable graph 
algorithms based on the manipulation of data structures representing junction trees or clique graphs 
\citep{giudici1999decomposable,deshpande2001efficient,petitjean2013scaling}.
\citet{DBLP:journals/corr/KovacsS13} describe a ``pruning'' approach for 
multivariate discrete distributions which removes links 
iteratively refining a junction-tree, optimising the total correlation 
(``information content''). 
\citet{szantai2013discovering} developed an algorithm specialised for a particular clique tree (the 
``t-cherry'' junction tree) and used it to approximate a multivariate discrete 
distribution. 
 To the best of our knowledge all structure learning methods\footnote{With the exception of \citet{szantai2013discovering}.} for decomposable models 
deal with the chordality constraint on an edge-by-edge basis and, differently from the proposed \mfcf{} approach, do not model the clique forest as an aggregation of cliques.

\paragraph{Constraint based} algorithms often start from a complete 
model and adopt a \textit{backward selection} approach by testing the 
independence of vertices conditioned on subsets of the remaining vertices (e.g. 
in the Spirtes-Glymour-Scheines (SGS) and Peter-Clark (PC) 
\citep{spirtes2000causation,zhou2011structure} algorithms) and removing edges 
associated to vertices that are conditionally independent. Conversely \textit{forward selection} algorithms start from a sparse 
model and add edges associated to vertices that are discovered to be 
conditionally dependent. An hybrid model is the Grow-Shrinkage (GS) \ algorithm 
where a number of candidate edges is added to the model (the ``grow'' step) in a 
forward selection phase and subsequently reduced using a backward selection step 
(the ``shrinkage'' step) \citep{margaritis2000bayesian,zhou2011structure}. The 
complexity of checking a large number of conditional independence statements 
makes these methods unsuitable for graphs with a large number of vertices. 
Furthermore, aside from the complexity of measuring conditional independence, 
these methods do not generally optimize a global function, such as likelihood or 
the Akaike Information Criterion \citep{akaike1974new,akaike1998information}, but 
they rather try to exhaustively test all the conditional independence 
properties of a set of data and therefore are difficult to use in a 
probabilistic framework.

\paragraph{Bayesian methods} consider the presence or absence of an edge in 
the inference network structure  as a random variable. More precisely 
\citep{Madigan1995,eaton2012bayesian,scutari2013prior}
the likelihood of a model result from the product of a discrete probability 
distribution over the space of all graphs ($P(G)$) and a continuous distribution 
of the random variables conditional on the graph structure $P(X,G)=P(X 
\vert G)P(G)$.  
Applying Bayes rule, the posterior probability of a graph can be taken as $P(G\vert X) \propto P(X \vert G) P(G)$. Therefore, if the probability on all graphs is the same optimising the probability of the graph is equivalent to optimising the marginal likelihood $P(X \vert G)$ (see for instance \citet{Madigan1995,eaton2012bayesian} and references therein). 
Usually the probability over the graph structure is indeed taken as uniform, 
meaning that each graph is equally probable. 
Conversely, in the present perspective, the \mfcf{} approach introduces a non-uniform prior which is recursively updated after every clique expansion.

\subsubsection{Sparse Graphical Gaussian Models}

When the variables $X_v$ follow a multivariate Gaussian distribution the conditional independence of two variables $X_a$ and $X_b$ is directly reflected
in the inverse of the covariance matrix (or \emph{precision matrix}) $J$: $X_a \ci X_b \mid (X_r, r \in V \setminus \left\lbrace a,b \right\rbrace)
\Longleftrightarrow J_{ab} = 0$. The pattern of missing edges corresponds to conditionally independent variables and it is exactly the same as the zero elements in the precision matrix. The problem of estimating a precision matrix is called ``Covariance selection'' \citep{dempster1972} and 
the associated graphical models are called Gaussian Graphical Models (GGM) or Gaussian Markov Random Fields (GMRF).
In the field of GMRFs there are several approaches 
\citep{d2008first,banerjee2008model,banerjee2006convex} that exploit the link 
between edges and zero-elements of the precision matrix: the general idea is to 
maximise the likelihood of the multivariate normal distribution (which can be 
expressed in terms of the sparse inverse covariance matrix) penalised by a 
non-decreasing function of the number and weight of the non-zero elements in the 
precision matrix.
Specifically, ridge regression uses a $\ell_2$-norm penalty; instead the 
\textit{lasso} method  \citep{Tibshirani1996}  uses an $\ell_1$-norm penalty and 
the elastic-net approach uses a convex combination of $\ell_2$ and $\ell_1$ 
penalties  \citep{zou2005regularization}. These 
approaches are among the best performing regularization methodologies presently 
available. The $\ell_1$-norm penalty term favours solutions with parameters with 
zero value leading to models with sparse inverse covariances. Sparsity is 
controlled by regularization parameters $\lambda_{ij} > 0$; the larger the value 
of the parameters the more sparse the solution becomes.
This approach is extremely popular and around the original idea a large 
body of literature has been published with several novel algorithmic techniques 
that are continuously advancing this method 
\citep{Tibshirani1996,MeinshausenBuehlmann2006,BanerjeeEtAl2008,RavikumarEtAl2011,HsiehEtAl2011,OztoprakEtAl2012} 
 among these the popular implementation \emph{Glasso} (Graphical-lasso)  
\citep{FriedmanEtAl2008} which uses lasso to compute sparse graphical models.
However, Glasso methods are computationally intensive and the selection of the 
penalisation parameter is often difficult to justify as it does not have an immediate link to the data.
Within this framework, the \mfcf{} approach introduced in this paper is a $L_0$-seminorm regularization. 
Compared to Glasso it is computationally more efficient and it has the advantage of returning the maximum likelihood solution of the sparse problem instead of a shrunken solution \citep{aste2020topological}.

\subsubsection{Information Filtering Networks} \label{sec:ifn}
With Information Filtering Networks we refer to a set of network-based approaches aimed at extracting information about the dependency structure from correlation (or, more broadly, similarity) matrices by retaining a sparse network from the most relevant correlations only.
This methodology originated in the Econophysics community where the interest in modelling dependence stems from studies on the spectral properties of the correlation 
matrix of financial portfolios \citep{laloux1999noise}, focusing on cleaning 
methodologies inspired by Random Matrix Theory (RMT) \citep{bun2017cleaning}. An 
alternative approach has been to use tools from topology to investigate the 
structure of financial markets. One seminal idea \citep{mantegna1999} was to 
use the Minimum Spanning Tree algorithm to build a hierarchical tree structure that retains the largest correlations. 
In further developments other topological constraints have been investigated, 
notably imposing the planarity of the filtered network \citep{tumminelloetal2005} 
and studying hyperbolic embeddings \citep{asteetal2005,tumminelloetal2007}. These 
methodologies have enabled the study of several properties of financial portfolios 
 with applications to portfolio diversification 
\citep{pozzi2013spread,musmeci2015risk}, clustering 
\citep{musmeci2015riskcorr,song2012hierarchical} and dynamics of correlation structure in markets
\citep{aste2010correlation}.

\subsubsection{Triangulated Maximally Filtered Graphs}
In \cite{TMFG} we proposed a greedy algorithm that builds a Triangulated Maximally Filtered Graph by recursively adding vertices to a $k$-width tree while minimising a given score function (which in a particular probabilistic application is the Kullbak-Leibler divergence). In \cite{barfuss2016parsimonious} this general algorithm was applied to the approximation of multivariate normal distributions by using the multivariate normal Kullback-Leibler divergence as a scoring function; in the same paper some basic results on Gaussian Markov random fields are used to provide applications to financial portfolio modelling. The TMFG produces planar and chordal networks by restricting the size of the cliques and clique-intersections and by constraining the topology of the clique tree. 
\citet{christensen2018network} carry out a comparison of Glasso and information filtering networks based on TMFG from the point of view of psychometric networks showing that TMFG have better interpretability.
The work in the present paper is a radical generalisation of the TMFG algorithm where the size of the clique is no longer a constraint but an adjustable parameter that can be tuned to the data, and the size and use of separators is driven by the gain in a score function.

\subsubsection{Other approaches}
Of course, there are many more approaches to the problem of structure learning, and we can only hope to hint at a few alternative methods that present less similarities with ours. One reasearch direction is towards spectral sparsification using the concept of \emph{effective resistance} \citep{batson2013spectral,spielman2011graph}. Other approaches are formulated, and make extensive use of, numeric tensor analysis to produce reduces-rank representation of large data represented as tensors \citep{cichocki2018tensor,cichocki2014era,cichocki2015tensor,kolda2009tensor}. Finally, many methods in topological data analysis aim at producing a simplified view of clouds of data points where the persistent topological features represent a version of the data where the noise has been removed \citep{wang2016object,edelsbrunner2000topological,edelsbrunner2008persistent,wasserman2018topological,guskov2001topological}.

\subsection{Notation and Definitions} \label{sec:definitions}

In  this section we provide the minimum amount of notation and results required by the paper, \citet[Chap. 2]{lauritzen1996} provides the full notation and proofs.

\subsubsection{Graphs and Chordal graphs}
A Markov random field is a collection $X_i, i \in \left\lbrace 1, \dots, p \right\rbrace$ of random variables together with
an undirected graph $G=G(V,E)$ with a vertex set $V$ and a edge set $E$. 
We call \emph{vertices} the elements of $V$ and \emph{edges} the elements of $E$.
The variables $X_i$ are in a one-to-one correspondence with the vertices  $V$ so that $|V| = p$. $E$ is a collection of unordered pairs of elements of $V$, e.g. $e \in E \Rightarrow e 
= \{a, b\}$ with $a \in V, b \in V$. 
A graph is chordal when every cycle of length $\ge 4$ has a ``chord'', that is an edge between two 
non-adjacent vertices in the cycle.  Graphical models whose underlying graph is chordal are 
 \emph{decomposable models} \citep[chap. 2]{lauritzen1996}.
 
We  use the notation $\mathbf{X} = \{X_1, \dots, X_p\}$ to indicate the multivariate random variable. When we examine the realisations of the  variables $X_i$, we will use the hat symbol on top of the variable (e.g. $\hat{\mathbf{X}}$) and will use $n$ to refer to the total number of realisations, so that: $\hat{X}^t_i, i \in \left\lbrace 1, \dots, p \right\rbrace, t \in \left\lbrace 1, \dots, n \right\rbrace$ is the $t$-th realisation of the $i$-th variable, and $\mathbf{\hat{X}}^t = \{\hat{X}_1^t, \dots, \hat{X}_p^t\}$. For a given subset $C \subset \{1, \dots, p\}$ we use the shorthand notation $\mathbf{\hat{X}}_C = \{\hat{X}_i, i \in C\}$.

\subsubsection{Clique Forest} 
A \emph{clique} is a maximal complete subset of vertices, that is, where any pair of vertices is joined by an edge, and it is not included in a larger complete set. We indicate with $\CG_G$ (or $\CG$ when there is no ambiguity) the set of cliques of $G$. 

Similarly we introduce the separator set,  $\SG$, made of intersections of the elements of $\CG$. For two cliques $C_a$ and $C_b$ the separator is $S_{ab} = C_a \cap C_b$.
Given a graph $G(V,E)$ with set of cliques $\CG = \left\lbrace C_1, C_2, \dots, C_m \right\rbrace$ of $G$, we say that there is separator between $C_i$ and $C_j$ if $C_i \cap C_j$ is not empty and we denote the set of separators with $\mathcal{S} = \left\lbrace S_1, S_2, 
\dots S_k \right\rbrace$. 
The graph $\KG(\CG, \SG)$ where the vertices are the cliques of $G$ and the edges are the not-empty intersections of the cliques is called the \emph{clique graph  or intersection graph} of $G(V,E)$. 

\begin{definition}[Clique intersection property and clique forest] \label{def:CIP}
A \emph{clique forest} for $G(V,E)$ is a clique graph with no cycles that includes all the cliques of $G$ and that additionally fulfils the \emph{clique-intersection} property (CIP)\citep{blairpeython1992}:

For any two cliques $C_i, C_j \in \CG$ the set $C_i \cap C_j$ is contained in every clique on the path between $C_i$ and $C_j$ in the tree. We will denote a clique forest with $\mathcal{F}(\CG, \SG)$.

When a graph $G=G(V,E)$ has a clique forest we say that the graph has the \emph{CF-property}.

\end{definition}

\begin{example}[Clique forest]
Figure \ref{fig:clique:forest} shows an example of a chordal graph with its associated representation as a clique forest.
(A single clique tree in this case.)
\end{example}

\begin{example}[The four-cycle does not have the CF-property]
The four cycle graph, shown in Figure \ref{fig:clique:forest:4c} is a negative example. Any tree spanning the nodes of the clique graph does not enjoy the running intersection property. The picture should provide an intuitive appreciation of the definition of a clique forest; there should be no ``holes'' in the graph whose boundary is made of cliques.
\end{example}

\begin{figure}[h]
\centering
\begin{subfigure}{.45\textwidth} 
\centering
  \includegraphics[width=.99\linewidth]{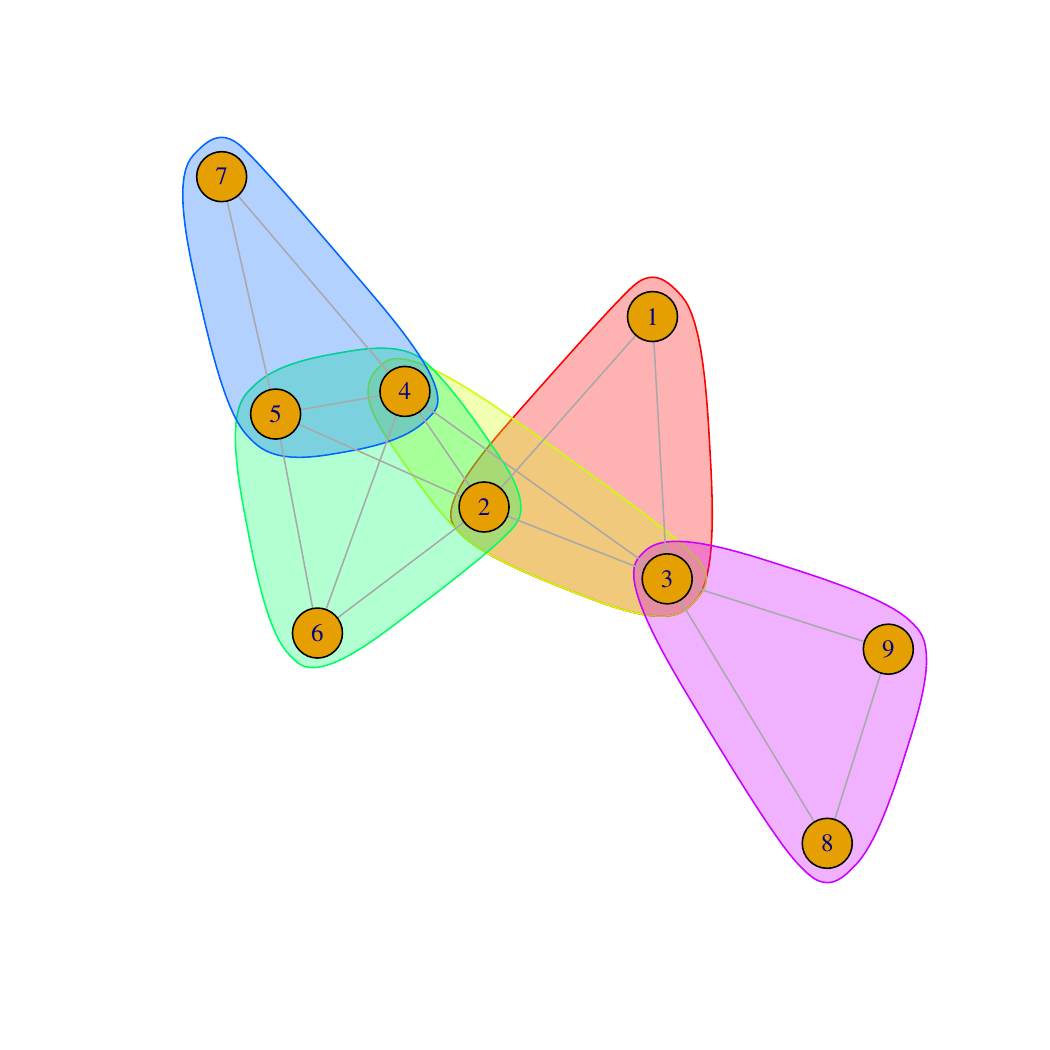}
  \caption[Graph with cliques highlighted]{A chordal graph with the \\ maximal cliques highlighted.}
  \label{fig:cg:1}
\end{subfigure}%
\begin{subfigure}{.47\textwidth}
\centering
  \includegraphics[width=.99\linewidth]{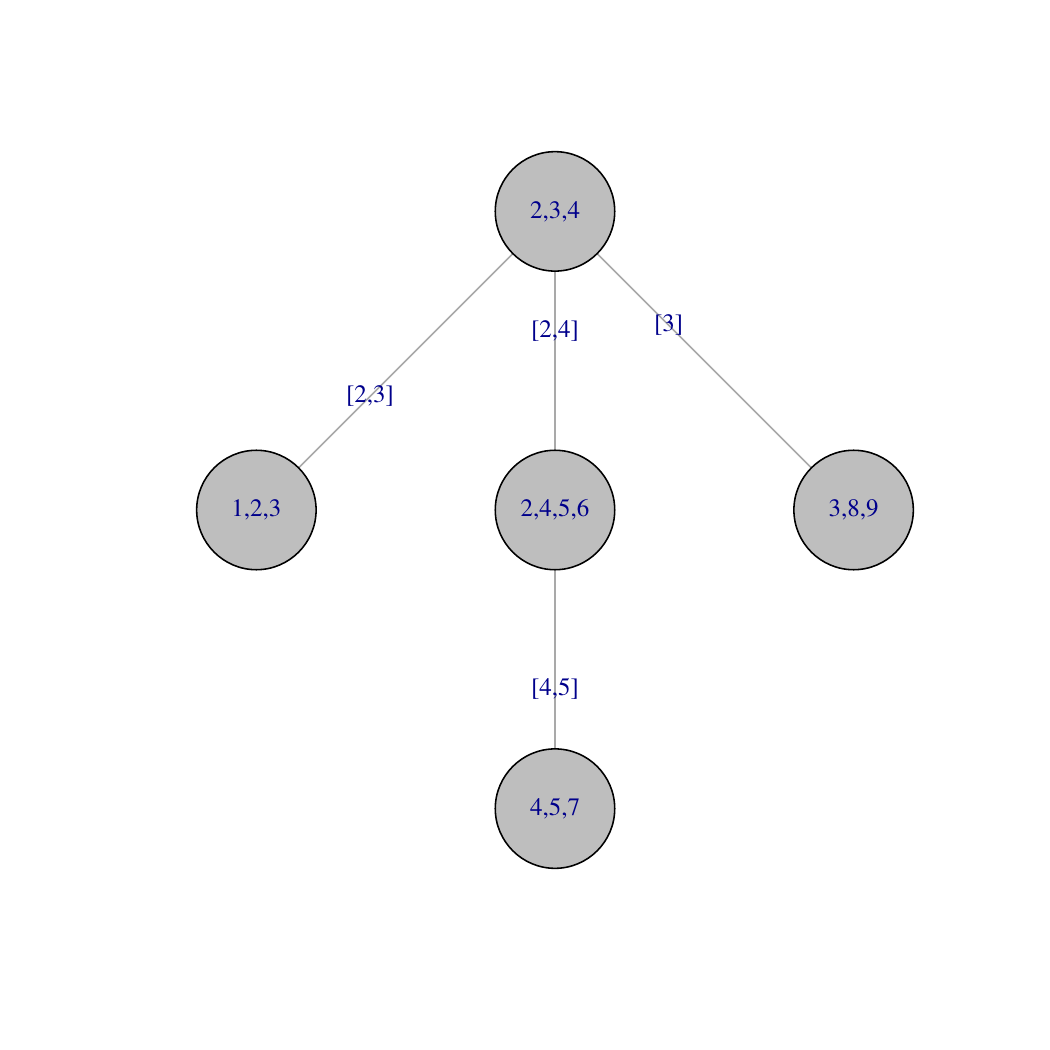}
  \caption[Clique Tree Representation]{The same graph represented as a clique tree. The edges of the tree are labelled with the elements of the intersection.}
  \label{fig:cg:2}
\end{subfigure}

\caption[Chordal graph with clique forest]{Illustration of the relationship between a chordal graph and the associated clique forest.}
\label{fig:clique:forest}
\end{figure}

\begin{figure}[h]
\centering
\begin{subfigure}{.47\textwidth} 
\centering
  \includegraphics[width=.97\linewidth]{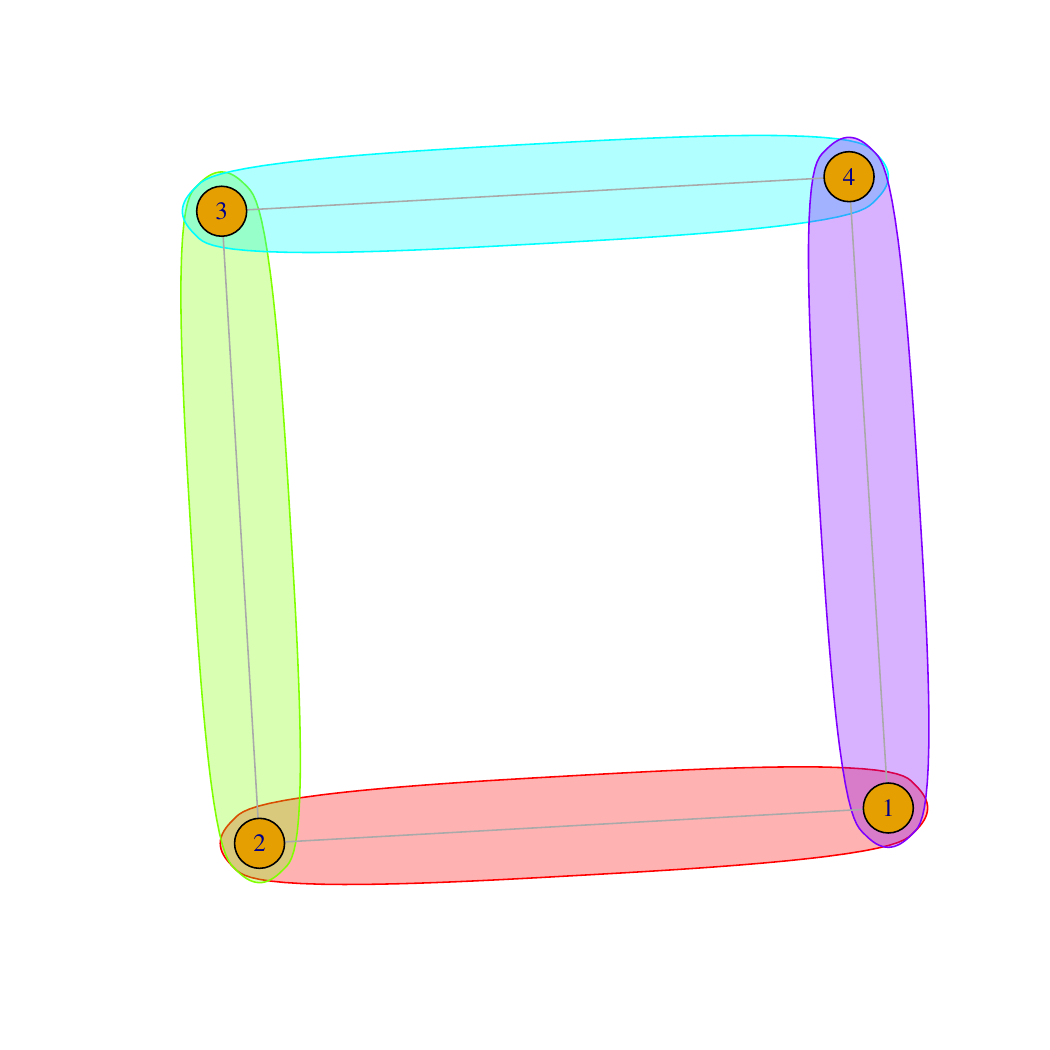}
  \caption[Four cycle graph with cliques highlighted]{The four cycle graph with the \\ cliques highlighted}
  \label{fig:cg:4c1}
\end{subfigure}%
\begin{subfigure}{.47\textwidth}
\centering
  \includegraphics[width=.97\linewidth]{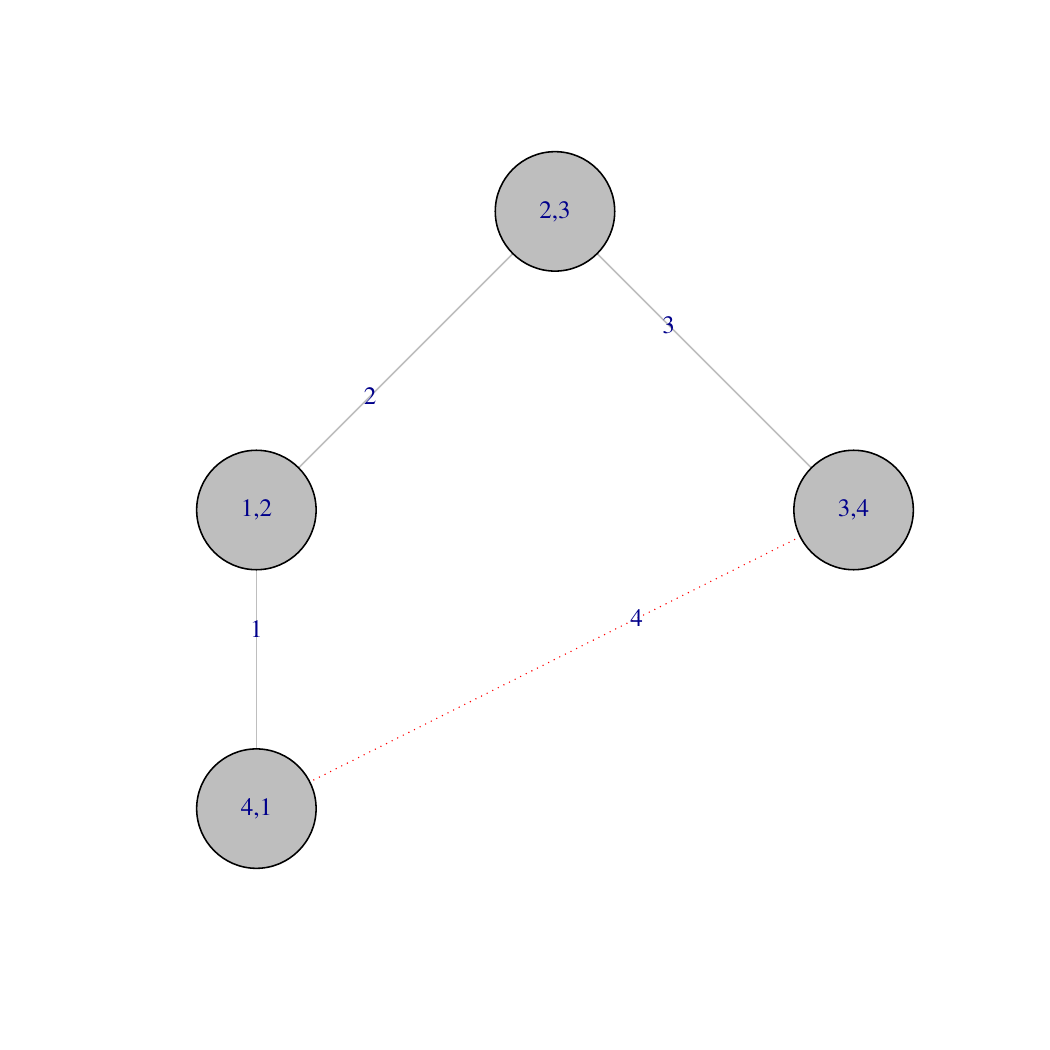}
  \caption[Clique Tree Representation of four cycle]{The four cycle as a clique tree (grey edges only). The edges of the tree are labelled with the elements of the intersection. This tree does not have the clique intersection property: the intersection between $\{4,1\}$ and $\{3,4\}$ ($4$, highlighted with the dotted red edge) is not included in the path between the two nodes along the tree.}
  \label{fig:cg:4c2}
\end{subfigure}

\caption[The four cycle graph does not have a clique forest]{The four cycle graph is the simplest graph that does not have a clique forest.}
\label{fig:clique:forest:4c}
\end{figure}

\begin{definition}[CF-property, CF-invariance]
We will say that a graph that has a clique forest enjoys the \emph{CF-property}. We will also say that any transformation of a graph that conserves the CF-property is \emph{CF-invariant}
\end{definition}

The following theorem \ref{th:chordal:cf} establishes the equivalence between chordality and CF-property for a graph. We will not make use of this theorem directly as we will provide direct demonstrations of the CF-property (see Corollary \ref{th:mfcf:cf:property}) and of the chordality for the networks produced by the \mfcf{} (see Theorem \ref{th:mfcf:chordality}).

\begin{theorem} \label{th:chordal:cf}
A graph $G$ has a clique forest $\mathcal{T}(\CG, \SG)$ if and only if $G$ is chordal. 
\end{theorem}

\begin{proof}
See \citet[Th. 3.1]{blairpeython1992} or \citet[Th. 4.3]{koller2009probabilistic} for a proof. 
\end{proof}

\FloatBarrier

\subsubsection{Decomposable Graphical Models}


\begin{definition}
Let $A \subset V$, $B \subset V$, $C \subset V$ be three subsets of vertices of a graph $G$. 
We say that the triple $(A, B, C)$ decomposes  $G(V,E)$ if $V = A \cup B \cup C$ and the following conditions hold: $C$ separates $A$ and $B$; $C$ is a complete graph.
\end{definition}

Let us assume that $(A,B,C)$ decomposes $G(V,E)$. 
We associate the random variables $\mathbf{X_A} = \mathbf{X}_{v \in A}$, $\mathbf{X}_B = \mathbf{X}_{v \in B}$, $\mathbf{X}_c = \mathbf{X}_{v \in C}$ if $\mathbf{X}_A$ and $\mathbf{X}_B$ are independent given $\mathbf{X}_C$:
\begin{equation}
\mathbf{X}_A \ci \mathbf{X}_B | \mathbf{X}_C
\end{equation}
we state that the joint distribution of the variables $\mathbf{X}_{v \in V}$ has the \emph{Global Markov Property} 

In this case, where $C = A \cap B$,  (as depicted in Figure \ref{fig:decomposable:graph} in the Appendix) the joint probability distribution can be shown 
\citep[Prop. 3.17]{lauritzen1996} to follow a generalisation of Bayes' rule\footnote{Provided that the probability measure is positive in the product space.}:

\begin{equation} \label{eq:bayes:like}
P(\mathbf{X}_A, \mathbf{X}_B) = \frac{P(\mathbf{X}_A) P(\mathbf{X}_B)}{P(\mathbf{X}_{A \cap B})}
\end{equation}

Equation \ref{eq:bayes:like} and Figure \ref{fig:decomposable:graph} are equivalent  to a tree decomposition of the set $V$ where the tree has two vertices $A$ and $B$ joined by the edge $C = A \cap B$. 

In case $A$ and $B$ are in turn decomposable (or fully connected cliques) it can be shown \cite[Chap. 3]{lauritzen1996} that the joint probability distribution can be further recursively factored into finer decomposable components until we get to the clique set ($\CG$) and  separator set ($\SG$) of $G$:

\begin{equation} \label{eq:lauritzen:formula}
P(\mathbf{X}) = \frac{\prod_{c \in\mathcal{C}} P(\mathbf{X}_c)} {\prod_{s \in\mathcal{S}} P(\mathbf{X}_s)}
\end{equation}

\begin{remark}
Note that in case two sets of variables are disjoint (unconditionally independent) the corresponding separator is the empty set and the tree structure is a forest.
\end{remark}

\begin{remark}
It is possible that a separator appears more than once in a clique forest, for example when it separates more than two cliques. In such a case in our notation the separator set reports the separator more than once, so that the separator multiplicity is automatically taken into account. 
\end{remark}

\section{Methodology} \label{sec:methodology}
\subsection{Generation of Clique Forests: the clique expansion operator} \label{sec:clique:expansion}
In this section we describe the tool for building clique forests that is originally introduced with this paper. 

\paragraph{} The \emph{clique expansion} procedure takes as input a clique $C_a$ and an isolated vertex $v$ and produces a new clique $C_b$ and a separator $S$. 
In general, $S \subseteq C_a$ contains the vertices that produce the largest gain when combined with $v$ and the new clique is $C_b = S \cup v$. 

Figure \ref{fig:edge:expansion} describes the \emph{clique expansion} operation. 
The inputs of the operation are: the clique $C_1 = \left\lbrace 1, 2, 3, 4 \right\rbrace$ and the isolated vertex $\left\lbrace 5 \right\rbrace$. 
Figure \ref{fig:hee:2} shows the output of the clique expansion in the general case: two cliques $C_1$ and $C_2 = \left\lbrace 1, 2, 5 \right\rbrace$ and the separator $S = C_1 \cap C_2 = \left\lbrace 1, 2 \right\rbrace$. 
There are two special cases. 
\begin{enumerate}
\item If none of the elements of $C_a$ have a strong relationship with $v$, then vertex $v$ is not attached and $C_b = \left\lbrace v \right\rbrace$ and $S = \emptyset$ (see Figure \ref{fig:hee:4}).
\item If all the elements in $C_a$ have a strong relationship with the isolated vertex then $C_a$ is extended with $v$ ($C_a \leftarrow C_a \cup v$) and both the separator $S$ and the new clique $C_b$ are empty(see Figure \ref{fig:hee:6}).
\end{enumerate}
We note that $S$, being a subset of a complete graph, is complete, and also that it separates $C_1 \setminus S$ and $C_2 \setminus S$.

\begin{figure}
\centering
\begin{subfigure}{.45\textwidth} 
\centering
  \includegraphics[width=.9\linewidth]{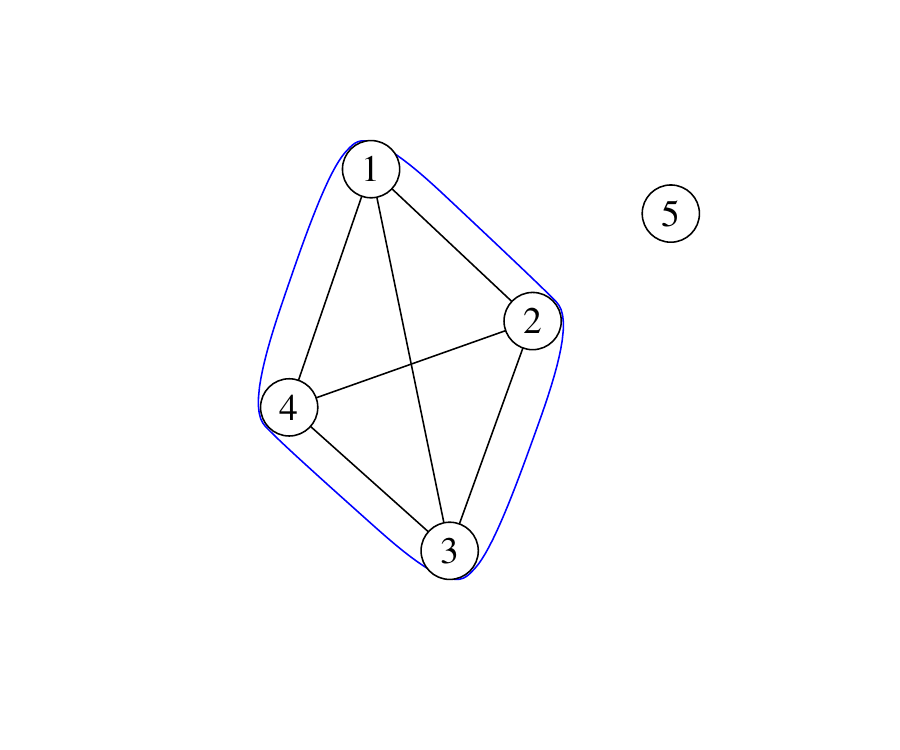}
  \caption{Before clique expansion (general case) \\
  $P(X=x \mid G_a) = \phi_{1234}(X_{1},X_{2},X_{3},X_{4}) \phi_{5}(X_{5})$}
  \label{fig:hee:1}
\end{subfigure}%
\begin{subfigure}{.05\textwidth}
 \centering
	\begin{tikzpicture}
		\draw[->, thick] (0,4) -- (1,4);
	\end{tikzpicture}
\end{subfigure}
\begin{subfigure}{.45\textwidth}
   \centering
  \includegraphics[width=.9\linewidth]{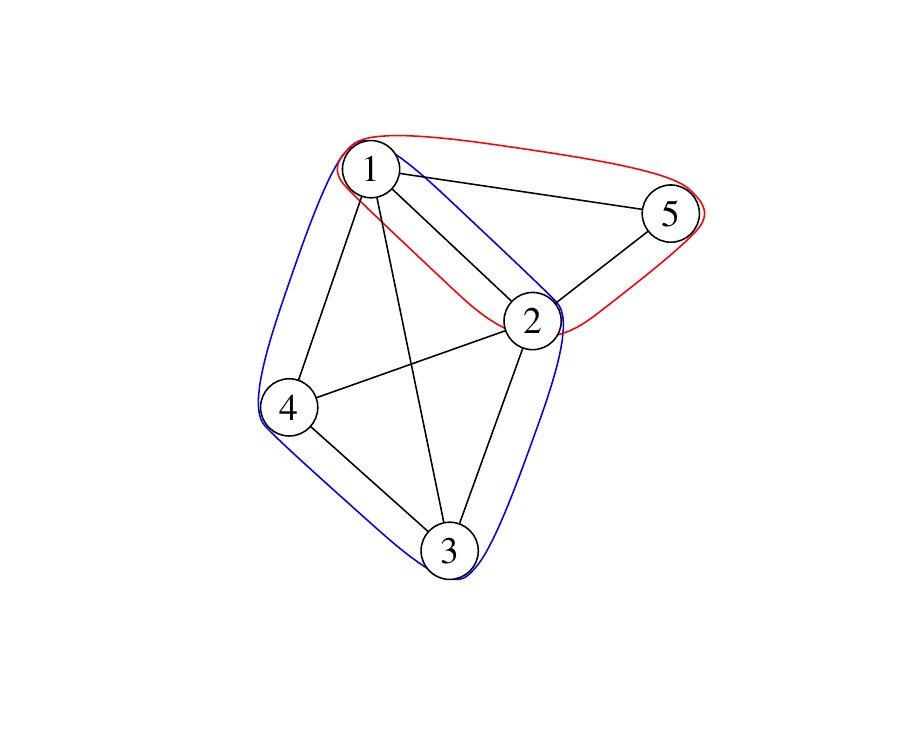}
  \caption{After clique expansion (general case), \\
  $P(X=x \mid G_b) =\frac{\phi_{1234}(X_{1},X_{2},X_{3},X_{4}) 
\phi_{125}(X_{1},X_{2},X_{5})}{\phi_{12}(X_{1},X_{2})}$\\
$S = \left\lbrace 1,2 \right\rbrace$}
  \label{fig:hee:2}
\end{subfigure}

\begin{subfigure}{.45\textwidth} 
\centering
  \includegraphics[width=.9\linewidth]{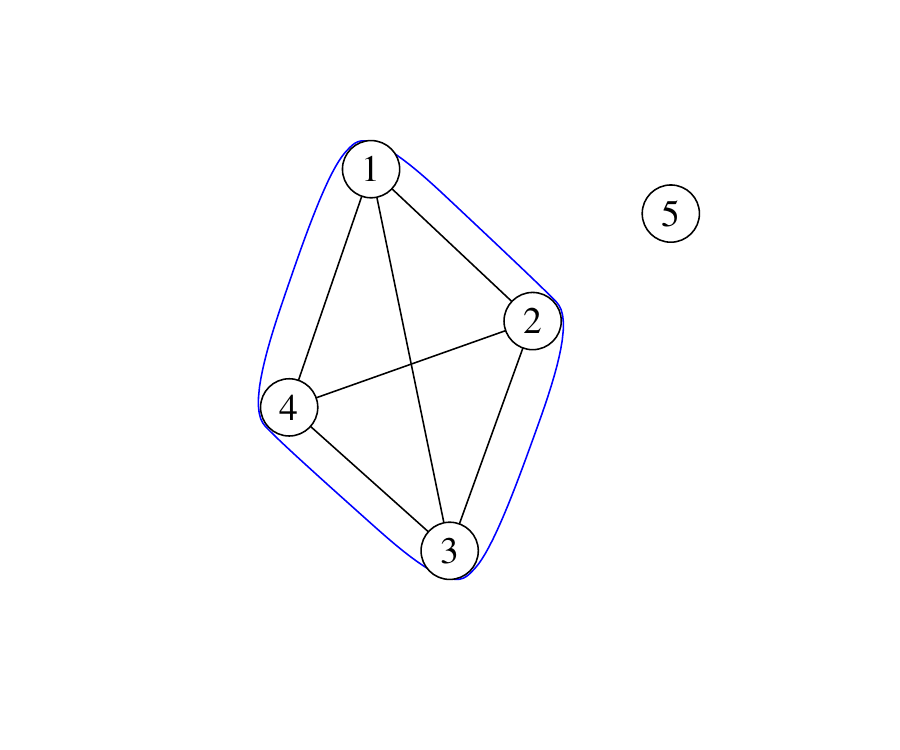}
  \caption{Before clique expansion (isolated vertex case) \\
  $P(X=x \mid G_a) = \phi_{1234}(X_{1},X_{2},X_{3},X_{4}) \phi_{5}(X_{5})$}
  \label{fig:hee:3}
\end{subfigure}%
\begin{subfigure}{.05\textwidth}
 \centering
	\begin{tikzpicture}
		\draw[->, thick] (0,4) -- (1,4);
	\end{tikzpicture}
\end{subfigure}
\begin{subfigure}{.45\textwidth}
   \centering
  \includegraphics[width=.9\linewidth]{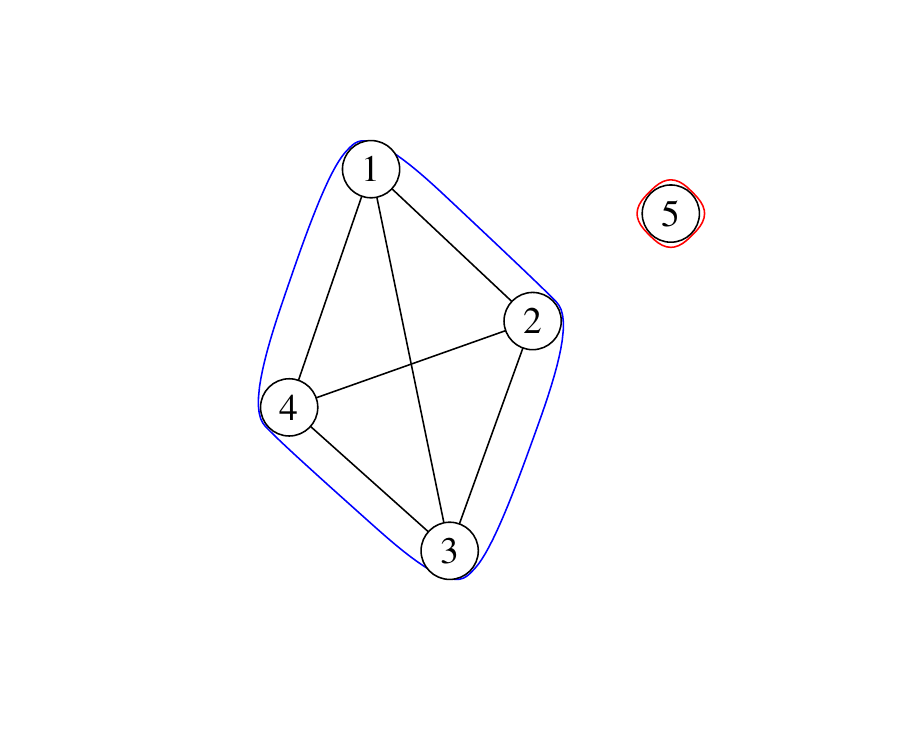}
  \caption{After clique expansion (isolated vertex case), \\
  $P(X=x \mid G_b) =\phi_{1234}(X_{1},X_{2},X_{3},X_{4}) \phi_{5}(X_{5})$\\
$S = \emptyset$}
  \label{fig:hee:4}
\end{subfigure}

\begin{subfigure}{.45\textwidth} 
\centering
  \includegraphics[width=.9\linewidth]{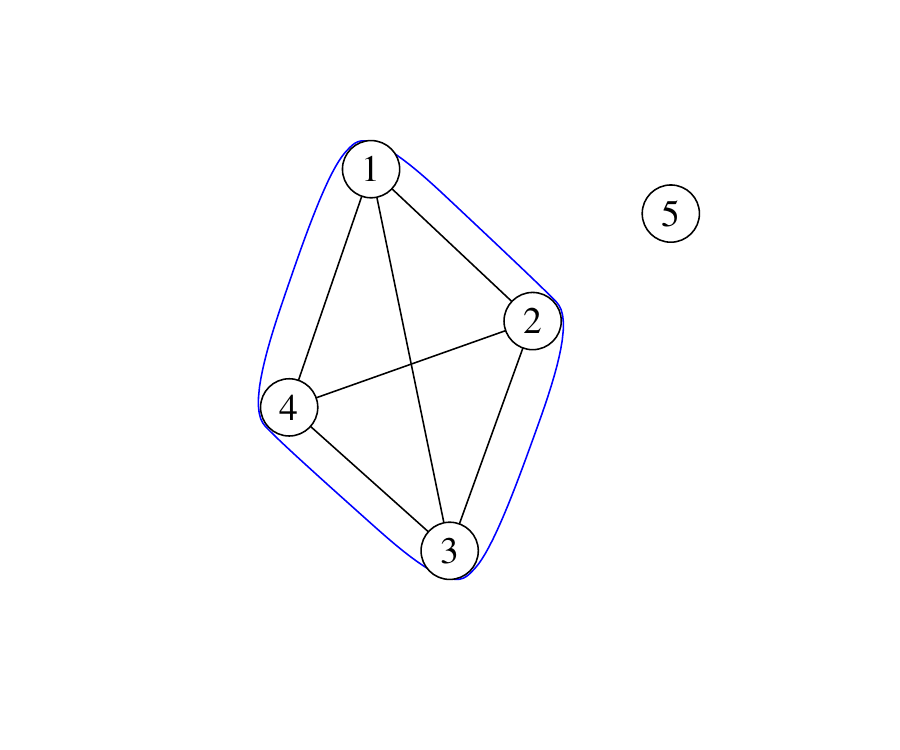}
  \caption{Before clique expansion (full expansion) \\
  $P(X=x \mid G_a) = \phi_{1234}(X_{1},X_{2},X_{3},X_{4}) \phi_{5}(X_{5})$}
  \label{fig:hee:5}
\end{subfigure}%
\begin{subfigure}{.05\textwidth}
 \centering
	\begin{tikzpicture}
		\draw[->, thick] (0,4) -- (1,4);
	\end{tikzpicture}
\end{subfigure}
\begin{subfigure}{.45\textwidth}
   \centering
  \includegraphics[width=.9\linewidth]{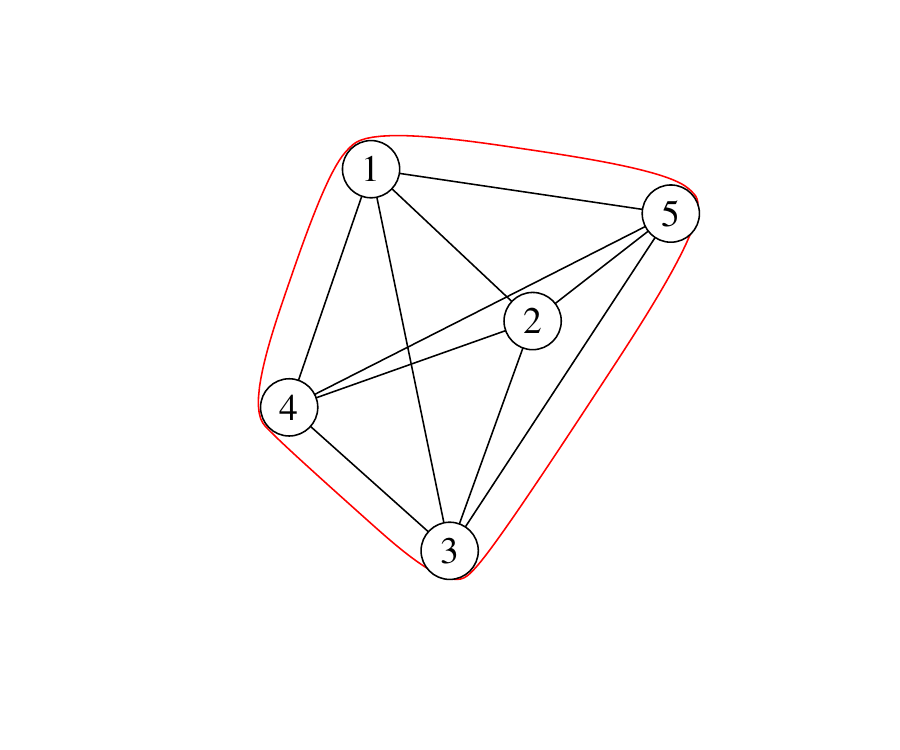}
  \caption{After clique expansion (full expansion), \\
  $P(X=x \mid G_b) =\phi_{12345}(X_{1},X_{2},X_{3},X_{4}, X_{5})$\\
$S = \emptyset$}
  \label{fig:hee:6}
\end{subfigure}

\caption{Illustration of the clique expansion operator. 
If $S$ is a proper subset of  $C_1$ the operation produces two cliques and a separator; if $S = \emptyset$ the result produces two 
disconnected cliques $C_1$ and $C_2 = \left\lbrace 5 \right\rbrace$ and the separator is the empty set if $S = C_1$ the operation purely expands the original clique with the new vertex and does not introduce a new clique or separator.}
\label{fig:edge:expansion}
\end{figure}

\begin{theorem}[CF-invariance of the clique expansion operator]\label{th:cf:inv:ceo}
 The clique expansion operator is CF-invariant.
\end{theorem}
\begin{proof}
Let $\mathcal{F}_{n-1}$ be a clique forest: in particular, this means that $\mathcal{F}_{n-1}$ has the CIP (as per Definition \ref{def:CIP}). 
We build a new clique forest $\mathcal{F}_{n}$ applying the clique expansion operator and adding a new node $n$ to a clique $C_m \in \mathcal{F}_{n-1}$. We want to show that $\mathcal{F}_{n}$ has the CIP. \\
Let us denote with $C_n$ the new clique created by the clique expansion operator.

In the general case (figure \ref{fig:hee:2}) a new clique  $C_n \subset C_m \cup \left\lbrace n \right\rbrace$ is introduced and the new node $n$ does not belong to any clique intersection, therefore any node belonging to the intersection of $C_n$ with any other clique $C_{m'}$ must belong to $C_m$ as well. Due to the inductive hypothesis, given any other clique $C_{m'} \in \mathcal{F}_{n-1}$, there is a path $C_m - C_{m_1} - \cdots - C_{m_k} - C_{m'}$ that contains all the nodes in $C_m \cap C_{m'}$. The path $C_n - C_m - C_{m_1} - \cdots - C_{m_k} - C_{m'}$ therefore connects $C_n$ to $C_{m'}$ in $\mathcal{F}_{n}$ and fulfills the CIP.

In the case (figure \ref{fig:hee:4}) when $n$ is added as new disconnected clique there is nothing to prove as the CIP follows from the inductive hypothesis.

In the case of full expansion (figure \ref{fig:hee:6}) the clique $C_n = C_m \cup \left\lbrace n \right\rbrace$ replaces $C_m$ and the new node $n$ does not belong to any clique intersection, therefore any node belonging to the intersection of $C_n$ with any other clique $C_{m'}$ must belong to $C_m$ as well. Due to the inductive hypothesis, given any other clique $C_{m'} \in \mathcal{F}_{n-1}$, there is a path $C_m - C_{m_1} - \cdots - C_{m_k} - C_{m'}$ that contains all the nodes in $C_m \cap C_{m'}$. The path $C_n - C_{m_1} - \cdots - C_{m_k} - C_{m'}$, obtained by replacing $C_m$ with $C_n$, \emph{a fortiori} connects $C_n$ to $C_{m'}$ in $\mathcal{F}_{n}$ and fulfills the CIP.

\end{proof}

\begin{theorem}[Decomposibility invariance of the clique expansion operator]
Let $\mathcal{F}_{n-1}$ be a clique forest underlying a graphical model with associated random variables $X_1, ..., X_{n-1}$ with underlying density $\phi_{X_{1},...X_{n-1}}(X_{1},...X_{n-1})$. If we add one random variable $X_n$ to the graphical model and use the clique expansion operator to attach the associated node to the graph, the density of the updated graphical model is $$\phi_{X_{1},...X_{n}}(X_{1},...X_{n}) = \phi_{X_{1},...X_{n-1}}(X_{1},...X_{n-1}) \frac{\phi_{X_{C_n}}(X_{C_n})} {\phi_{X_S}(X_S) }$$ where $C_n$ and $S$ are respectively the clique and separator introduced by the clique expansion operator.

\end{theorem}
\begin{proof}
The proof is immediate since the separator $S$ is complete by construction and separates the graph. In cases where the network is fully built using the clique expansion operator decomposition shown in Equation \ref{eq:lauritzen:formula}
follows by induction.\end{proof}

\subsection{The MFCF algorithm} \label{sec:mfcf:algorithm}
The \mfcf{} algorithm that we introduce in this paper is a generalisation to clique forests of Prim's
minimum spanning tree algorithm (\cite{prim1957}). Prim's algorithm 
constructs the Minimum Spanning Tree tree starting with an arbitrary vertex and 
adds the closest (i.e. with minimal edge weight) unconnected vertex. In case the
graph is not connected, for instance when some subsets are at infinite distance, 
the algorithm can be applied to the distinct connected components to produce a 
minimum spanning forest. 
\citep[Ch. 4]{blairpeython1992} illustrates the connection between Prim's algorithm and the maximum cardinality search algorithm (MCS, \citet{tarjan1976maximum}), used to test graph chordality.

In our generalisation the vertices are replaced by cliques and we optimise a 
gain function, rather than an edge weight; depending on the function we might 
look for the minimum (e.g. minimum cost) or the maximum (e.g. maximum 
gain). The algorithm starts by selecting one or more cliques and, at each 
stage, one of the unconnected vertices is added to the clique forest by 
performing an edge expansion. The vertex is chosen so as to optimise the scoring 
function. The initial clique(s) can be chosen with an heuristic (as in the 
variant of the algorithm presented here) or they could be assumed as given from 
previous knowledge or expert judgement. For instance  in 
genetic regulatory networks there is interest in incorporating certain 
topological \emph{motifs} that are known to appear frequently in this kind of 
networks \citep{fiori2012topology}. In this case the cliques provided must be a clique forest.

\begin{algo}
MFCF: Builds a clique forest with given  clique size range.
\end{algo}

\noindent\textbf{Description:} 
Builds a clique forest by applying the clique expansion operator repeatedly until there are no more outstanding vertices. For performance reasons, the algorithm maintains a gain table that holds the possible scores for any combination of cliques already added to the forest and the outstanding vertices.

\noindent\textbf{Input:}
\begin{itemize}
	\itemsep0em
	\item[] \algvar{W} [mandatory]: Either a data matrix with $n$ rows of $p$-variate observations (e.g. time series of stock market returns) or a $p$-by-$p$ similarity matrix (e.g. correlation matrix of returns).
	\item[] \algvar{gain\_function} [mandatory]: a function that calculates the gain of a clique expansion based on \algvar{W}.
	\item[] \algvar{max\_cl\_size} [optional]: size of maximal clique (default value: 4, range $= [2, p]$).
	\item[] \algvar{min\_cl\_size} [optional]: size of minimal clique (default value: 4, range $= [1, max\_clique\_size]$).
	\item[] \algvar{reuse\_separators} [optional]: whether to use separators more than once (default value: TRUE).
	\item[] \algvar{C\textsubscript{I}}, \algvar{S\textsubscript{I}}  [optional]: Initial list of cliques and separators.
\end{itemize}

\noindent\textbf{Output:}
\begin{itemize}
	\itemsep0em
	\item[] \algvar{cliques}: list of cliques of the clique forest, ordered as a perfect sequence of sets.
	\item[] \algvar{separators}: list of separators of the clique forest.
	\item[] \algvar{tree}: topological description of the clique forest.
\end{itemize}

\noindent\textbf{Algorithm:}
\begin{itemize}

	\item[S1.] [Initialize] 
	$p \leftarrow $ number of variables. 
	$cliques \leftarrow \emptyset$.
	$separators \leftarrow \emptyset$.
	$outstanding\_vertices \leftarrow \left\lbrace 1, \dots, p \right\rbrace$
	
	\item[S2.] [Initialize list of cliques]\\
	- If $C_I$ is  empty \\
	\phantom{- } - $C_1 \leftarrow FirstClique()$ \\ 
	\phantom{- } - $cliques \leftarrow C_1$ \\
	\phantom{- } - $outstanding\_vertices \leftarrow outstanding\_vertices \setminus \left\lbrace v, v \in C_1 \right\rbrace $.\\
	- Else \\
	\phantom{- } - $cliques \leftarrow C_I$, \\
	\phantom{- } - $separators \leftarrow S_I$, \\
	\phantom{- } - $outstanding\_vertices \leftarrow outstanding\_vertices \setminus  \left\lbrace v, v \in C_I \right\rbrace $.

	\item[S3.] [Init Gain Table].
	For every $v \in oustanding\_vertices$ and every $C \in cliques$, calculate score and optimal separator for $C$ and $v$ and add to gain table.
	
	\item[S4.] [Check for termination]. 
	If $ outstanding\_vertices = \emptyset$ then return $cliques$, $separators$, $tree$.
	
	\item[S5.] [Get best possible expansion]. 
	Select from gain table the clique $C_a \in cliques$, separator $S \subset C_a$ and vertex $v \in outstanding\_vertices$ corresponding to the entry with the highest score.
	
	\item[S6.] [Create new clique / separator]. \\
	 - If $S$ is a proper subset of $C_a$ then \\
	 \phantom{- } $C_b \leftarrow S \cup v$\\
	 \phantom{- } $cliques \leftarrow cliques \cup C_b$\\
	 \phantom{- } $separators \leftarrow separators \cup S$.\\
	 - If $S = C_a$ (extension without new separators) then \\
	 \phantom{- } $C_a \leftarrow C_a \cup v$. \\
	 - If $S = \emptyset$ (disconnected cliques) then \\
	 \phantom{- } $C_b \leftarrow v$\\
	 \phantom{- }  $cliques \leftarrow cliques \cup C_b$.	
	
	\item[S7.] [Update outstanding vertices, tree]. 
	$outstanding\_vertices = outstanding\_vertices \setminus v$. \\
	Set the edge between $C_a$ and $C_b$ to be the separator: $tree(C_a, C_b) \leftarrow S$
	
	 \item[S8.] [Update gain table]. 
	 Delete from gain table all entries where the vertex is $v$. \\
	 Add to gain table entries with gains for $C_b$.\\
	 If $reuse\_separators$ is false, delete from the gain table where the separator is $S$.\\
	 Update gains for $C_a$.
	
     \item[S9.] [Close loop]. 	
	 Return to [S4.].\\
	  
\end{itemize}

\begin{remark}
The function $FirstClique()$ provides an estimate of the best first clique. It can be obtained by starting with a clique made of the two vertices (or a larger clique) with the largest gain and growing it using the clique expansion operator until it reaches the minimum size required. 
\end{remark}

The  \mfcf{} algorithm is a radical extension of the TMFG algorithm proposed in  \cite{TMFG} which build a planar graph  applying a special case of clique expansion operator ($T_2$ in that paper) with the following constraints: (a) the maximum clique size is 4, (b) the minimum clique size is also 4, and (c) a separator are triangles and they can be used only once. 
Note that the MST is also retrievable with the \mfcf{} algorithm by setting  the maximum clique size to 2, the minimum clique size also at 2, and separators can be used more than once.
Interestingly, there is another class of networks, which lies between the MST and the TMFG, which is constituted of  triangular cliques and edge separators and has not been explored yet in the literature. 


A straightforward application of Theorem \ref{th:cf:inv:ceo} shows that the cliques produced by the \mfcf{} are clique forests.

\begin{corollary}[The \mfcf{} produces clique forests] \label{th:mfcf:cf:property}
The networks produced by the \mfcf{} have the CF-property.
\end{corollary}
\begin{proof}
This results from the fact that the \mfcf{} applies the clique expansion operator, which is CF-invariant by virtue of Theorem \ref{th:cf:inv:ceo}.
\end{proof}

\subsection{Gain Functions} \label{sec:gain:function}

An advantage of the formulation in terms of clique trees is that we can use any 
multivariate function as a scoring function. Specifically, we aim for a 
function that, given a set of random variables $\mathbf{X}_c$ produces a score which is a 
measure of the strength of association between the variables, taking into 
account local interactions through the separator.

\begin{definition}
The \emph{gain} is the increase in score that results from a clique expansion. 
\end{definition}

We  also want to require the score to have a  validation procedure, so that we can test significance. 
Significance can be validated both through a parametric statistical test under some assumptions or non-parametrically via permutation test. Further, cross-validation can be implemented by computing the gain on the validation sample using the function estimated on the train sample.
In many cases the contribution to the score is made up of a positive contribution due to the introduction of a new clique and a negative correction 
from the separator. 
The meaning of the negative correction is sometimes related to double counting and sometimes related to conditioning. 

Let us here exemplify the above by introducing two gain functions: one based on scores from a similarity matrix, and a second one based on log-likelihood with an explicit formulation for the gaussian case.

\subsubsection{Gain Function from Similarity Matrix} 
As discussed in \citet{TMFG} there are applications where it is required to build a network that maximises the sum of the weights of a 
similarity matrix subject to some constraint. 
Examples are the correlation networks mentioned in Section \ref{sec:ifn} \citep{mantegna1999,Eurodollar02,InterestRates05,tumminelloetal2005,asteetal2005,tumminelloetal2007,aste2010correlation,song2012hierarchical,pozzi2013spread,musmeci2015risk,musmeci2015riskcorr}.

Let us define a symmetric matrix of weights $W$, where $w_{ij}$ quantifies the ``similarity'' of elements $i$ and $j$ and $w_{i,i}=0$. If $C$ is a subset of the row indices of $W$ we define $Score(C) = \sum_{i \in C, j \in C} W_{ij}$.


The gain function returns the best available separator that, joined 
with a vertex, gives the highest possible sum of the weights.
In this case the total score is the sum of the weights of the cliques minus the 
sum of the weights of the separators. The total score is given by 
\begin{equation}
Score = \sum_{c \in \mathcal{C}} \sum_{i \in c, j \in c} W_{ij} - \sum_{s \in 
\mathcal{S}} \sum_{i \in s, j \in s} W_{ij}
\end{equation}
When we perform a clique expansion and introduce a new clique $\tilde{c}$ and 
a new separator $\tilde{s}$ the corresponding gain in score is:
\begin{equation}
G(\tilde{c}, \tilde{s})  = \sum_{i \in \tilde{c}, j \in \tilde{c}}W_{ij} - \sum_{i \in 
\tilde{s}, j \in \tilde{s}}W_{ij}
\end{equation}
In the special case when the clique expansion results in the extension of a previous clique, such that $\tilde{C} = C \cup v$, the gain is the difference in score between the new clique and the old one (the separator is obviously zero):
\begin{equation}
G(\tilde{c}, c)  = \sum_{i \in \tilde{c}, j \in \tilde{c}}W_{ij} - \sum_{i \in 
c, j \in c}W_{ij} = \sum_{i \in \tilde{c}}W_{iv} 
\end{equation}
One might also add a form of validation to this gain function and add only the edges with weights that are significantly larger than zero and exceed a 
given threshold either in-sample or off-sample (cross validation).

\subsubsection{Gain function from log-likelihood} Equation 
\ref{eq:lauritzen:formula} is a likelihood for a given realisation $\mathbf{X}=\realz{x}$:

\begin{equation} \label{eq:lauritzen:formula:L}
\mathcal{L}(\mathbf{X} = \realz{x} \vert \left\lbrace c \in \mathcal{C}\right\rbrace, 
\left\lbrace s \in \mathcal{S}\right\rbrace) = \frac{\prod_{c \in\mathcal{C}} 
P(\mathbf{X}_c = \realz{x}_c)} {\prod_{s \in\mathcal{S}} P(\mathbf{X}_s = \realz{x}_s)}
\end{equation}

and accordingly the log-likelihood is:

\begin{equation} \label{eq:lauritzen:formula:l}
\ell(\mathbf{X}=\realz{x} \vert \left\lbrace c \in \mathcal{C}\right\rbrace, \left\lbrace s \in 
\mathcal{S}\right\rbrace) = 
\sum_{c \in\mathcal{C}} \log P(\mathbf{X}_c=\realz{x}_c) - \sum_{s \in\mathcal{S}} \log 
P(\mathbf{X}_s=\realz{x}_s)
\end{equation}

When we add a new clique $\tilde{c}$ and a new separator $\tilde{s}$ the gain in log-likelihood is:

\begin{equation}
G(\tilde{c}, \tilde{s}) = \log P(\mathbf{X}_{\tilde{c}} = \realz{x}_{\tilde{c}}) -  \log 
P(\mathbf{X}_{\tilde{s}}=\realz{x}_{\tilde{s}})\;.
\end{equation}

Instead, when we add a new clique $\tilde{c}$ by expanding an existing one $c$ the 
gain in log-likelihood is:

\begin{equation}
G(\tilde{c}, c) = \log P(\mathbf{X}_{\tilde{c}} = \realz{x}_{\tilde{c}}) -  \log 
P(\mathbf{X}_{c}=\realz{x}_{c})\;.
\end{equation}

It is possible to add a significance test to this gain function  since the model with the additional clique and separator are nested and the difference in log-likelihood is one-half of the \emph{deviance} \citep{Wasserman:2010:SCC:1965575}. 
Under some relatively mild assumptions the deviance is asymptotically distributed as a chi-squared variable with $k$ degrees of freedom, where $k$ is the number of edges added to the model with the clique expansion \citep[Ch. 5.2.2]{lauritzen1996}.
Other possible siginificance tests could be a cross-validation on a different set or an information criteria such as AIC or BIC (\cite{akaike1974new,schwarz1978estimating}).

When a test statistic is available it is conceivable to use the $p$-value as a gain function. The intuitive meaning is to build a network where the links of greatest statistical significance are added first.

From a Bayesian perspective, here we are  maximizing posterior probability $P(G|X)$ updating recursively the graph $G$ with the clique expansion move. 
In this framework, one starts with a model which assumes independent variables except for the initial clique $C_I$. Then, the gain table allows to chose the clique expansion operation which maximizes posterior probability by including one extra variable into the dependency structure. 
Given the greedy, recursive nature of the algorithm there is no guarantee to end at the global maxima. Nonetheless, this is an effective way to  estimate inference structures with high posterior probabilities solving efficiently a problem which is otherwise NP-complete.

\subsubsection{Gain Function from log-likelihood for the Multivariate Normal Distribution} \label{sec:gaussian:models}
In the important specific case of a $p$-variate normal distribution the log-likelihood function for a given clique forest structure $\HT$ can be written, using Equation \ref{eq:lauritzen:formula:l}:

\begin{multline} \label{eq:lauritzen:formula:l:gaussian}
\ell(\mathbf{X}
=\realz{x} \vert \left\lbrace c \in \mathcal{C}\right\rbrace, \left\lbrace s \in \mathcal{S}\right\rbrace) 
= - \frac p2 \ln(2\pi) \\ 
 +  \frac{1}{2}  \sum_{c \in\mathcal{C}} \left(\ln \vert J_c \vert - (\realz{x}_c - \mu_c)^t J_c (\realz{x}_c -\mu_c) \right) \\
  -  \frac{1}{2}  \sum_{s \in\mathcal{S}} \left( \ln \vert J_s \vert - (\realz{x}_s - \mu_s)^t J_s
(\realz{x}_s -\mu_s) \right) \\ 
 =  - \frac p2 \ln(2\pi) 
 + \frac{1}{2}  \sum_{c \in\mathcal{C}} \left( \ln \vert J_c \vert - \Tr(\hat{\Sigma}_c J_c) \right) 
  - \frac{1}{2}  \sum_{s \in\mathcal{S}} \left( \ln \vert J_s \vert - \Tr(\hat{\Sigma}_s J_s) \right)
\end{multline}

where 

\begin{enumerate}
	\item $\hat{\Sigma}_c$ (resp. $\hat{\Sigma}_s$) is the sample covariance matrix of the variables $\mathbf{X}_c$ (resp. $\mathbf{X}_s$) , and 
	\item $J_s$ (resp. $J_c$) is the inverse covariance matrix (precision matrix). 
\end{enumerate}

For a given clique forest structure the  likelihood is maximised by $J_c = \hat{J}_c = \hat{\Sigma}^{-1}_c$ and $J_s = \hat{J}_s = \hat{\Sigma}^{-1}_s$. In this case we have  $\sum_{c \in\mathcal{C}} \Tr(\hat{\Sigma}_c \hat{J}_c)- \sum_{s \in\mathcal{S}} \Tr(\hat{\Sigma}_s \hat{J}_s) = p$ and do not change with the application of the clique expansion operator and therefore Equation \ref{eq:lauritzen:formula:l:gaussian} can be simplified to:

\begin{equation}
\ell(\mathbf{X} = \realz{x} \vert \left\lbrace c \in \mathcal{C}\right\rbrace, \left\lbrace s \in 
\mathcal{S}\right\rbrace) =
-\frac p2 \ln(2\pi) + \sum_{c \in\mathcal{C}}  \frac{1}{2} \ln \vert \hat{J}_c \vert  
  -\sum_{s \in\mathcal{S}}  \frac{1}{2} \ln \vert \hat{J}_s \vert + \frac p2
\end{equation}

where the maximum likelihood estimations of the matrices $\hat{J}_c$ and $\hat{J_s}$ depend on the observations $\realz{x}$ for both the structure and their values. 

When we perform a clique expansion and introduce a new clique $\tilde{c}$ and 
a new separator $\tilde{s}$ the corresponding gain in score is:

\begin{equation} \label{equation:delta:l1:gaussian}
G(\tilde{c}, \tilde{s}) = 
  \frac{1}{2} \left(  \ln \vert {\hat J}_{\tilde{c}} \vert  -
  \ln\vert \hat J_{\tilde{s}} \vert \right)    =   \frac{1}{2} \left(  - \ln \vert \hat \Sigma_{\tilde{c}} \vert  +
  \ln \vert \hat \Sigma_{\tilde{s}} \vert \right)
\end{equation}

Instead, when a clique is expanded ($\tilde{c} = c \cup v$) the corresponding gain in score can be expressed as:
 
\begin{equation} \label{equation:delta:l2:gaussian}
G(\tilde{c}, c) =
  \frac{1}{2} \left(  \ln \vert \hat J_{\tilde{c}} \vert  -
  \ln\vert \hat J_{c} \vert \right) 
   =   \frac{1}{2} \left(  - \ln \vert \hat \Sigma_{\tilde{c}} \vert  +
  \ln \vert \hat \Sigma_{c} \vert \right)
\end{equation}

Note that this can be interpreted as an increase in likelihood or as a decrease 
in \emph{entropy}. In this case beside the asymptotic tests of the  
log-likelihood ratio, there are also several small sample tests that work in the 
``big data'' cases where $p  \gg n$ (with $n$ the number of data observations for $\mathbf{X}$).

It is possible to apply a significance test to the gain expressed in Equations  \ref{equation:delta:l1:gaussian} and \ref{equation:delta:l2:gaussian} by using a variant of the likelihood ratio test \citep[Par. 7.1]{rencher2003methods}. Indeed, if we have two alternative covariance matrices called for instance $\hat \Sigma_1$ and $\hat \Sigma_0$ it is possible to test whether they are significantly different by using the following statistics:
\begin{equation} \label{eq:lr:gaussian}
u = \nu \left( \log \vert \hat \Sigma_0 \vert - \log \vert \hat \Sigma_1 \vert + \Tr \left( \hat \Sigma_1^{-1} \hat \Sigma_0  \right) - p  \right)\;.
\end{equation}
Where $\nu$ is, in the case of un-pooled data, the number of observation in the time series minus one\footnote{It is also possible to apply a small sample correction to the statistics $u$, see \citet[Eq. 7.2]{rencher2003methods} for the details.}, $(n-1)$.

In general $u$ is $\chi^2$ distributed with the number of degrees of freedom equal to $\frac{1}{2} p(p+1)$, with $p$ the dimension of the matrix. However, if the two matrices are nested, the number of degrees of freedom is the number of additional parameters in the nested model, with respect to the external model. In the case of the clique expansion operator the degrees of freedom are therefore $p-1$. 

%
%

\section{Experiments} \label{sec:experiments}

We performed a set of experiments to analyse the performances of the \mfcf{} methodology by comparing the results against two  methodologies for covariance or correlation matrices estimation that are widely accepted in the literature \citep{fan2011sparse}: the Graphical Lasso \citep{FriedmanEtAl2008} and a shrinkage estimator.

Since the Graphical Lasso optimises a penalised version of Gaussian log-likelihood, we have used two score functions that are closely related to it, so that the results are effectively comparable:

\begin{itemize}
	\item Gaussian log-likelihood, described in Section \ref{sec:test:gain:1}, where we fix the size of the cliques to the same  value and evaluate the results for a range of clique sizes.
	\item Gaussian log-likelihood statistically validated, described in Section \ref{sec:test:gain:2}, where we allow cliques of any size up to a maximum value. 
\end{itemize}

\subsection{Construction of the precision matrix in the multivariate Gaussian case} \label{sec:PrecisionMatrix}
In the Gaussian case, once the clique forest structure is known, the maximum likelihood estimate of the precision matrix is given in explicit form (see \citet[Prop. 5.9]{lauritzen1996} or \citet{barfuss2016parsimonious}):
\begin{equation} \label{ml:inverse}
\hat J = \sum_{c \in \CG} \left[  \left( \hat \Sigma_c \right)^{-1} \right]^V  - \sum_{s \in \SG} \left[  \left( \hat \Sigma_s \right)^{-1}\right]^V
\end{equation}

\noindent where the notation $\left[ M_c \right]^V$ in Equation \ref{ml:inverse} means a matrix of dimension $p = |V|$ where all the elements are zero, excepting for the ones with the indices in the clique $c$; that is $\left[ M_c \right]^V_{ij} = M_{ij}$ if $i \in c$ and $j \in c$, $\left[ M_c \right]^V_{ij} = 0$ otherwise.

\subsection{MFCF shrinkage procedures} \label{sec:shrinkage:target}
In all the experiments we have applied some shrinkage to the maximum likelihood estimate. The shrinkage parameter has been calibrated to the validation data set by looking for the best likelihood and performing a grid search. 
We have used two shrinkage targets: the commonly used identity matrix \citep[Sec. 3.3]{ledoit2003improved}, and a new shrinkage target we call ``the clique tree target'', described in \ref{sec:clique:tree:target}. 

\subsection{Gain Functions Used}

\subsubsection{Gaussian log-likelihood} \label{sec:test:gain:1}

For the gain function we use Equations \ref{equation:delta:l1:gaussian} and \ref{equation:delta:l2:gaussian} to compute the increase in 
log-likelihood for a clique expansion in the multivariate gaussian case. 
Let $C_a$ denote a clique and $v_0$ an isolated vertex. The number of elements in $C_a$ is equal to the clique size ($k$).
Given a vertex outside $C_a$, the gain function must return both the gain and a subset of the vertices in $C_a$ yielding the highest statistically significant gain relative to $v_0$.
The construction of the separator is performed in the gain function in a greedy way, first by 
selecting from $C_a$ the vertex $v_1$ with the highest gain relative to with $v_0$, so 
that $S \leftarrow v_1$. The next best vertex $v_2$ is the one that added to $S$ 
increases the most the log-likelihood (Equation \ref{equation:delta:l2:gaussian}), and we put $S \leftarrow S \cup v_2$, and so on until the clique size is reached.

\subsubsection{Gaussian log-likelihood Statistically Validated} \label{sec:test:gain:2}
This  scoring function is also based on the increase of the log-likelihood 
in the multivariate Gaussian case. However, in this case we perform the clique expansion only if the gain is significantly greater than zero using Equation \ref{eq:lr:gaussian} to test the hypothesis within a given confidence 
level. If it is not statistically different from zero the corresponding score is 
set to zero.  This results in sparser networks than the ones 
obtained using the score function in \ref{sec:test:gain:1}, especially when the number of data points is low.

The selection of the separator is also performed in a greedy way similar to process 
described in \ref{sec:test:gain:1}, with the difference that the process might 
stop before the maximum clique size is reached if the increase in log-likelihood 
is not statistically significant.

\subsection{Data} \label{sec:Data}

We test the performance of the algorithm on three types of synthetic data and on a real dataset of stocks returns.  The synthetic data are multivariate Gaussian generated using respectively: 
(1) a sparse chordal inverse matrix with known sparsity pattern; 
(2) a factor model; 
(3) a random positive definite matrix generated from random eigenvalues and a random rotation. 
The real example is taken from a long-return series of stock prices. 
All the datasets used in the experiments have been produced for 100 variables ($p = 100$) and varying time series lengths ($n \in \{25, 50, 75, 100, 200, 300, 400, 500, 750, 1000, 1500\}$). The details about the data generation process are described in Appendix, sub-Sections \ref{sec:random:clique:forest}, \ref{sec:random:cluster}, \ref{sec:random:factor} and \ref{sec:exp:real:data}.

For every type of data we generate the following datasets: 

\begin{enumerate}
	\item The \emph{train data set} which is used to learn the model parameters, such as the \mfcf{} network and the elements of the precision matrix. For every type of data we generate 5 distinct training data sets to test reproducibility. 
	\item The \emph{validation data set} is used to select the model hyper-parameters: these are the $L_1$ penalty for the graphical lasso, the shrinkage parameter for the shrinkage method, and the maximum clique size and shrinkage parameter for the \mfcf{}. For all methods we perform a grid search over the hyper-parameters and select the model that achieves the best likelihood on the validation dataset. In analogy with the train data we generate 5 distinct validation data sets.
	\item The \emph{test data set} is used to assess the performance of the models. We use 10 distinct test datasets for every training/validation data set and therefore for every data type we have 50 test datasets.
\end{enumerate}

The full description of the methods used to generate the dataset is in the Appendix (\ref{sec:appendix}).

\subsection{Algorithms and methodologies }

We have generated sparse inverse correlation\footnote{Although the problem is \emph{covariance} estimation, we have used correlation matrices in all of the test cases, since the problem is equivalent.} estimates with different constructions of the  \mfcf{} algorithm and compared their performances with the Glasso and shrinkage estimators, the real benchmark and the null hypothesis.

\begin{enumerate}

	\item \glasso{}\_XVAL: the Graphical Lasso \citep{FriedmanEtAl2008}; we use the implementation provided by the R package \texttt{huge} \citep{huge}. 
	The penalty parameter is estimated through cross-validation using an adaptive grid search in the interval $[0.01, 1]$ . 
	The precision matrix is estimated, for a given penalty parameter, on the training data set; the penalty parameter selected is the one that produces the estimate with the highest log-likelihood on the validation data set\footnote{The minimum penalty of $0.01$ has been used because for smaller values we have encountered convergence problems with some of the test cases.}. Performances are assessed on the test data sets.

	\item \shrinkage{}: a shrinkage estimator with target the identity matrix. We produce shrunk correlation matrices estimators from the training dataset using a grid search for the shrinkage parameter associated with the highest  likelihood on the validation data set.  
Performances are assessed on the test data sets. Recall that this method does not produce sparse precision matrices.

	\item \mfcffix{}: the \mfcf{} algorithm with fixed clique size, the shrinkage target is the clique tree target described in \ref{sec:shrinkage:target}, and the gain function described in \ref{sec:test:gain:1}. We proceed in two steps: initially the correlation matrix built from the training set is shrunk by a small parameter $\epsilon = 0.05$ using the identity matrix as a target\footnote{This step is performed to stabilise numerically the algorithm; otherwise the matrices for some  cliques might be near singular numerically and lead to problems in the calculation of the gains. The parameter $0.05$ has not been tuned but just used as a reasonably small number.}. Then we produce a set of models with clique sizes between 2 and 20.\footnote{Larger values would lead to essentially dense models.} 
The precision matrix estimates are produced using the training datasets and the shrinkage procedure described in Section \ref{sec:shrinkage:target}. 
The shrinkage parameter  is the one that achieves the best likelihood on the validation data set, estimated with a grid search as we do for the graphical lasso and the shrinkage estimators.

	\item \mfcffixi{}:  same as \mfcffix{}, excepting for the shrinkage target where we use the identity matrix.

	\item \mfcfvar{}: same as \mfcffix{} (in particular using the clique tree target as a shrinkage target) but  with variable clique sizes between 2 and 20 and the gain function described in \ref{sec:test:gain:2}. 
The p-value used (in \ref{sec:test:gain:2})  for the likelihood ratio test was 0.05. 

\item \mfcfvari{}: same as  \mfcfvar{} excepting the shrinkage target where we use the identity matrix.

	\item REAL\_OR\_ML:  the benchmark `real' precision matrix. For synthetic data, when the structure of the correlation matrix is known exactly, we use the exact inverse; in the case of real data, for which we do not know the real correlation matrix, we use the inverse of the sample correlation matrix computed on the entire time series. 

	\item NULL hypothesis: the identity matrix as the inverse precision matrix.

\end{enumerate}

\subsection{Performance indicators}

For every test set we collect the following performance indicators:

\begin{enumerate}

	\item Log likelihood;
	
	\item Accuracy;
	
	\item Sensitivity;
	
	\item Specificity;
	
	\item The \emph{correlation} of the estimated precision matrix with the true precision matrix;
	
	\item \emph{Eigenvalue distance}; and
	
	\item \emph{Eigenvalue inverse distance}.

\end{enumerate}

The definition for these performance indicators can be found in the Appendix (\ref{sec:perf:ind}).

%

\subsection{Results} \label{sec:results:full}

The full set of results for all the dataset is reported in the Appendix (\ref{sec:results:full}). In this section we wish to draw the attention to the main highlights of the method and some interesting features of the solution.

\subsubsection{Sparse Decomposable Systems}

Figure \ref{fig:ex:Chordal_red} (and Figure \ref{fig:ex:Chordal} in the Appendix) provide a box plot representing the mean, the confidence interval and the extreme values of the log-likelihood achieved by the algorithms over the test data sets, broken down by the length of the series\footnote{The boxplots in this paper have been produced with the R \citep{rlanguage} package GGPLOT2 \citep{ggplot}. According to the package documentation the first lower and upper hinges correspond to the first and third quantile, the upper whisker covers the values form the third quartile hinge to $1.5$ times the inter-quartile range away from the hinge, and similarly the lower whisker covers the values between the first quartile hinge and $1.5$ times the interquartile range below the hinge. The remaining points are considered outliers and plotted individually.}. 
We observe that, in all cases, the \mfcf{} algorithms outperform both the graphical lasso and the shrinkage estimator. 
The graphical lasso improves performances as the length increases but does not exceed \mfcf{}s. The same results are presented in tabular form in Table \ref{tab:Chordal_summary:1}.
The dispersion around the mean is similar for all methods and it has been computed by repeating the experiments on 50 independent datasets (10 testing sets for each 5 training and validating sets, as explained in section \ref{sec:Data} in the appendix).

\begin{figure} 
\centering
\includegraphics[scale=0.8]{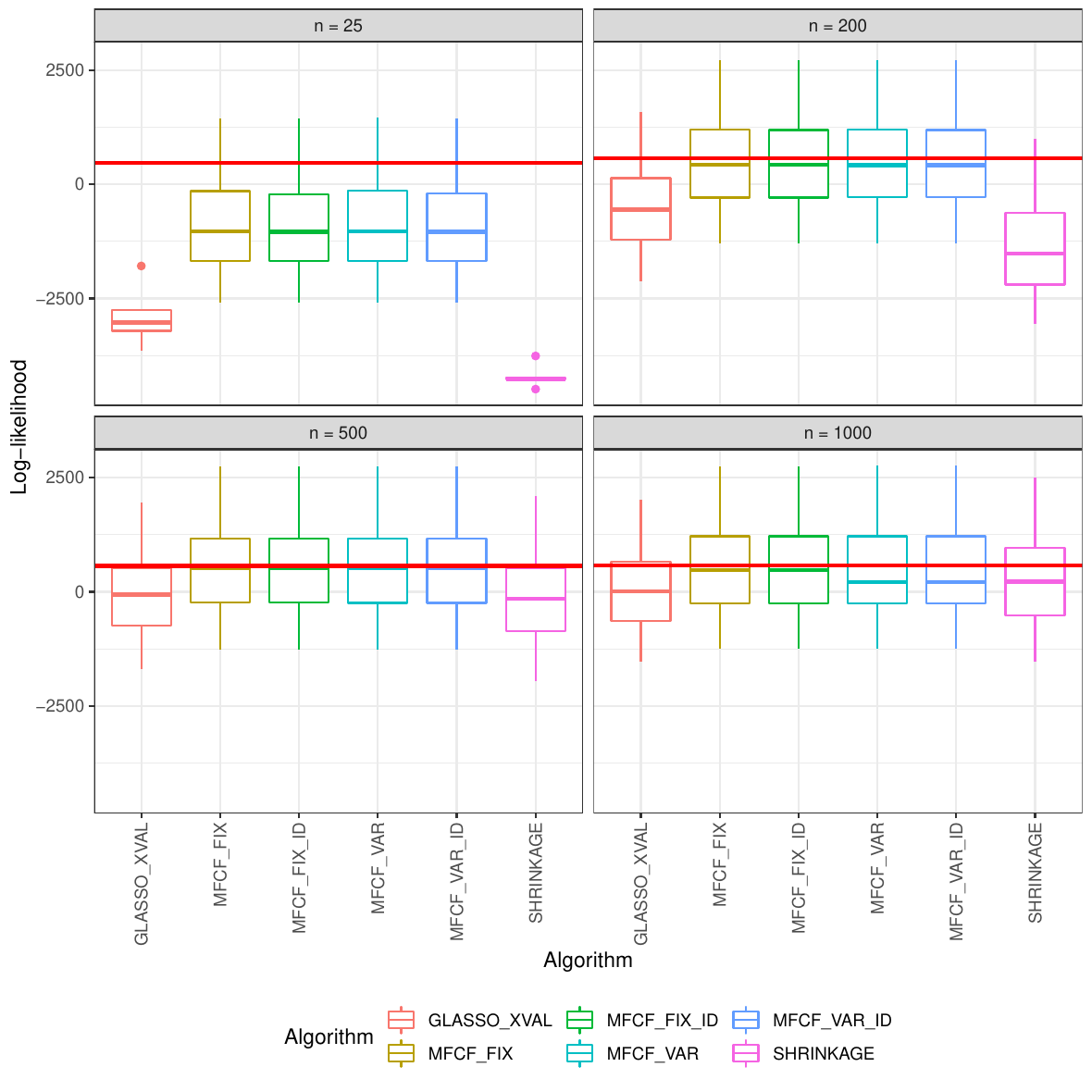}
  \caption{Box plot for the log-likelihood of the algorithms on synthetic data (sparse decomposable precision matrix) for different lengths of the series. The statistics is based on a total of 50 test sets (10 test sets for each of 5 different training / validation sets). The red horizontal line shows the average log-likelihood achieved by the real model, which is seen as an upper limit.}
  \label{fig:ex:Chordal_red}
\end{figure}

\begin{table}[ht]
\centering
\begingroup\footnotesize
\begin{tabular}{rrrrrrrr}
  \hline
\thead{Series \\ length} & \thead{GLASSO \\ XVAL} & \thead{MFCF \\FIX} & \thead{MFCF \\FIX ID} & \thead{MFCF \\VAR} & \thead{MFCF \\ VAR ID} & \thead{REAL \\ OR ML}& SHRINKAGE \\ 
  \hline
 25 & -2882.80 & -799.43 & -817.47 & -793.26 & -811.31 & 701.03 & -4200.59 \\ 
   50 & -2232.17 & 241.89 & 238.47 & 229.19 & 229.28 & 768.55 & -3590.65 \\ 
   75 & -1804.68 & 400.90 & 400.83 & 410.29 & 410.23 & 696.69 & -3117.99 \\ 
  100 & -1231.09 & 468.27 & 468.24 & 472.50 & 472.18 & 690.84 & -2643.30 \\ 
  200 & -436.25 & 548.88 & 548.13 & 547.93 & 546.42 & 659.23 & -1277.56 \\ 
  300 & -145.02 & 594.35 & 595.46 & 596.03 & 597.04 & 662.61 & -566.08 \\ 
  400 & -34.19 & 598.96 & 599.91 & 601.84 & 602.07 & 653.18 & -295.88 \\ 
  500 & -3.83 & 580.88 & 580.91 & 582.46 & 582.45 & 620.99 & -69.26 \\ 
  750 & 66.79 & 587.58 & 587.52 & 586.60 & 586.60 & 615.98 & 204.10 \\ 
  1000 & 100.79 & 590.33 & 590.33 & 537.03 & 537.03 & 612.48 & 327.54 \\ 
  1500 & 130.50 & 596.72 & 596.72 & 598.33 & 598.33 & 612.41 & 431.36 \\ 
   \hline
\end{tabular}
\endgroup
\caption{Log-likelihood of the models. Mean values over 5 distinct training / validation sets and for 10 test set each. (Underscore removed from the name of the algorithms to allow line breaks.)} 
\label{tab:Chordal_summary:1}
\end{table}

Figure \ref{fig:ex:Chordal:1:red} shows a summary of the performance measures. We observe that that the \mfcf{} family is better, overall,  than the graphical lasso especially for what concerns accuracy and specificity (the two measures are very close in this example, due to the high sparsity of the model). While the graphical lasso is more sensitive picking up more true positives. However, it is also less selective and produces denser precision matrices with a much higher number of false negatives. We observe that the performance of the graphical lasso improves in all measures for time series of length greater than 200, when the penalty parameter is essentially fixed at 0.01.
The \mfcf{} exhibit better log-likelihood, as already observed, and also larger correlations with the true precision matrix. 

\begin{figure}
\centering
\includegraphics[scale=0.8]{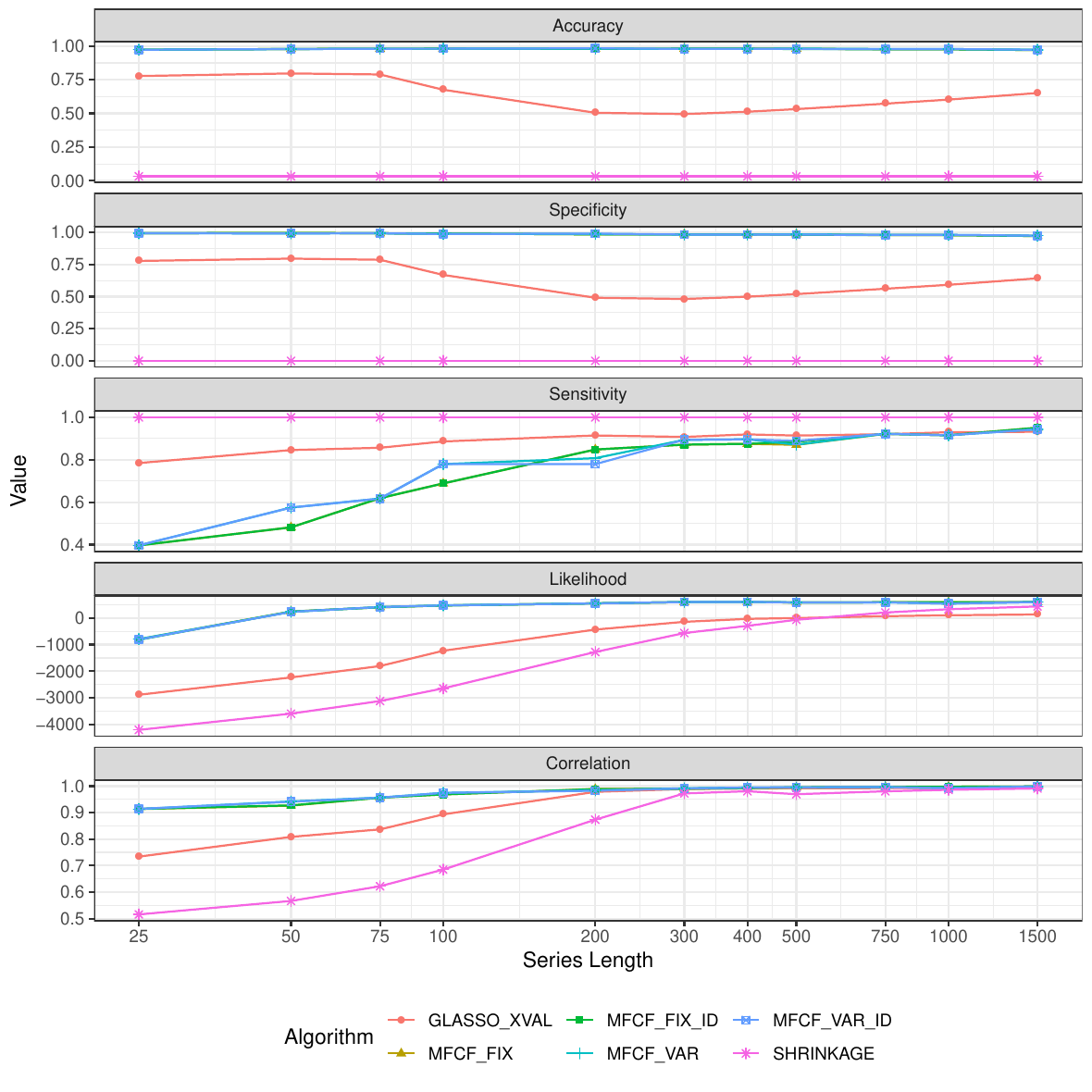}
 \caption{Performance measures of the algorithms on synthetic data (sparse decomposable precision matrix) for different lengths of the series. The statistics is based on a total of 50 test sets (10 test sets for each of 5 different training / validation sets).}
  \label{fig:ex:Chordal:1:red}
\end{figure}

Figure \ref{fig:ex:Chordal:4:red} (and Figure \ref{fig:ex:Chordal:4} in the Appendix) shows the number of cliques of different size produced by the \mfcfvar{} algorithm as a function of the maximum allowed clique size and of the time series length. We note that as the time series length increases the test becomes  less stringent with a higher number of large cliques in the model; conversely, when the time series is shorter ($n < p$), the models produced are more parsimonious. The number of cliques of size smaller than the maximum is linked to the degree of sparsity of the model. We will see in Section \ref{sec:factor:model:noise} that in the case of systems that are inherently dense the vast majority of the cliques will have the maximum allowed clique size.

\begin{figure}
\centering
\includegraphics[scale=0.8]{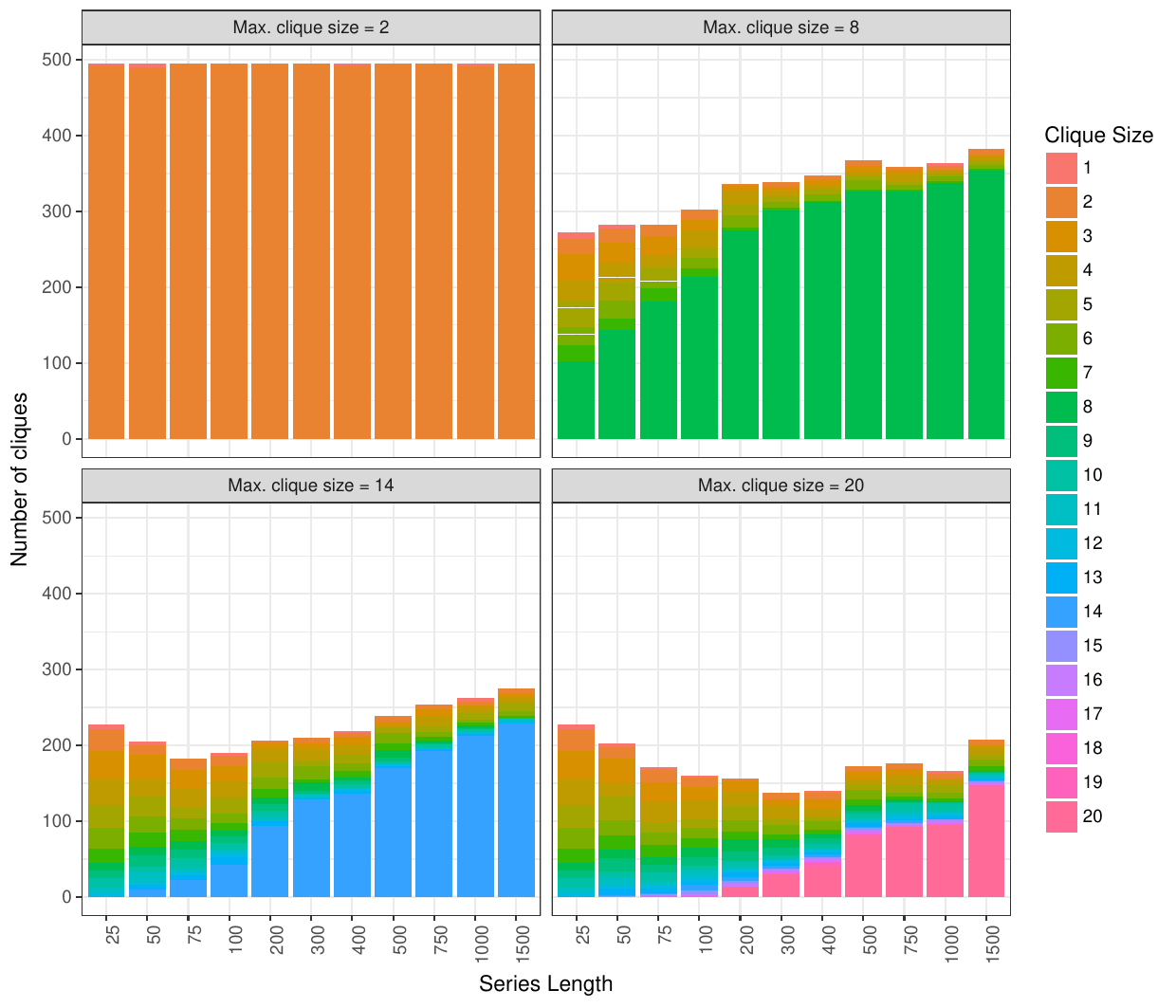}
 \caption{Composition of cliques from \mfcfvar{} for synthetic data (sparse decomposable precision matrix).   The statistics is based on a total of 5 different training / validation sets.}
  \label{fig:ex:Chordal:4:red}
\end{figure} 

\FloatBarrier

\subsubsection{Real Data}

Performance measures are reported in Figures \ref{fig:ex:Real:red}-\ref{fig:ex:Real:4:red}. Please note that this model is essentially non-sparse and therefore there are no true negatives; for this reason we do not report specificity and accuracy and sensitivity are essentially the same measure (see \ref{sec:perf:ind} for the definition of the performance indicators used. For the same reason the SHRINKAGE model performs better in terms of accuracy than the other models, but not significantly so in terms of log-likelihood. We see from the inspection of Figures \ref{fig:ex:Real:red} and \ref{fig:ex:Real:1:red} that with real data the log-likelihood is comparable across all models, with slight better values for \mfcffix{} and \mfcfvar{} for shorter time series. It is worth noting that, in the family of the \mfcf{} algorithms, the two that use the clique tree shrinkage target described in Equation \ref{ml:inverse:cost:corr} (\mfcffix{} and \mfcfvar{}) perform significantly better, for short time series, than the models with the same structure but the simpler identity matrix (\mfcffixi{} and \mfcfvari{}) as a shrinkage target. However, as Table \ref{tab:Real:param_no} shows, in the case of long time series Glasso uses a significantly higher number of parameters than the MFCF. The same results are shown in tabular form in Table \ref{tab:Real_summary:1}.

\begin{figure} 
\centering
\includegraphics[scale=0.8]{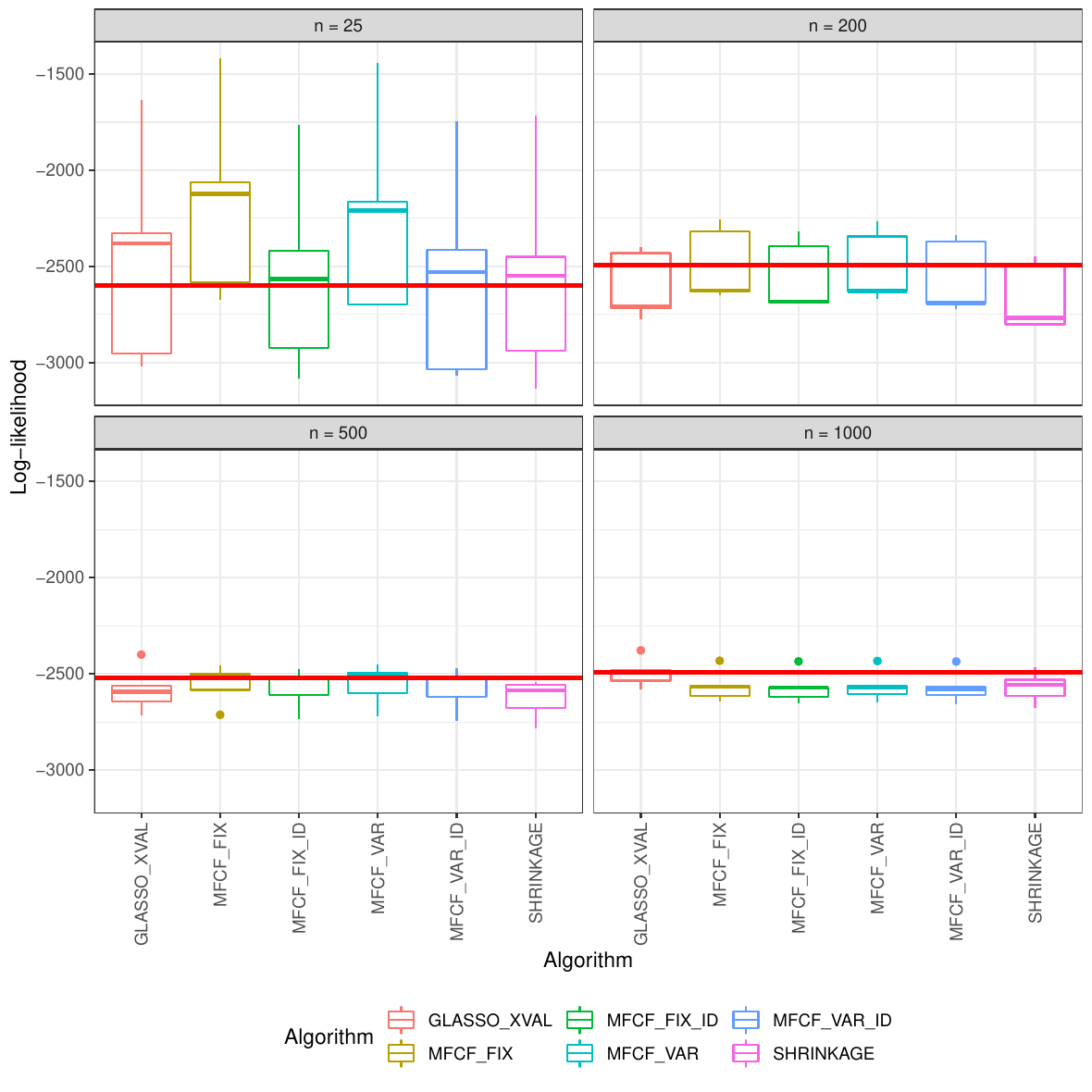}
  \caption{Log-likelihood of the algorithms on real data (stock returns) for different length of the series. The statistics is based on a total of 50 test sets (10 test sets for each of 5 different training / validation sets) per length of the time series. The red horizontal line shows the average log-likelihood achieved by the ``maximum likelihood'' model, the one where the correlation matrix is estimated from the full data set, i.e. using the longest possible time series.}
  \label{fig:ex:Real:red}
\end{figure} 

\begin{table}[ht]
\centering
\begingroup\footnotesize
\begin{tabular}{rrrrrrrr}
  \hline
\thead{Series \\ length} & \thead{GLASSO \\ XVAL} & \thead{MFCF \\ FIX} & \thead{MFCF \\ FIX ID} & \thead{MFCF \\ VAR} & \thead{MFCF \\ VAR ID} & \thead{REAL \\ OR ML} & SHRINKAGE \\ 
  \hline
 25 & -2463.20 & -2172.21 & -2550.56 & -2243.48 & -2559.13 & -1280.20 & -2558.49 \\ 
   50 & -2528.49 & -2427.37 & -2670.42 & -2438.69 & -2674.56 & -1461.33 & -2725.01 \\ 
   75 & -2618.80 & -2621.05 & -2761.78 & -2619.53 & -2759.37 & -1836.75 & -2821.22 \\ 
  100 & -2610.36 & -2578.57 & -2684.32 & -2572.94 & -2699.43 & -1822.45 & -2799.69 \\ 
  200 & -2605.54 & -2495.72 & -2555.83 & -2506.77 & -2561.94 & -1958.70 & -2664.73 \\ 
  300 & -2613.91 & -2527.65 & -2556.48 & -2524.18 & -2553.71 & -2089.36 & -2662.54 \\ 
  400 & -2618.82 & -2606.22 & -2628.84 & -2607.56 & -2630.98 & -2185.06 & -2705.29 \\ 
  500 & -2583.27 & -2551.57 & -2571.16 & -2554.57 & -2572.88 & -2171.48 & -2628.21 \\ 
  750 & -2450.17 & -2521.73 & -2531.22 & -2516.58 & -2527.05 & -2205.35 & -2563.06 \\ 
  1000 & -2493.24 & -2563.27 & -2569.82 & -2563.31 & -2569.77 & -2311.93 & -2568.76 \\ 
  1500 & -2467.06 & -2549.02 & -2553.15 & -2552.45 & -2556.22 & -2340.61 & -2519.63 \\ 
   \hline
\end{tabular}
\endgroup
\caption{Log-likelihood of the models. Mean values over 5 distinct training / validation sets and for 10 test set each. (Underscore character removed from the names of the alogorithms to allow for line breaks.} 
\label{tab:Real_summary:1}
\end{table}

\begin{figure}
\centering
\includegraphics[scale=0.8]{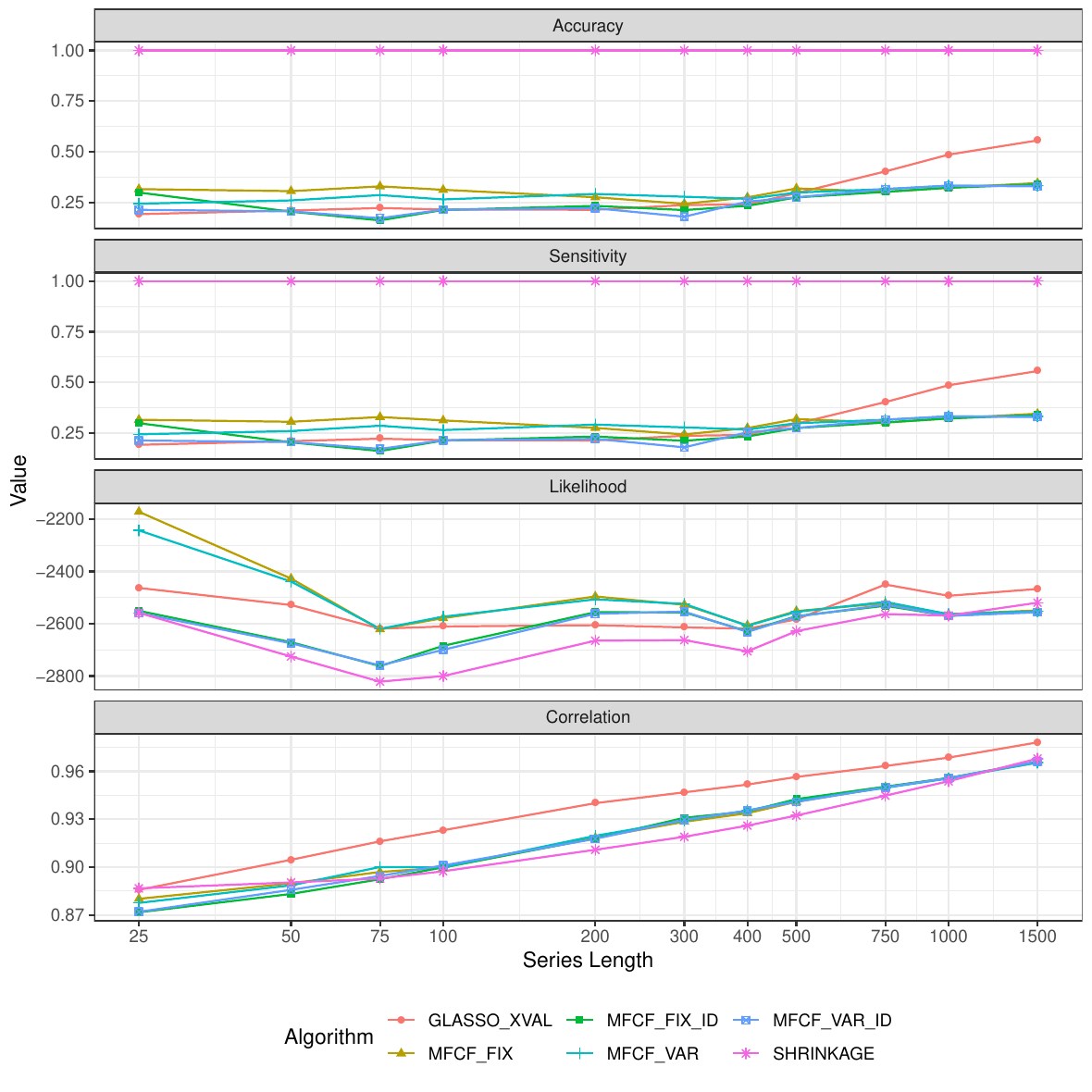}
 \caption{Performance of the algorithms on real data (stock returns) for different lengths of the series. The statistics is based on a total of 50 test sets (10 test sets for each of 5 different training / validation sets) per length of the time series.}
  \label{fig:ex:Real:1:red}
\end{figure} 

Figure \ref{fig:ex:Real:4:red} shows that, excepting for the shortest time series, the \mfcfvar{} algorithms almost always use the largest allowed clique size.

\begin{figure}
\centering
\includegraphics[scale=0.8]{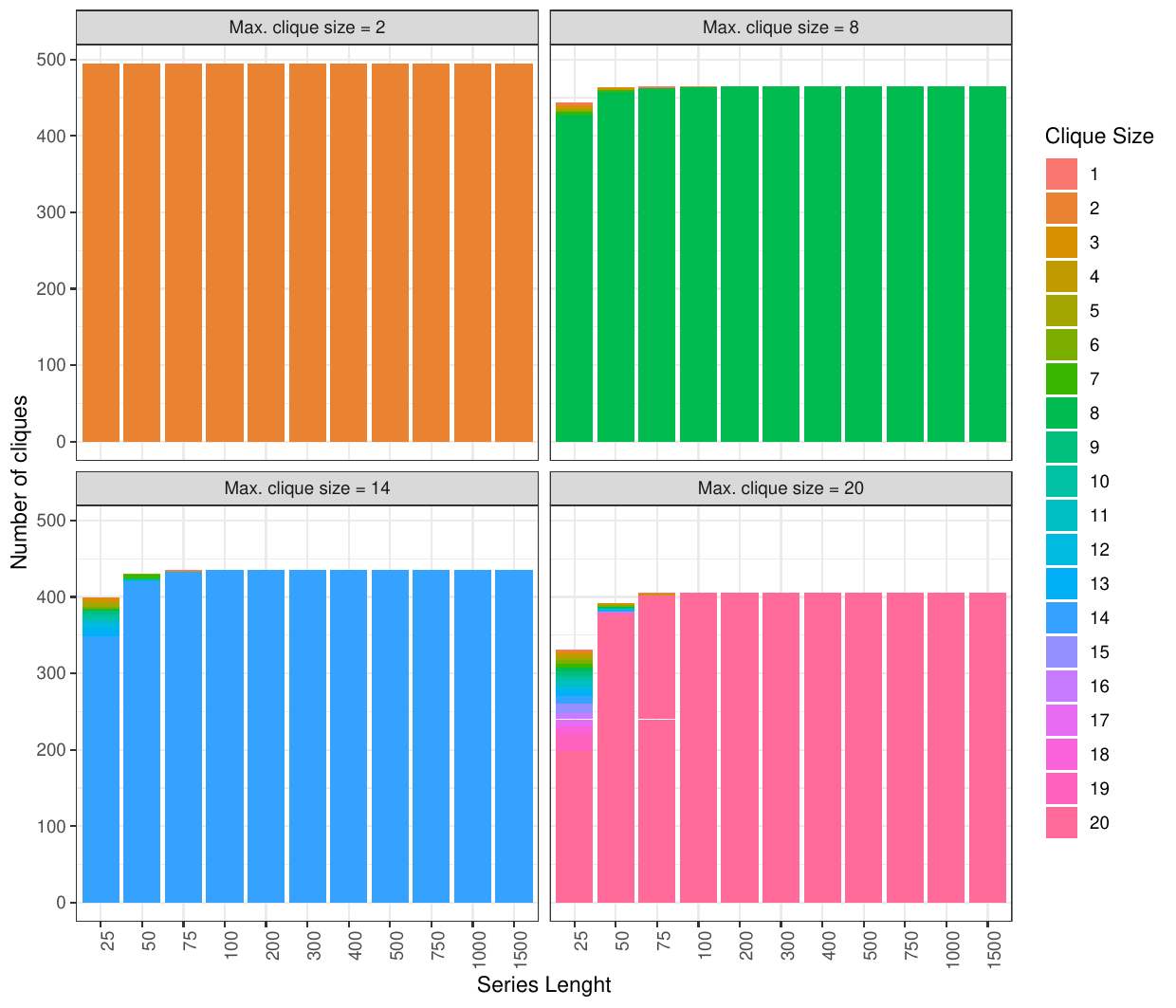}
 \caption{Composition of cliques from \mfcfvar{} for real data (stock returns). The statistics is based on a total of 5 different training / validation sets per length of the time series.}
  \label{fig:ex:Real:4:red}
\end{figure} 

\FloatBarrier
\section{Conclusions} \label{sec:conclusions}
In this paper we have introduced the \mfcf{} algorithm which generates chordal graphs (clique forests) through a local elementary move named clique expansion. 
The clique expansion move preserves chordality so it can be applied freely to local configurations without the need to check if the global graph structure is chordal.  
The  \mfcf{}  is a topological learning algorithm. 
In this paper it is used to estimate the inference structure of a system of random variables from data. 
The clique expansion move can be conditioned to various kinds of gain functions and statistical validations.
This methodology can be directly compared with well-established regularization techniques such as Shrinkage \citep{margaritis2000bayesian,zhou2011structure} and Lasso \citep{Tibshirani1996,FriedmanEtAl2008}.
For this purpose, in this paper, we discuss a set of experiments  for covariance selection using log-likelihood gain functions with both  synthetic data and real financial time series. 
Results reveal that  \mfcf{} can outperform both Shrinkage and Lasso in several cases with better relative performances particularly when the number of observations is small.  
 
We discovered that the structure of validated clique forests provides important insights on the sparsity and locality of the data. The geometric approach advocated by the \mfcf{} methodology provides, in our view, several benefits over other  approaches.

\begin{itemize}

\item The \mfcf{} calculates the clique-separator structure, as well as a perfect elimination ordering, and this can help with variable elimination and execution of the junction tree algorithm. 

\item The proposed methodology allows to seed the algorithm with a given clique tree, and this will help when incorporating previous established knowledge or expert judgement.

\item The clique forest structure allows also to identify which edges could be ``pruned'' without compromising the clique forest geometry \citep[Lemma 2.19]{lauritzen1996}. We have not made use of this possibility here but it is a possible direction of investigation for future works where the clique forest could be updated dynamically as the data changes. 

\item Finally we have also described a possible shrinkage target (the ``clique tree target'') that exploits the geometry of the clique forest and seems to provide encouraging results especially with small samples.

\end{itemize}

This work is a radical generalization of the TMFG methodology previously introduced by the authors in \cite{TMFG}.  
As for the TMFG case, the \mfcf{} generates a chordal graph associated to a multivariate system of variables where the network structure reflects the degree of conditional independence between the variables. 
TMFG structure was demonstrated to be useful across domains from psychology \cite{christensen2018network,christensen2018reopening} to finance \cite{barfuss2016parsimonious}.  
The main limitation in the TMFG approach was the rigidity of its structural construction that allowed 4-cliques  with triangular separators between couples of cliques only.  
\mfcf{} allows instead cliques of arbitrary sizes and the separators can be between multiple cliques. 
This, combined with the possibility to use different gain functions, greatly generalizes the construction of chordal graphs from data. 
This also lifts the overall topology of the generated chordal graphs beyond planarity opening the way for the use of  \mfcf{} in the field of topological data analytics and, in particular, in the domain of persistent homology \citep{edelsbrunner2008persistent}. 
 Indeed, the local and the step-by-step nature of the  \mfcf{}  construction makes it easy to keep track of all topological and hierarchical properties of  the network.

Since the \mfcf{} is a generalisation of Prim's MST algorithm, a legitimate question would be what is, if possible, a generalisation of Kruskal's algorithm \citep{kruskal1956}. It turns out that it is possible to produce an analogue of the \mfcf{} based upon Kruskal's algorithm. This requires the introduction of a ``bridge'' operator based upon the direct join of two clique trees \citep[pagg. 22-23]{lauritzen1996}.
However, preliminary investigations reveal that the Kruskal version of the algorithm tends to privilege the creation of very tightly linked small cliques in the first steps which are difficult to join in later steps while keeping the decomposability of the system and a good performance.

The main purpose of this paper was to  introduce the \mfcf{} algorithm and the clique expansion elementary move demonstrating their effectiveness with a set of experiments. 
The experiments were limited to the covariance selection problem because this is a domain of particular interest and there are well-established  alternative methods which can be used for comparison.
Results shows that the \mfcf{} approach outperforms state-of-the-art Glasso and shrinkage methods.

The methodologies  introduced in this paper open many possibilities and have great potentials that will be explored in following works.  
For instance, we would like to explore a larger range of gain functions and validations beyond likelihood. 
Preliminary results  (see, \cite{DanielThesis19}) show that gain functions resulting form the application of local, non-linear regressions can be very powerful especially for the analytics of datasets with data at different frequencies (i.e. the combination of daily and quarterly time series) and of different kinds (i.e. the combination of continuous, discrete and categorical variables).

\section{Code} \label{sec:code}
The original R code used to carry out the experiments, as well as the code used to generate all the graphs and tables in this paper, is stored in project "Learnig Clique Forests (Code)" hosted in the Open Science Framework repository \citep{Massara_2021}.

An implementation of the \mfcf{} algorithm is also presnt in Alex Christensen's NetworkToolbox \citep{nwt}.

Finally, there is a GitHub repository \citep{Massara_mfcf_matlab} that contains ongoing developments and algorithms related to the \mfcf{} in Matlab \citep{Matlab} and Octave \citep{Octave}.


\acks{
We thank Daniel Savu for exploring complementary aspects of this algorithm that helped us to contextualise the value and potentials of the present apperoach. Many thanks to the UCL-FCA group for help with discussions and proof reading. Finally, we are grateful to the ESRC for funding the Systemic Risk Centre (ES/K002309/1), the EPSRC for funding the BARAC project (EP/P031730/1) and to EC for funding the FinTech project (H2020-ICT-2018-2 825215), that partially support this work.
}

\bibliography{Clique_Tree_Learning}

\section{Appendix} \label{sec:appendix}

In this section we include some additional definitions and theorems related to the \mfcf{} that highlight important characteristics of the algorithm, but that are not used in the specific application to the Covariance Selection problem.

\subsubsection{The Running Intersection Property and Perfect Sequences of Sets} 

\begin{definition}
An \emph{ordering} of the cliques of a graph is a bijective application $\sigma$ from the first $m$ natural numbers (where $m$ is the cardinality of $\CG$, the number of maximal cliques) into $\CG$, $\sigma: \{1, \dots, m\} \rightarrow \CG$. The cliques in $\CG$ are ordered as $C_{\sigma(1)} < \dots < C_{\sigma(m)}$. As a shorthand notation we will also write $\sigma = [C_1, \dots, C_m]$ meaning that the cliques are ordered according to $\sigma$.
\end{definition}

\begin{definition}
Let $G(V, E)$ be a graph, $\CG$ the set of cliques of $G$ and $\sigma = [ C_1, C_2, \dots, C_m ]$ an ordering of $\CG$.  
We say that $\sigma$ has the \emph{running intersection property} if for every clique $C_i$ with $2 \leq i \leq m$ there is a clique $C_j$, with $2 \leq j < i$ such that:
\begin{equation} \label{eq:rip}
C_i \cap (C_1 \cup C_2 \cup \dots C_{i-1}) \subset C_j
\end{equation}
\end{definition}

In graphical models the terms $H_i = (C_1 \cup C_2 \cup \dots C_{i})$ are called the ``histories'', whereas $S_i = C_i \cap (C_1 \cup C_2 \cup \dots C_{i-1})$ are called the separators and $R_i = C_i \setminus H_{i-1}$ the residuals. 
$C_j$ is called the \emph{parent} of $C_i$. 
\begin{definition}
If all the separators are complete, the sequence of cliques in that order is called a \emph{perfect sequence of subsets}. 
As the use of the term ``parent'' hints, a graph that has a perfect sequence of sets has also a clique forest (see \citet[Th. 3.4]{blairpeython1992}). 
\end{definition}

This means in practice  that, as we add cliques such as $C_i$ following 
the ordering of a perfect sequence of subsets, the intersection with the previous cliques is always contained in a 
single predecessor clique at most, which means that every clique has at most 
one ``parent'' clique in the ordering. This gives an heuristic illustration to the following 
result, formally proved in \cite[Th. 3.4]{blairpeython1992}: any connected graph 
has a clique tree if and only if the cliques have the running intersection 
property\footnote{For non connected graphs the same result applies to the 
connected components}.
This also introduces the concept of nested hierarchies with separators forming a poset as described in \cite{SongNestedHier}.
A property that can be used for clustering \citep{song2012hierarchical}. 

In summary the three following conditions are equivalent:
\begin{enumerate}
	\item $G$ is chordal,
	\item $G$ has a clique forest $\HT$, or
	\item there is and ordering of $\CG$ $\sigma = [C_1, \dots, C_m]$ that has the running intersection property.
\end{enumerate}

\subsubsection{Perfect Elimination Order}
Let $G(V,E)$ be a graph. Let $Adj(v_i)$ be the vertices that are adjacent to 
$v_i$ in $G$. 
Let $\sigma = [ v_1, v_2, \dots, v_n ]$ be an ordering of the vertices of  $G(V,E)$. We define $V_{[i,n]}$ as the set of vertices $\left\lbrace v_i, 
v_{i+1}, \dots, v_n \right\rbrace$ and $G_i$ as the graph induced by 
$V_{[i,n]}$.
\begin{definition}
We say that the ordering $\sigma$ is a \emph{perfect elimination order} if 
$Adj(v_i) \cap G_{i+1}$ is a clique in $G_{i+1}$. 
\end{definition}
This is the same as saying that $v_i$ is a \emph{simplicial vertex} in $G_i$. 
The underlying idea is that $v_i$ can be eliminated from the $G_i$ using the process of variable elimination without introducing any additional link; this is a fundamental property in recursive algorithms where variables are eliminated one at a time and allows to maintain the sparsity of the graph, see \citet[Ch. 12]{golumbic2004algorithmic} for applications to Gaussian Elimination.

\def\circleA{(0,0) circle (1.5cm)}
\def\circleB{(0:2cm) circle (1.5cm)}

\begin{figure} 
\centering
\begin{tikzpicture}
    \draw \circleA node {$A$};
    \draw \circleB node {$B$};
    \node at (1,0) {$C$}; 
\end{tikzpicture} 
 \caption{Set-theoretic representation of a decomposable system, with $C = A \cap B$.}
 \label{fig:decomposable:graph}  
\end{figure}

\subsection{Additional theorems}

\begin{theorem} \label{th:mfcf:chordality}
Let $G(V,E)$ a chordal graph with $|V| \ge 2$ and at least an isolated vertex $v_i$. 
Then expanding one clique of $G$  with $v_i$  does not introduce a  chordless cycle of length $\ge 4$.
\end{theorem}

\begin{proof}
Let us call $C_a$ any clique of $G$. We choose any subset $S \subset C_a$ as a 
separator of the clique expansion. If $|S|> 0$ we have that $S$ does not have 
any chordless cycle of length $\ge 4$ because $S$ is a complete induced subgraph 
of $C_a$ and the clique expansion adds all the edges between $v_i$ and any 
vertex of $S$, resulting in a clique $C_b = S\cup v_i$ which is complete and 
therefore free from chordless cycles of length $\ge 4$. If $S = G_a$ the expansion generates a larger clique 
$C_b = C_a \cup v_i$ which is complete. If $S$ is empty then the 
expansion trivially does not add any chordless cycle of length $\ge 4$.
\end{proof}

\begin{theorem} \label{prop:perfect:sequence}
The \mfcf{} algorithm produces cliques in a perfect order, provided that the 
initial cliques $C_I$ are arranged in a perfect order.
\end{theorem}

\begin{proof}
The demonstration can be performed by induction on the number of vertices added. 
Let's assume that the algorithm has added $m-1$ vertices, and by definition 
there are cliques $C_1, \dots, C_j$ that are perfectly ordered. When adding the 
next $v_m$ vertex there are three possibilities:
\begin{enumerate}
 \item[a)] The algorithm selects a clique $C_i, 1 \le i < j$ with a non empty 
separator $S_i$ and therefore a new clique $C_k = S_i \cup v_m$ is created. The 
separator is clearly complete and by construction we have that $S_i \subset 
C_i$.
 \item[b)] The algorithm  selects a clique $C_i, 1 \le i < j$ and the separator 
$S_i = C_i$ (extension of clique $C_i$). By hypothesis there is a clique $C_h$ 
with $1 \le h < i$ such that $S_i \subset C_h$ and $S_i$ is complete. Since 
$v_m$ was disconnected from all the cliques it was in particular disconnected 
from $C_h$ and therefore it does not change the intersection $C_h \cup S_i$, and 
therefore $C_h$ still fulfils the requirements that $S_i \subset C_h$.
 \item[c)] The algorithm does not select a clique and adds a new clique made 
only of the vertex $v_m$. The intersection with any clique is the empty set and 
the result follows trivially.
 \end{enumerate}
\end{proof}

\subsection{The clique tree shrinkage target} \label{sec:clique:tree:target}

The clique tree target is a generalisation of the constant correlation target \citep{Ledoit110} where the target matrix is created gluing together smaller correlation target matrices that represent the cliques of a clique tree. The matrix is built is steps, starting from the calculation of the average correlation between elements in every clique $c$ in the clique tree.

\begin{equation}
\hat{\rho}_c =  \frac{\sum_{i \in c, j \in c, j > i} (\Sigma_c)_{ij}}{\sum_{i \in c, j \in c, j > i} 1}
\end{equation}
\noindent Next we build a clique level correlation matrix for every clique $c$ using the following rules:

\begin{itemize}
\item If $i \in c, j \in c, i = j $ then $(\hat{\hat{\Sigma}}_c)_{ij} = 1$
\item If $i \in c, j \in c, i \neq j $ then $(\hat{\hat{\Sigma}}_c)_{ij} = \frac{ \sum_{c^{\prime} \in \CG, i \in c^{\prime}, j \in c^{\prime}} \hat{\rho}_{c^{\prime}}}{\sum_{c^{\prime} \in \CG, i \in c^{\prime}, j \in c^{\prime}}1}$, that is we calculate average correlation as the mean of the average correlations of all the cliques the element belongs to.
\end{itemize}

\noindent And finally we build the estimate for the inverse applying the ``Lauritzen formula'' (\ref{ml:inverse}) 
\begin{equation} \label{ml:inverse:cost:corr}
\hat{J} = \sum_{c \in \CG} \left[  \left( (1-\theta) \hat{\Sigma}_c + \theta \hat{\hat{\Sigma}}_c \right)^{-1} \right]^V  -
 \sum_{s \in \SG} \left[  \left(   (1-\theta) \hat{\Sigma}_s + \theta \hat{\hat{\Sigma}}_s   \right)^{-1}\right]^V
\end{equation}

\begin{remark}
The constant correlation estimates $\hat{\hat{\Sigma}}$ are positive definite because every $\hat{\hat{\Sigma}}_c$ $(\hat{\hat{\Sigma}}_s)$ is the normalized sum of positive definite matrices.
\end{remark}




\begin{theorem} \label{prop:peo}
The \mfcf{} algorithm adds vertices in reverse perfect elimination order, provided 
that the vertices in the initial cliques $C_I$ are arranged in a reverse 
perfect order.
\end{theorem}

\begin{proof}
By induction on the number of vertices, let us assume that the total number of 
vertices is $k$ and that $j$  have been added by the \mfcf{} and that they are 
ordered in reverse perfect order $\left\lbrace v_k, v_{k-1}, \dots, v_j \right\rbrace$. 
When adding the next vertex $v_i$ it is by construction added to a separator 
$S_i$ which is complete, and therefore $adj(i) \cap G_j = S_i$ is trivially 
complete.
\end{proof}

\section{Experiments}

\subsection{Data} \label{sec:Data}

We test the performance of the algorithm on three types of synthetic data and on a real dataset of stocks returns.  The synthetic data are multivariate Gaussian generated using respectively: 
(1) a sparse chordal inverse matrix with known sparsity pattern; 
(2) a factor model; 
(3) a random positive definite matrix generated from random eigenvalues and a random rotation. 
The real example is taken from a long-return series of stock prices. 
All the datasets used in the experiments have been produced for 100 variables ($p = 100$) and varying time series lengths ($n \in \{25, 50, 75, 100, 200, 300, 400, 500, 750, 1000, 1500\}$). The details about the data generation process are described in the sub-Sections \ref{sec:random:clique:forest}, \ref{sec:random:cluster}, \ref{sec:random:factor} and \ref{sec:exp:real:data}.

For every type of data we generate the following datasets: 

\begin{enumerate}
	\item The \emph{train data set} which is used to learn the model parameters, such as the \mfcf{} network and the elements of the precision matrix. For every type of data we generate 5 distinct training data sets to test reproducibility. 
	\item The \emph{validation data set} is used to select the model hyper-parameters: these are the $L_1$ penalty for the graphical lasso, the shrinkage parameter for the shrinkage method, and the maximum clique size and shrinkage parameter for the \mfcf{}. For all methods we perform a grid search over the hyper-parameters and select the model that achieves the best likelihood on the validation dataset. In analogy with the train data we generate 5 distinct validation data sets.
	\item The \emph{test data set} is used to assess the performance of the models. We use 10 distinct test datasets for every training/validation data set and therefore for every data type we have 50 test datasets.
\end{enumerate}

\FloatBarrier
\subsubsection{Synthetic data: sparse decomposable precision matrix} \label{sec:random:clique:forest}

This data has been produced with a multivariate model from a sparse inverse covariance matrix (the benchmark precision matrix) where the non-zero structure pattern is a clique forest. The clique forest was generated by applying  repeatedly the clique expansion operator  with a random choice of the vertices, cliques and separators that were available at any steps.


For every clique $c \in \mathcal{C}$  we have defined a factor $F_c$ distributed as $\mathcal{N}(1,1)$. Next we add up the factor contributions at the variable level: $X_i = \sum_{\left\lbrace  c \in \mathcal{C} \vert i \in c \right\rbrace} F_c + \epsilon_{i,c}$, where $\epsilon_{i,c} \sim \mathcal{N}(0, 0.1)$ is a small noise factor to avoid perfect correlation between the variables in the clique $c$. Finally we assemble the sparse inverse covariance matrix  by using the familiar sum over cliques and separators (see the description of Eq. \ref{ml:inverse} for the explanation of the notation): $ J = \sum_{c \in \CG} \left[  \left( \Sigma_c \right)^{-1} \right]^V  - \sum_{s \in \SG} \left[  \left( \Sigma_s \right)^{-1}\right]^V$. 

As an example to help with the intuition, let us consider two cliques $c_A$ and $c_B$ with non empty intersection $S = c_A \cap c_B$. In this case the structure of the precision matrix consists of two blocks that overlap on the variables $X_S$. This means that the partial correlations of the variables $X_A$ and $X_B$, controlling for $X_S$, are all zero, and this in turn means that $X_A \ci X_B \vert X_S$. In the particular case where $S = \emptyset$, the variables $X_A$ and $X_B$ are unconditionally independent.

The exact inverse of the precision matrix has been used to generate the training, validation and test data sets, using the function \texttt{mvrnorm} of the package MASS \citep{MASS} developed for the R language \citep{rlanguage}.

\FloatBarrier
\subsubsection{Synthetic data: Full Positive Definite Matrix from package ``clusterGeneration''}  \label{sec:random:cluster}

This data has been generated using the R package ``clusterGeneration'' (\cite{clusterGeneration}). The methodology is to produce a vector of random eigenvalues ($p = 100$ values in the range $[0.01, 100]$ in this experiment) and to {rotate} the diagonal matrix of eigenvalues with a random orthogonal matrix to produce a dense positive definite matrix that is used as the benchmark reference covariance. As in the previous example the generation of the data sets has been carried out using the package `MASS' as described in \ref{sec:random:clique:forest}.

\FloatBarrier
\subsubsection{Random Factor Model with noise} \label{sec:random:factor}

This data set has been generated by building a factor model with 5 factors. For a review of factor models, with particular regards to large factor models see \citet{bai2008large}; here we follow their conventions and model the variables $\mathbf{X}$ as $\mathbf{X} = \mathbf{\Lambda} \mathbf{F} + \mathbf{\epsilon}$ where: $\mathbf{F}$ is an $f \times n$ matrix, with $f < p $ the number of factors, $\mathbf{\Lambda}$ is the $p \times f$ matrix of \emph{factor loadings} and $\mathbf{\epsilon}$ is the $p \times n$ idiosyncratic term.

Accordingly, the correlation matrix breaks down in two parts: $\mathbf{\Sigma} = \mathbf{\Lambda} \mathbf{\Lambda'} + \mathbf{\Omega}$, where  $\mathbf{\Lambda} \mathbf{\Lambda'}$ is the systematic component and $\mathbf{\Omega}$ is  the idiosyncratic component .

The training, validation and test matrices have been generated using $f = 5$. The factor loadings have been randomly generated from independent normal distribution and the factors have been generated as independent normal variates. As the factor loadings are in general different from zero, this model is dense.

\FloatBarrier
\subsubsection{Real Data} \label{sec:exp:real:data}

This data set contains a set of stock returns for 342 companies over 4025 trading days, as described in \cite{barfuss2016parsimonious}. 
For every training, validation and test execution we have sampled randomly without replacement $p=100$ time series. 
The training, validation and testing datasets have been sampled taking days, with replacement, from the total time series of 4025 trading days.

The estimate of the `real'  reference covariance matrix has been produced using the full dataset, and this has been used as a benchmark for the estimates produced by the models.

\subsection{Performances indicators} \label{sec:perf:ind}

For every test set we collect the following performance indicators: \footnote{We define: TP (True Positives) as the count of elements in the precision matrix that are correctly predicted as different from zero, TN (True Negatives) as the count of elements in the precision matrix that are correctly predicted as zero, and FP (False Positives) and FN (False Negatives) in analogous fashion. This is possible only when the `true' precision matrix is known (synthetic data) and these measures are meaningful only when it is sparse.}

\begin{enumerate}

	\item Log likelihood, which is   $ \frac{p}{2} \left( \log \vert J \vert  - \Tr \left( \hat{\Sigma} \cdot J \right) \right)$ (consistently with the definition of the objective function used in the R package glasso we omit the constant). Please note that $J$ is the precision matrix estimated using the training dataset, while $\hat{\Sigma}$ is the sample correlation estimated on the test dataset.
	
	\item $Accuracy = \frac{TP + TN}{TP + TN + FP + FN}$, which is the fraction of entries in the precision matrix $J$ that are correctly predicted as zero or non-zero.
	
	\item $Sensitivity = \frac{TP}{TP+FN}$, which is the fraction of non-zero entries in the precision matrix $J$ that are correctly predicted by the models.
	
	\item $Specificicty = \frac{TN}{TN+FP}$, which is the fraction of zero entries in the precision matrix $J$ that are correctly predicted by the model.  
	
	\item The \emph{correlation} of the estimated precision matrix with the true precision matrix (which is known in the case of synthetic data) or with the maximum likelihood estimate of the precision matrix computed on the longest possible data set (in the case of real data). The correlation is calculated as if the two matrices were vectors in $\mathbb{R}^{p^2}$.
	
	\item \emph{Eigenvalue distance} is the $\mathbb{R}^2$ norm of the vector of differences of the eigenvalues of the real or maximum-likelihood estimate precision matrix and the estimated precision matrix $\left(\sum_{i=1}^p (\hat{\lambda}_i - \lambda_i)^2 \right)^{\frac{1}{2}}$.
	
	\item \emph{Eigenvalue inverse distance} is the $\mathbb{R}^2$ norm of the the vector of differences of the reciprocal of the eigenvalues of the real or maximum-likelihood estimate precision matrix and the estimated precision matrix $\left(\sum_{i=1}^p (\hat{\lambda}_i^{-1} - \lambda_i^{-1})^2 \right)^{\frac{1}{2}}$.

\end{enumerate}

\subsection{Results}

\subsubsection{Synthetic data: sparse decomposable precision matrix} \label{sec:random:clique:forest:res}

Figure \ref{fig:ex:Chordal} provides a box plot representing the mean the confidence interval and the extreme values of the log-likelihood achieved by the algorithms over the test data sets, broken down by the length of the series\footnote{The boxplots in this paper have been produced with the R \citep{rlanguage} package GGPLOT2 \citep{ggplot}. According to the package documentation the first lower and upper hinges correspond to the first and third quantile, the upper whisker covers the values form the third quartile hinge to $1.5$ times the inter-quartile range away from the hinge, and similarly the lower whisker covers the values between the first quartile hinge and $1.5$ times the interquartile range below the hinge. The remaining points are considered outliers and plotted individually.}. 
We observe that, in all cases, the \mfcf{} algorithms outperform both the graphical lasso and the shrinkage estimator. 
The graphical lasso improves performances as the length increases but does not exceeds \mfcf{}s.
The dispersion around the mean is similar for all methods and it has been computed by repeating the experiments on 50 independent datasets (10 testing sets for each 5 training and validating sets, as explained in \ref{sec:Data}).

Table \ref{tab:Chordal:shrink} reports the average value of the graphical lasso penalty parameter and of the shrinkage parameter as selected by the grid search. 
As expected, the parameters become smaller as the series length grows, with \mfcf{}s requiring less shrinkage/penalisation than the other methodologies\footnote{The comparison of penalty and shrinkage parameter is purely indicative, as the two parameters are not directly comparable.}, especially with short time series. We believe that this is a desirable feature of the \mfcf{} algorithm: the topological constraint allows to estimate with good accuracy the cliques with high likelihood, and excludes edges with low likelihood with the end effect of requiring less shrinking.

\begin{table}[ht]
\centering
\begingroup\footnotesize
\begin{tabular}{|r|C{2cm}|C{2cm}|C{2cm}|C{2cm}|C{2cm}|C{2cm}|}
  \hline
\thead{Series \\ length} & \thead{GLASSO \\ XVAL} & \thead{MFCF \\ FIX} & \thead{MFCF \\ FIX ID} & \thead{MFCF \\ VAR} & \thead{MFCF \\ VAR ID} & SHRINKAGE \\ 
  \hline
  25 & 0.150 & 0.005 & 0.005 & 0.005 & 0.005 & 0.726 \\ 
    50 & 0.120 & 0.002 & 0.002 & 0.002 & 0.001 & 0.522 \\ 
    75 & 0.102 & 0.002 & 0.001 & 0.002 & 0.001 & 0.374 \\ 
   100 & 0.056 & 0.001 & 0.001 & 0.001 & 0.001 & 0.257 \\ 
   200 & 0.014 & 0.001 & 0.001 & 0.001 & 0.001 & 0.074 \\ 
   300 & 0.011 & 0.001 & 0.000 & 0.001 & 0.000 & 0.022 \\ 
   400 & 0.010 & 0.000 & 0.000 & 0.000 & 0.000 & 0.020 \\ 
   500 & 0.010 & 0.000 & 0.000 & 0.000 & 0.000 & 0.004 \\ 
   750 & 0.010 & 0.000 & 0.000 & 0.000 & 0.000 & 0.000 \\ 
  1000 & 0.010 & 0.000 & 0.000 & 0.000 & 0.000 & 0.000 \\ 
  1500 & 0.010 & 0.000 & 0.000 & 0.000 & 0.000 & 0.000 \\ 
   \hline
\end{tabular}
\endgroup
\caption{Mean penalty (GLASSO\_XVAL) or shrinkage parameters by length of time series. The statistics is based on 5 different calibrations (one per each training / validation set) of the shrinkage parameters per each length of the time series.}
\label{tab:Chordal:shrink}
\end{table}

Table \ref{tab:Chordal:param_no} reports the number of non zero elements in the precision matrix for every length of the time series. One can observe that the \mfcf{} algorithms are much more parsimonious than the graphical lasso (the shrinkage method produces always a full precision matrix). 

\begin{table}[ht]
\centering
\begingroup\footnotesize
\begin{tabular}{|r|R{2cm}|R{2cm}|R{2cm}|R{2cm}|R{2cm}|R{2cm}|}
  \hline
\thead{Series \\ length} & \thead{GLASSO \\ XVAL} & \thead{MFCF \\ FIX} & \thead{MFCF \\ FIX ID} & \thead{MFCF \\ VAR} & \thead{MFCF \\ VAR ID} & SHRINKAGE \\ 
  \hline
25 & 1194 & 99 & 99 & 98 & 98 & 4950 \\ 
  50 & 1118 & 99 & 99 & 135 & 135 & 4950 \\ 
  75 & 1161 & 138 & 138 & 136 & 136 & 4950 \\ 
  100 & 1730 & 158 & 158 & 191 & 191 & 4950 \\ 
  200 & 2586 & 216 & 216 & 195 & 176 & 4950 \\ 
  300 & 2632 & 216 & 216 & 231 & 231 & 4950 \\ 
  400 & 2549 & 216 & 216 & 231 & 231 & 4950 \\ 
  500 & 2451 & 216 & 236 & 213 & 231 & 4950 \\ 
  750 & 2254 & 255 & 255 & 246 & 246 & 4950 \\ 
  1000 & 2108 & 255 & 255 & 244 & 244 & 4950 \\ 
  1500 & 1867 & 294 & 294 & 281 & 281 & 4950 \\ 
   \hline
\end{tabular}
\endgroup
\caption{Mean number of non-zero coefficient in the precision matrix by length of time series. The statistics is based on 5 different calibrations (one per each training / validation set) per each length of the time series.} 
\label{tab:Chordal:param_no}
\end{table}

\begin{figure} 
\centering
\includegraphics[scale=0.8]{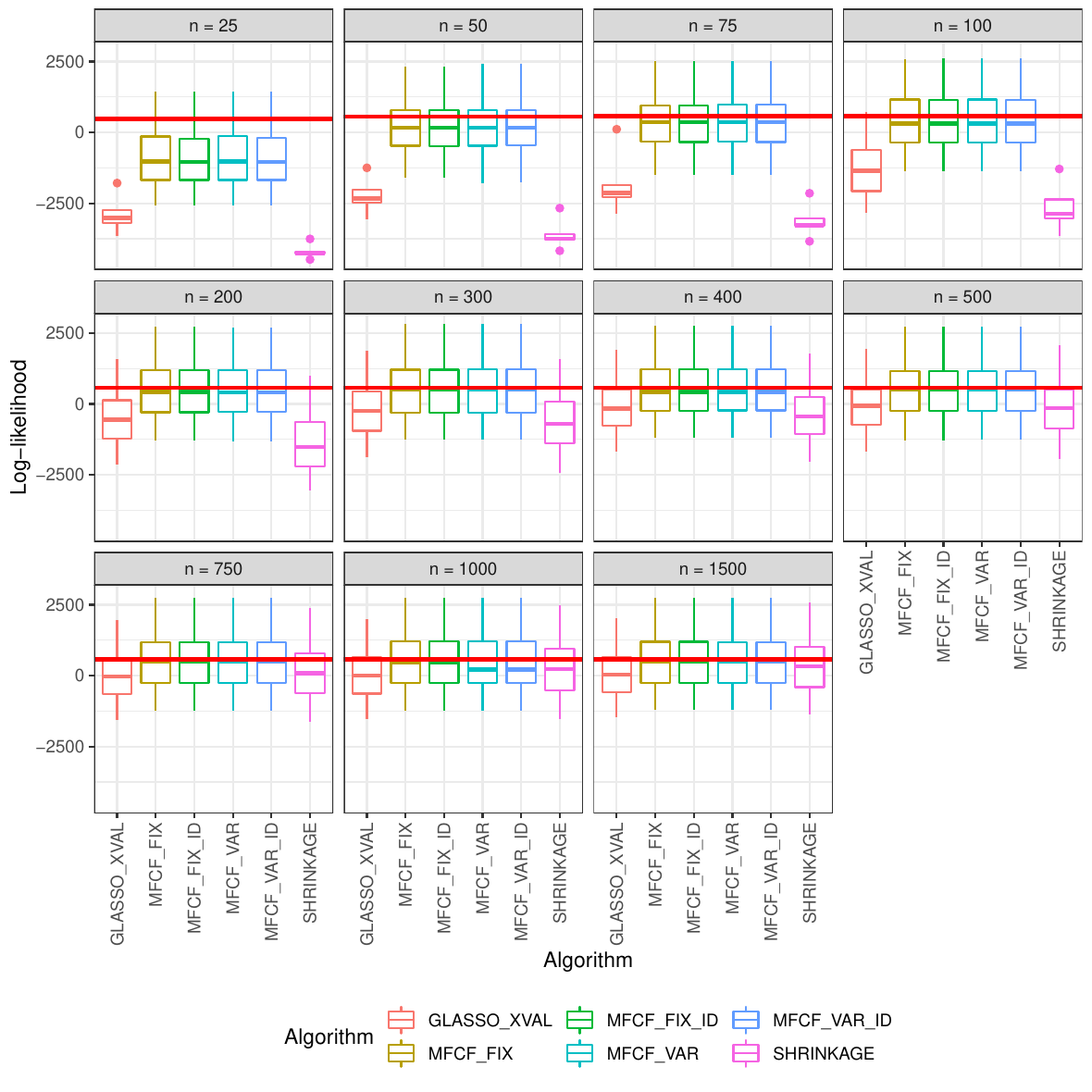}
  \caption{Box plot for the log-likelihood of the algorithms on synthetic data (sparse decomposable precision matrix) for different lengths of the series. The statistics is based on a total of 50 test sets (10 test sets for each of 5 different training / validation sets).}
  \label{fig:ex:Chordal}
\end{figure} 

Figure \ref{fig:ex:Chordal:1} shows a summary of the performance measures. We observe that that the \mfcf{} family is better, overall,  than the graphical lasso especially for what concerns accuracy and specificity. While the graphical lasso is more sensitive picking up more true positives. However, it is also less selective and produces denser precision matrices with a much higher number of false negatives. We observe that the performance of the graphical lasso improves in all measures for time series of length greater than 200, when the penalty parameter is essentially fixed at 0.01.
The \mfcf{} exhibit better log-likelihood, as already observed, and also larger correlations with the true precision matrix. 

\begin{figure}
\centering
\includegraphics[scale=0.8]{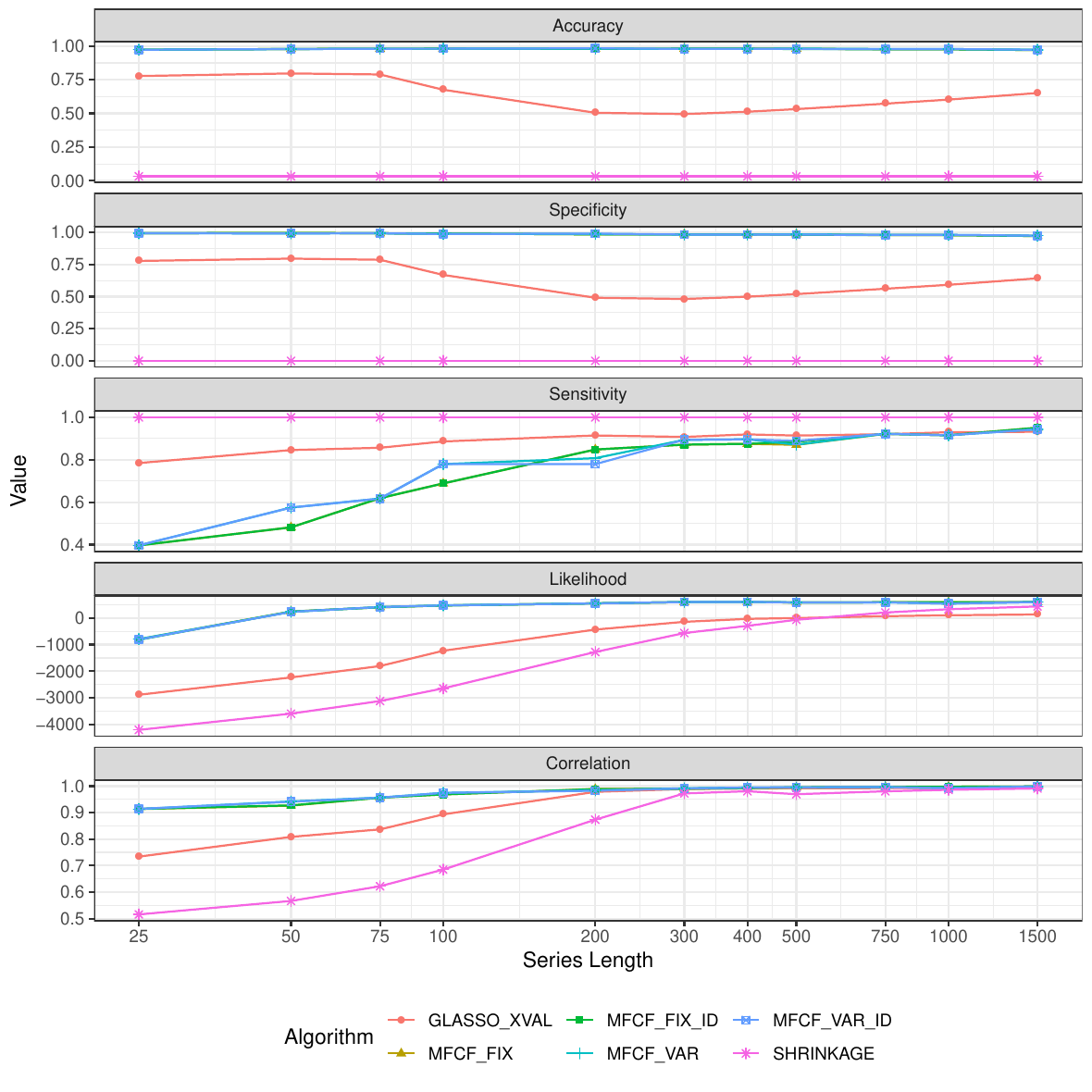}
 \caption{Performance measures of the algorithms on synthetic data (sparse decomposable precision matrix) for different lengths of the series. The statistics is based on a total of 50 test sets (10 test sets for each of 5 different training / validation sets).}
  \label{fig:ex:Chordal:1}
\end{figure} 

Figure \ref{fig:ex:Chordal:2} shows the distance between the spectra of the precision matrix produced by the models and the true precision matrix. The measure is normalised so that the identity matrix has distance  one. We observe that the \mfcf{} algorithms always perform better and the performance improves for all methods as the time series length increases, with the exception of the graphical lasso in the region where the penalty parameter is floored at 0.01.

\begin{figure}
\centering
\includegraphics[scale=0.8]{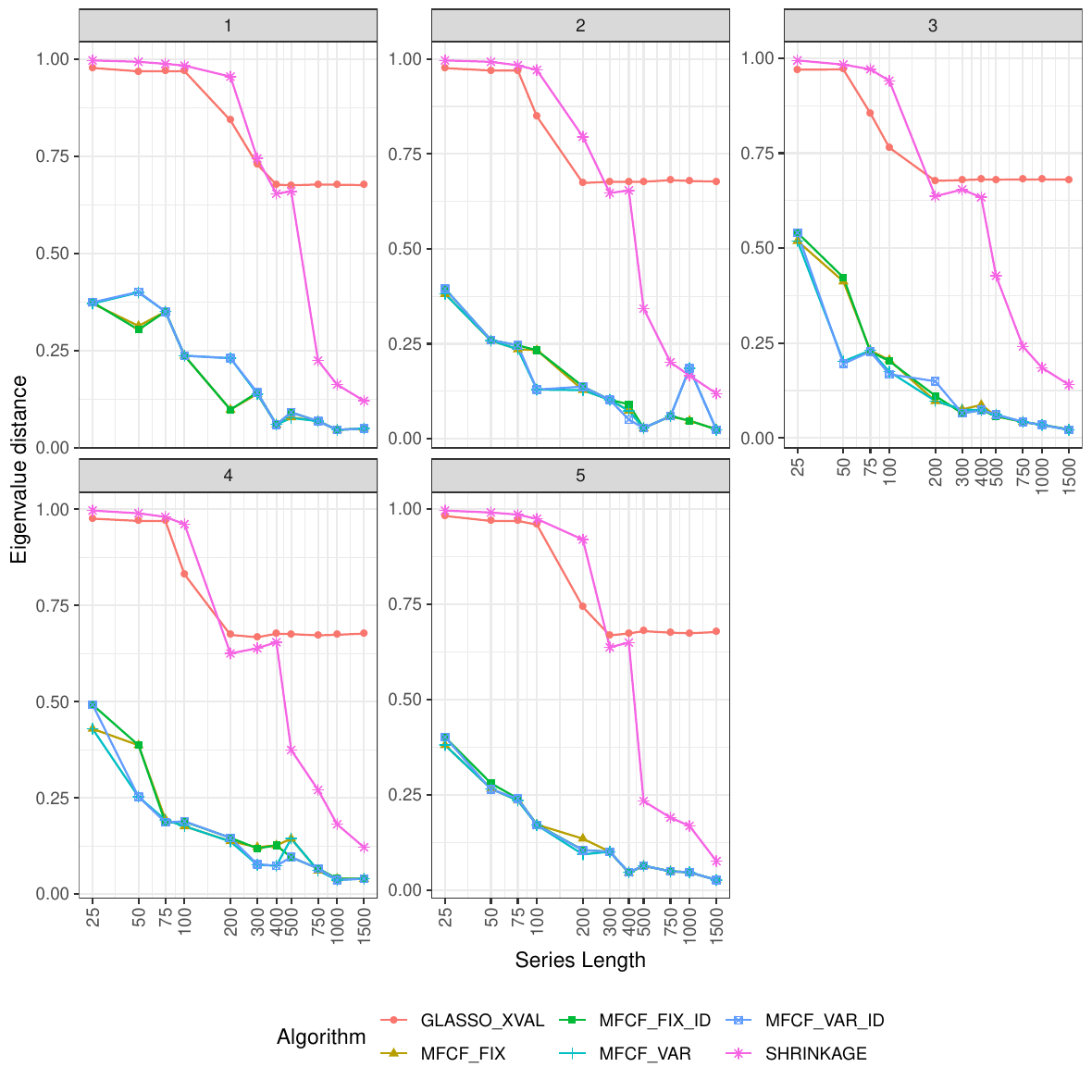}
 \caption{Eigenvalue distance for synthetic data (sparse decomposable precision matrix). The five panels show the values at different time series lengths for 5 training  / validation datasets.}
  \label{fig:ex:Chordal:2}
\end{figure} 

Figure \ref{fig:ex:Chordal:3} shows the distance between the inverse spectra of the precision matrix ($\mathbf{\lambda_i^{-1}}$) produced by the models and the ones for the true precision matrix. We observe that the \mfcf{} algorithms perform slightly better than the  graphical lasso and shrinkage but performance is very similar. Interestingly, in this case, the distance decreases with the time series length for all algorithms and there is no apparent effect due to the flooring of the graphical lasso penalty parameter.

\begin{figure}
\centering
\includegraphics[scale=0.8]{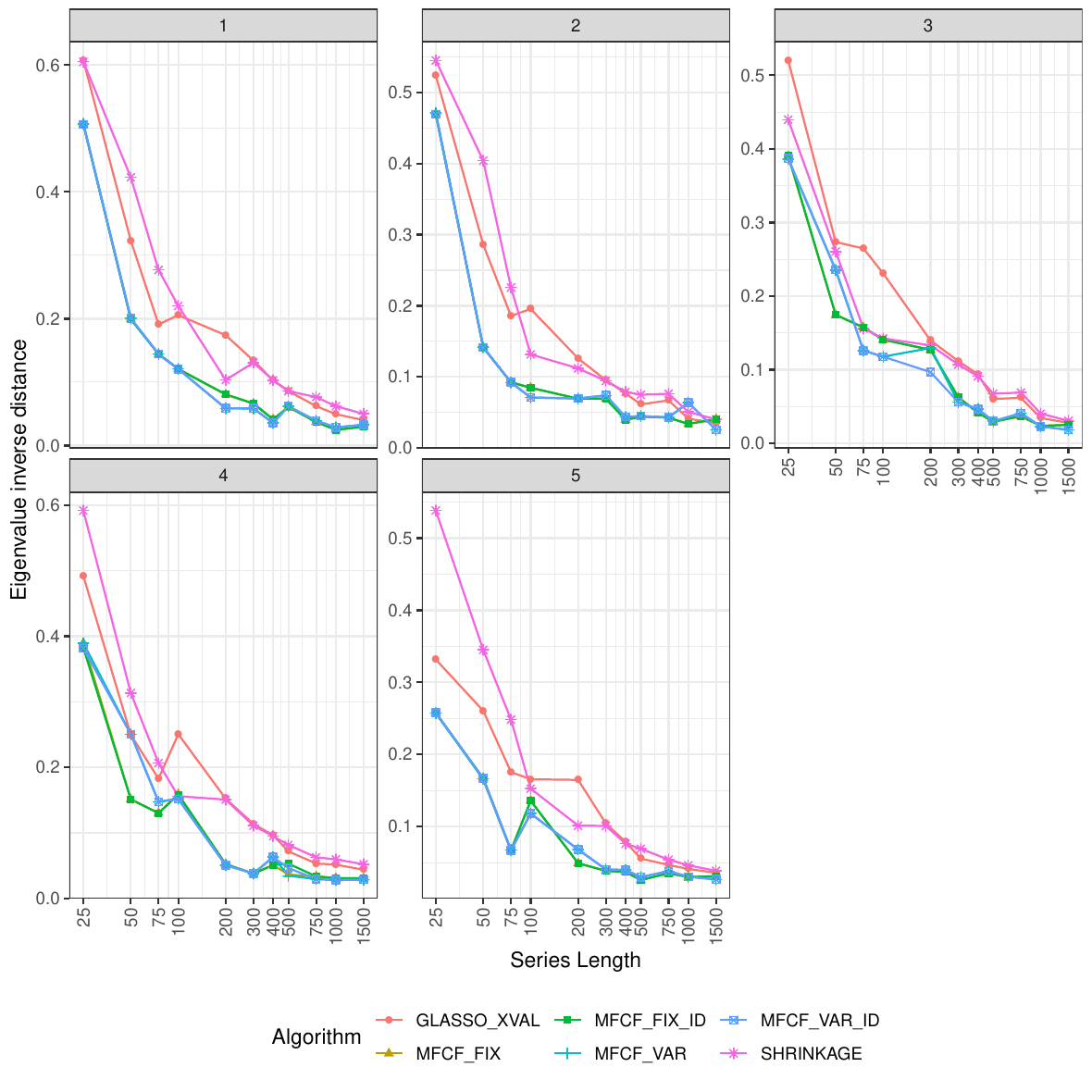}
 \caption{Inverse eigenvalue distance for synthetic data (sparse decomposable precision matrix).  The five panels show the values at different time series lengths for 5 training  / validation datasets.}
  \label{fig:ex:Chordal:3}
\end{figure} 

Figure \ref{fig:ex:Chordal:4} shows the number of cliques of different size produced by the \mfcfvar{} algorithm as a function of the maximum allowed clique size and of the time series length. We note that as the time series length increases the test becomes  less stringent with a higher number of large cliques in the model; conversely, when the time series is shorter ($n < p$), the models produced are more parsimonious. The number of cliques of size smaller than the maximum is linked to the degree of sparsity of the model. We will see in Section \ref{sec:factor:model:noise} that in the case of systems that are inherently dense the vast majority of the cliques will have the maximum allowed clique size.

\begin{figure}
\centering
\includegraphics[scale=0.8]{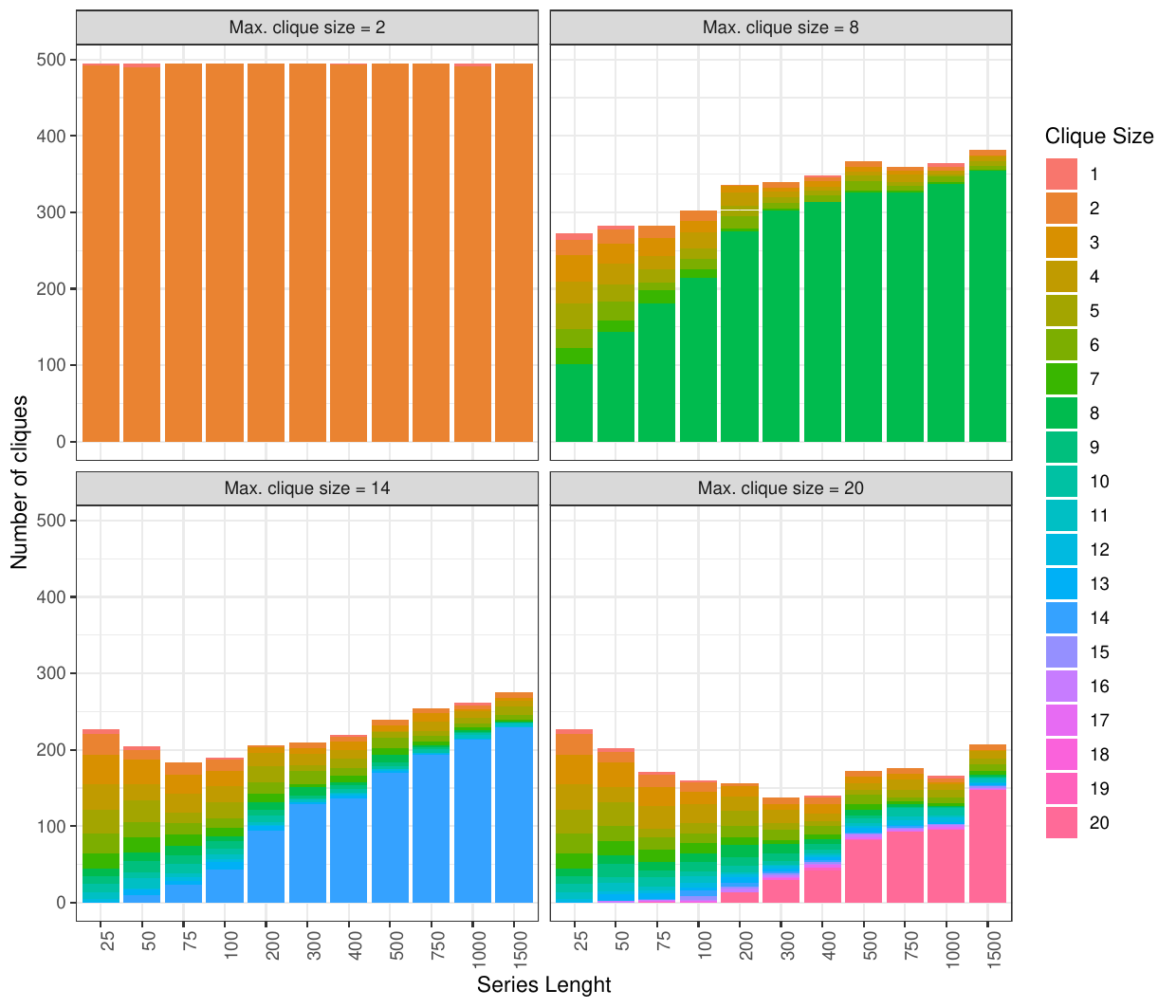}
 \caption{Composition of cliques from \mfcfvar{} for synthetic data (sparse decomposable precision matrix).   The statistics is based on a total of 5 different training / validation sets.}
  \label{fig:ex:Chordal:4}
\end{figure}

\FloatBarrier
\subsubsection{Synthetic data: Full Positive Definite Matrix from package ``clusterGeneration''. } \label{sec:posdef:matrxix:results}

In this sub-section and in the next two we repeat on different datasets all the analyses described in the previous subsection \ref{sec:random:clique:forest:res}.
Figure \ref{fig:ex:ClusterGen} displays the log-likelihood of the models. We observe that \mfcf{} algorithms perform overall better than either \glasso{} or \shrinkage{}, but is worth noting the overall low level of the log-likelihood for all models. In particular the \glasso{} performs worse than the null hypothesis (which has log-likelihood of 5000) for short time series. From Table \ref{tab:ClusterGen:shrink} we observe that the penalty or shrinkage parameters decrease but they retain higher overall values than in the other examples.

\begin{table}[ht]
\centering
\begingroup\footnotesize
\begin{tabular}{|r|C{2cm}|C{2cm}|C{2cm}|C{2cm}|C{2cm}|C{2cm}|}
  \hline
\thead{Series \\ length} & \thead{GLASSO \\ XVAL} & \thead{MFCF \\ FIX} & \thead{MFCF \\ FIX ID} & \thead{MFCF \\ VAR} & \thead{MFCF \\ VAR ID} & SHRINKAGE \\ 
  \hline
 25 & 0.33 & 0.97 & 0.95 & 0.99 & 0.99 & 0.99 \\ 
   50 & 0.24 & 0.97 & 0.89 & 0.97 & 0.91 & 0.99 \\ 
   75 & 0.22 & 0.94 & 0.88 & 0.95 & 0.87 & 0.95 \\ 
  100 & 0.17 & 0.88 & 0.83 & 0.90 & 0.83 & 0.93 \\ 
  200 & 0.12 & 0.85 & 0.75 & 0.86 & 0.77 & 0.89 \\ 
  300 & 0.12 & 0.73 & 0.64 & 0.72 & 0.62 & 0.86 \\ 
  400 & 0.12 & 0.63 & 0.55 & 0.63 & 0.54 & 0.84 \\ 
  500 & 0.12 & 0.56 & 0.48 & 0.54 & 0.46 & 0.80 \\ 
  750 & 0.10 & 0.40 & 0.36 & 0.39 & 0.38 & 0.76 \\ 
  1000 & 0.08 & 0.34 & 0.31 & 0.32 & 0.28 & 0.71 \\ 
  1500 & 0.05 & 0.20 & 0.18 & 0.20 & 0.18 & 0.62 \\ 
   \hline
\end{tabular}
\endgroup
\caption{Mean penalty/shrinkage parameter by length of time series. The statistics is based on a total of 5 different training / validation sets per length of the time series.} 
\label{tab:ClusterGen:shrink}
\end{table}

Table \ref{tab:ClusterGen:param_no} shows the number of non zero elements in the precision matrix for every length of the time series. We note that the statistically validated methods \mfcfvar{} and \mfcfvari{} produce consistently sparser models, without significant deterioration on the performance in terms of log-likelihood or correlation.

\begin{table}[ht]
\centering
\begingroup\footnotesize
\begin{tabular}{|r|R{2cm}|R{2cm}|R{2cm}|R{2cm}|R{2cm}|R{2cm}|}
  \hline
\thead{Series \\ length} & \thead{GLASSO \\ XVAL} & \thead{MFCF \\ FIX} & \thead{MFCF \\ FIX ID} & \thead{MFCF \\ VAR} & \thead{MFCF \\ VAR ID} & SHRINKAGE \\ 
  \hline
25 & 484 & 1062 & 1164 & 234 & 36 & 4950 \\ 
  50 & 459 & 465 & 679 & 250 & 276 & 4950 \\ 
  75 & 343 & 592 & 1129 & 255 & 353 & 4950 \\ 
  100 & 472 & 555 & 757 & 247 & 308 & 4950 \\ 
  200 & 472 & 351 & 687 & 238 & 301 & 4950 \\ 
  300 & 260 & 352 & 466 & 272 & 382 & 4950 \\ 
  400 & 155 & 331 & 369 & 272 & 343 & 4950 \\ 
  500 & 108 & 294 & 351 & 297 & 364 & 4950 \\ 
  750 & 165 & 370 & 408 & 330 & 428 & 4950 \\ 
  1000 & 171 & 313 & 427 & 292 & 336 & 4950 \\ 
  1500 & 922 & 313 & 313 & 282 & 349 & 4950 \\ 
   \hline
\end{tabular}
\endgroup
\caption{Mean number of non-zero coefficient in the precision matrix by length of time series. The statistics is based on a total of 5 different training / validation sets per length of the time series.} 
\label{tab:ClusterGen:param_no}
\end{table}

\begin{figure} 
\centering
\includegraphics[scale=0.8]{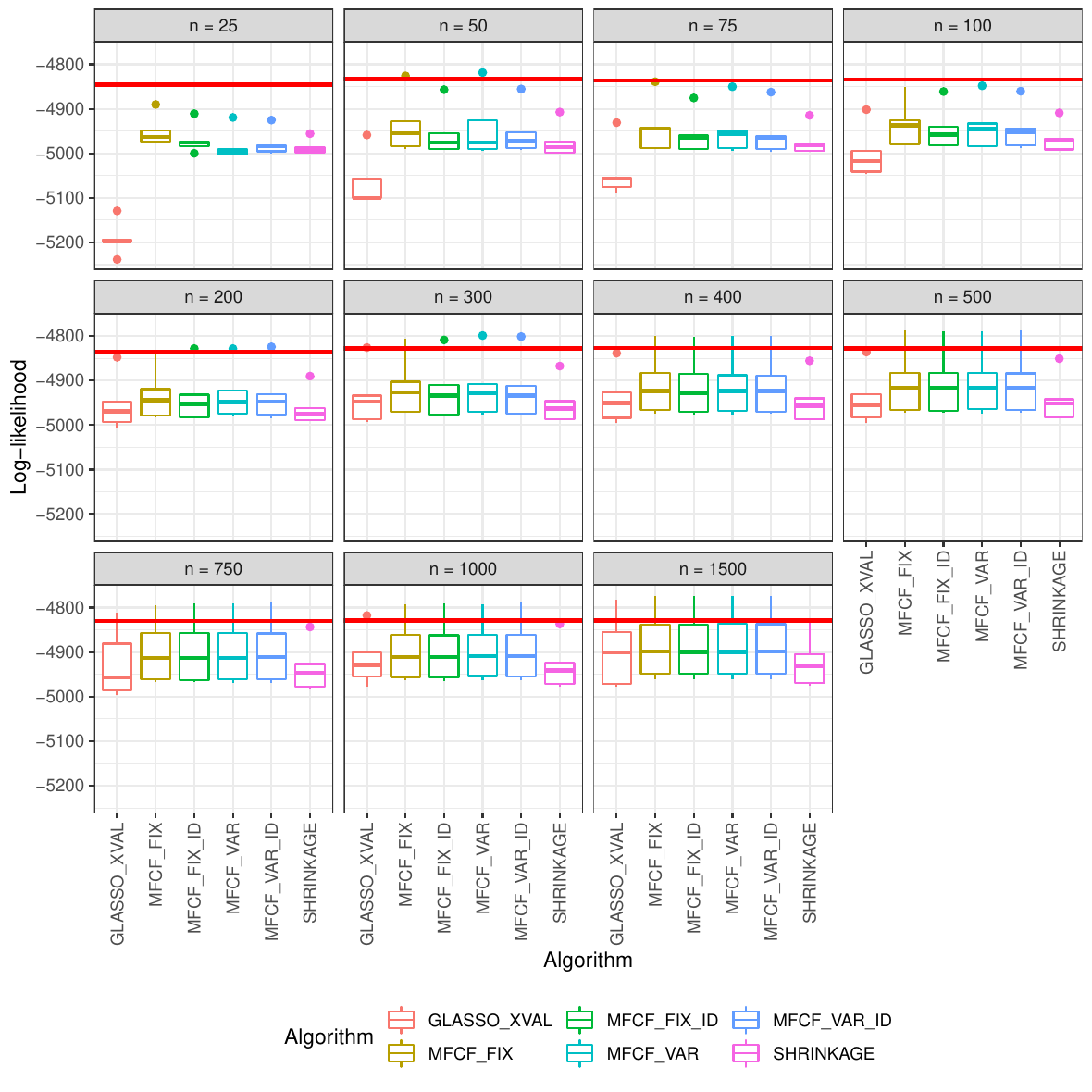}
  \caption{Box plot for the likelihood of the algorithms on synthetic data (random positive definite matrix generated by ClusterGen) for different lengths of the series. The statistics is based on a total of 50 test sets (10 test sets for each of 5 different training / validation sets) per length of the time series.}
  \label{fig:ex:ClusterGen}
\end{figure} 

The measures of performance are reported in Figures \ref{fig:ex:ClusterGen}-\ref{fig:ex:ClusterGen:4}. We note that the \mfcffix{} and \mfcffixi{} are more accurate for short time series as they pick up many more matrix elements than the validated methods, but this does not translate in improvements for the other measures of performance. 
The \mfcf{} methods seem to perform better than \glasso{} and \shrinkage{} also when it comes to distance of the eigenvalues, especially with short time series.
Interestingly, the composition of the clique structure produced by the \mfcfvar{} shown in Figure \ref{fig:ex:ClusterGen:4}suggests that even for medium and long time series the algorithm produces a mostly small or large cliques with only a small fraction of cliques with intermediate sizes.

\begin{figure}
\centering
\includegraphics[scale=0.8]{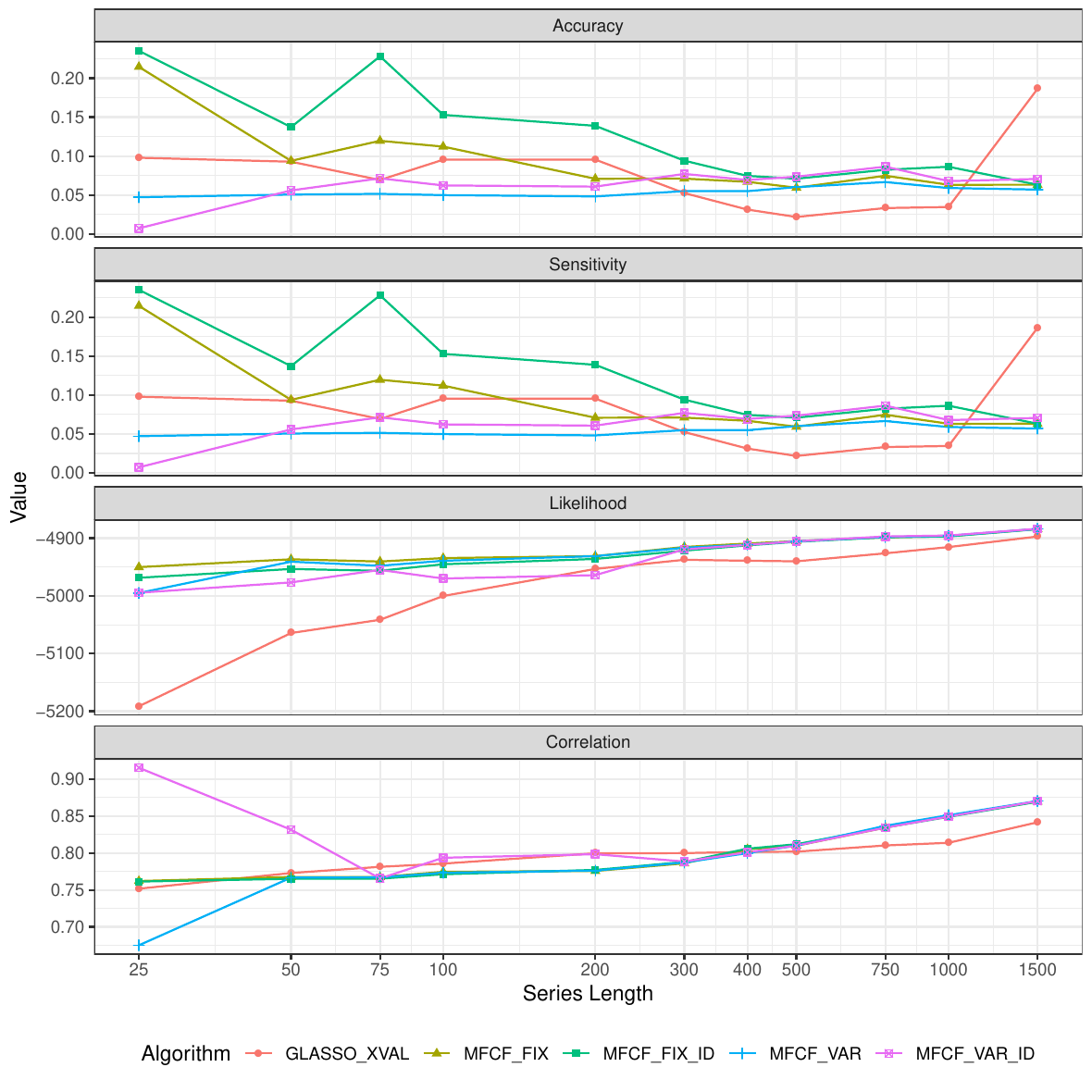}
 \caption{Performance of the algorithms on synthetic data (random positive definite matrix generated by ClusterGen) for different lengths of the series. The statistics is based on a total of 50 test sets (10 test sets for each of 5 different training / validation sets).}
  \label{fig:ex:ClusterGen:1}
\end{figure}

\begin{figure}
\centering
\includegraphics[scale=0.8]{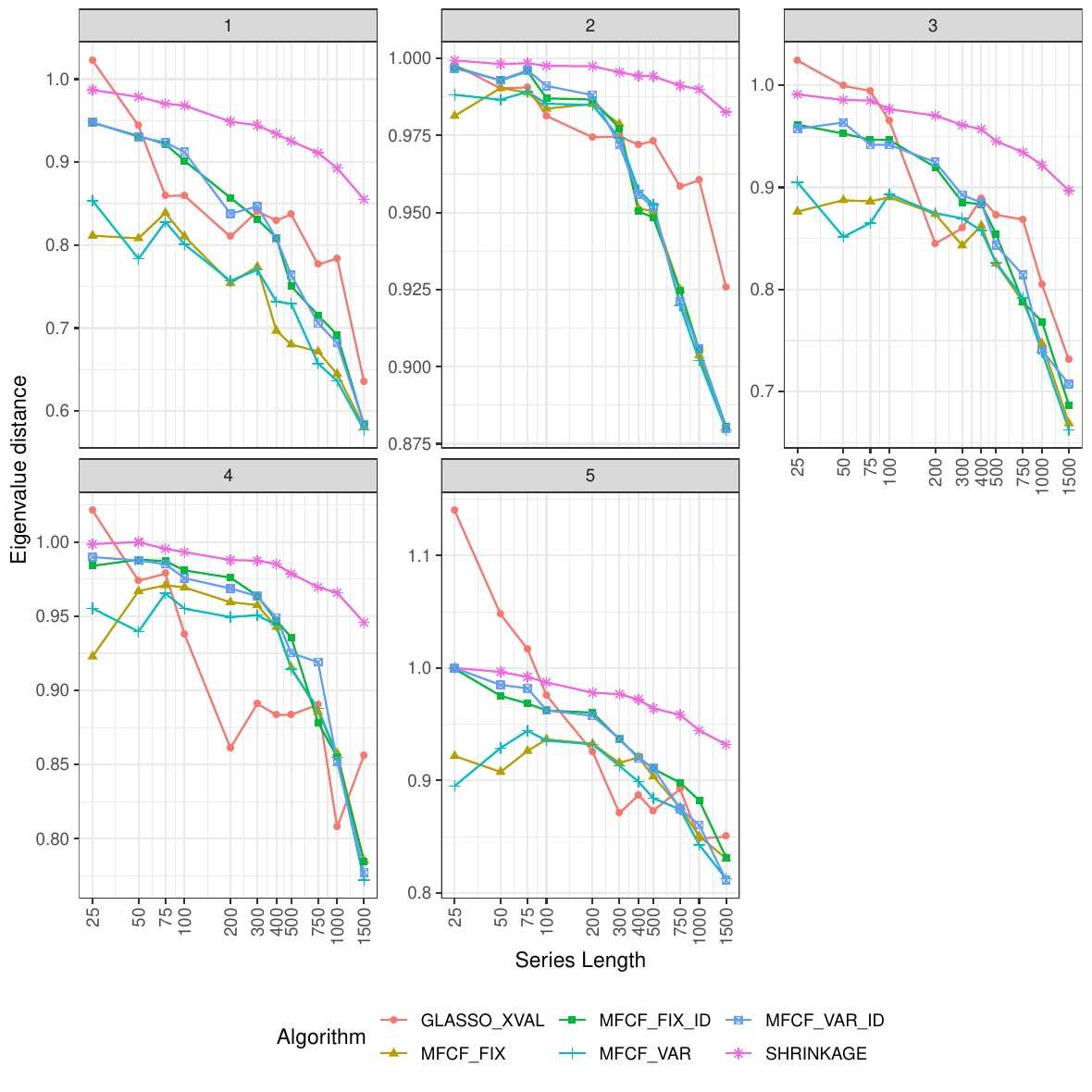}
 \caption{Eigenvalue distance for synthetic data (random positive definite matrix generated by ClusterGen). The five panels show the values at different time series lengths for 5 training  / validation datasets.}
  \label{fig:ex:ClusterGen:2}
\end{figure}

\begin{figure}
\centering
\includegraphics[scale=0.8]{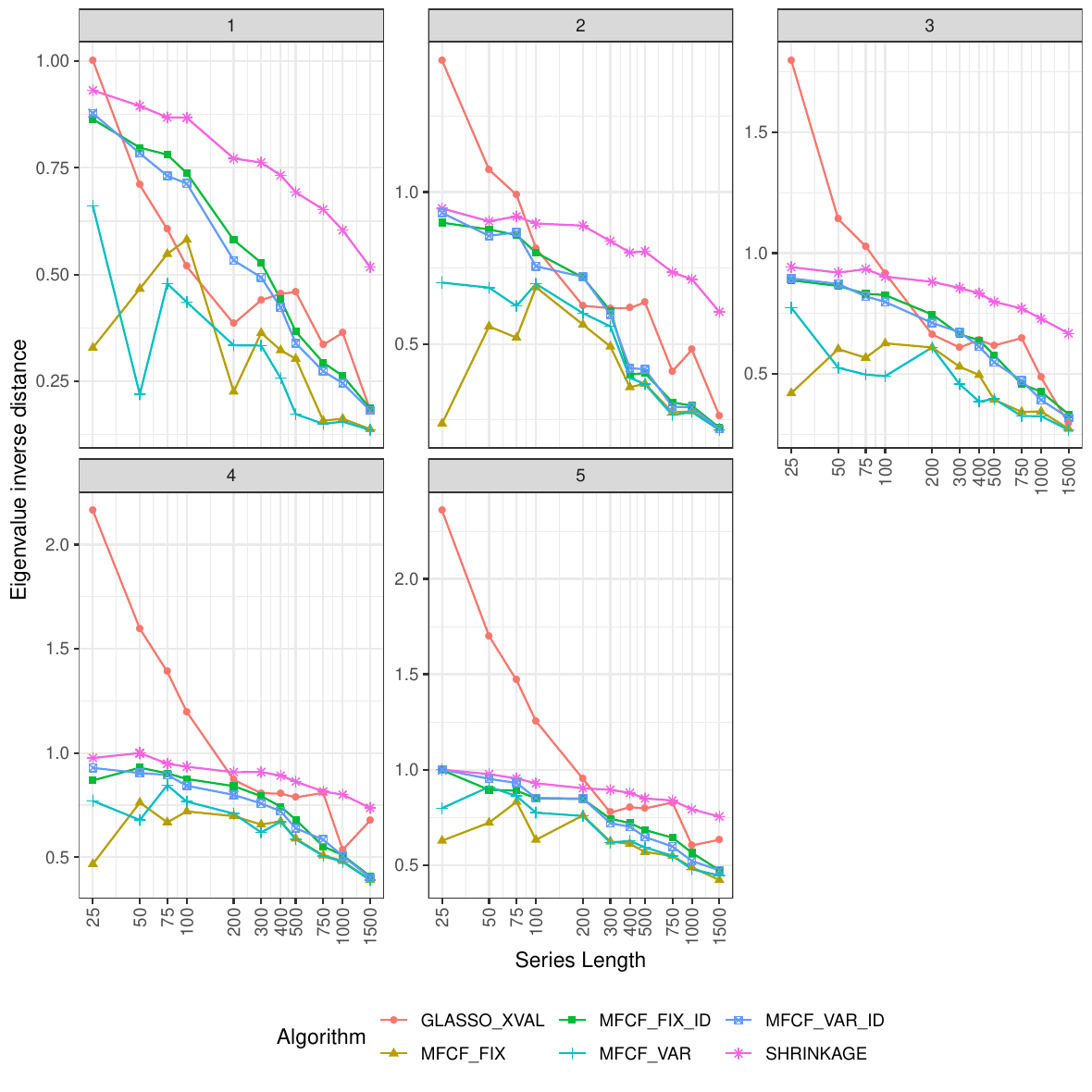}
 \caption{Inverse eigenvalue distance for synthetic data (random positive definite matrix generated by ClusterGen). The five panels show the values at different time series lengths for 5 training  / validation datasets.}
  \label{fig:ex:ClusterGen:3}
\end{figure}

\begin{figure}
\centering
\includegraphics[scale=0.8]{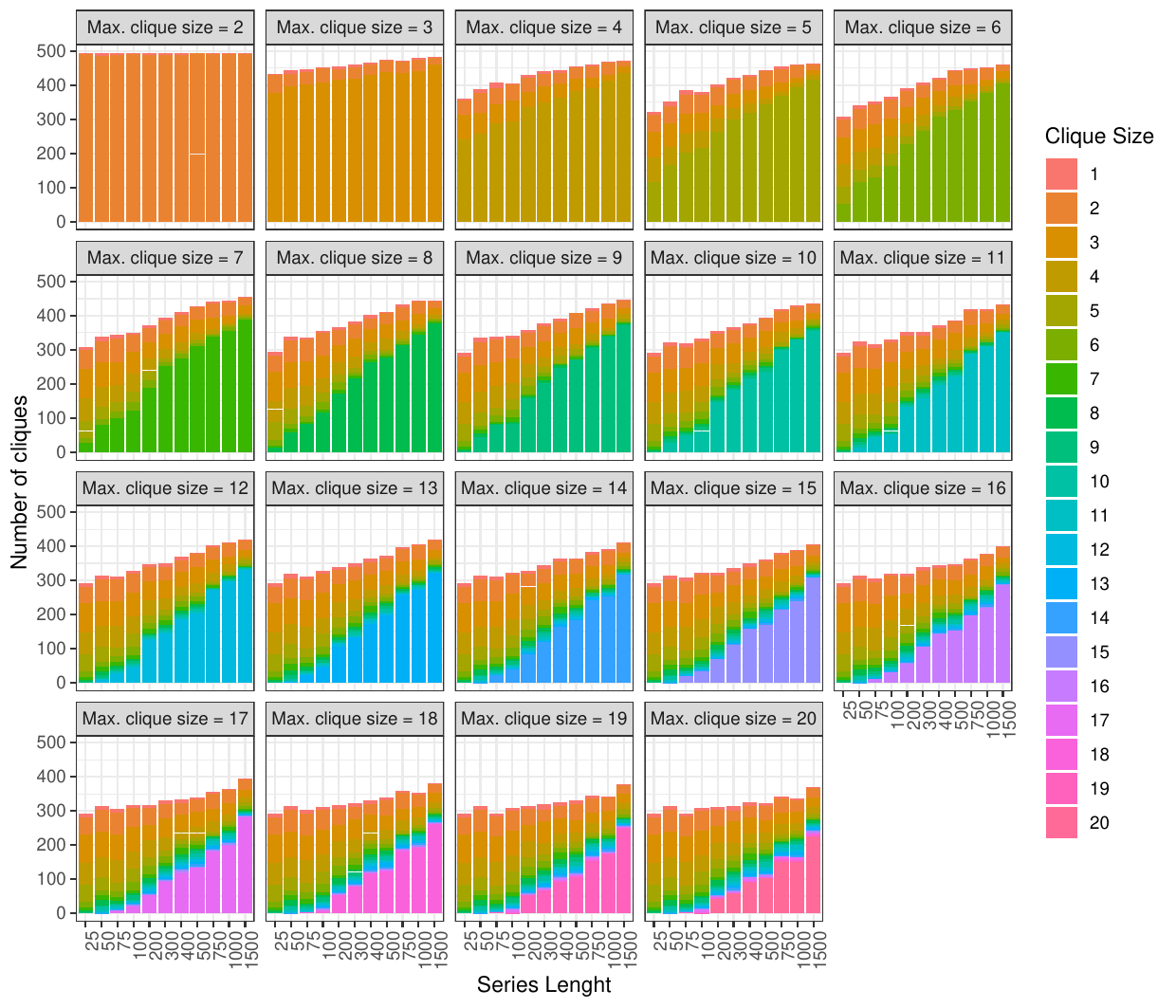}
 \caption{Composition of cliques from \mfcfvar{} for synthetic data (random positive definite matrix generated by ClusterGen). The statistics is based on a total of 5 different training / validation sets per length of the time series.}
  \label{fig:ex:ClusterGen:4}
\end{figure}

\FloatBarrier
\subsubsection{Random Factor Model with noise} \label{sec:factor:model:noise}

Performance measures are reported in Figures \ref{fig:ex:Factor}-\ref{fig:ex:Factor:4}. As discussed in Section \ref{sec:Data}, the loadings to the 5 factors are different from zero for every time series, and therefore the model is not suitable for local algorithms such as the \mfcf{}; this would probably explain why in this instance the \glasso{} performs better in terms of almost all measures. 
Table \ref{tab:Factor:shrink} shows that the behaviour of the shrinkage or penalty parameters are decreasing with series length as expected. 
Table \ref{tab:Factor:param_no} highlights how all the models tend to use as many parameters as possible, consistently with the constraints imposed on penalty and clique size.
Figure \ref{fig:ex:Factor:4} supports the idea that the underlying model is non local, since the validated methods tend to use exclusively the largest cliques allowed by the constraints. This suggests that the analysis of the clique sizes might provide insight into the sparsity of the data set, when the data generation process is not known.


\begin{table}[ht]
\centering
\begingroup\footnotesize
\begin{tabular}{|r|C{2cm}|C{2cm}|C{2cm}|C{2cm}|C{2cm}|C{2cm}|}
  \hline
\thead{Series \\ length} & \thead{GLASSO \\ XVAL} & \thead{MFCF \\ FIX} & \thead{MFCF \\ FIX ID} & \thead{MFCF \\ VAR} & \thead{MFCF \\ VAR ID} & SHRINKAGE \\ 
  \hline
 25 & 0.12 & 0.27 & 0.26 & 0.25 & 0.25 & 0.30 \\ 
   50 & 0.12 & 0.19 & 0.18 & 0.18 & 0.18 & 0.23 \\ 
   75 & 0.09 & 0.15 & 0.14 & 0.15 & 0.15 & 0.20 \\ 
  100 & 0.06 & 0.13 & 0.13 & 0.12 & 0.12 & 0.18 \\ 
  200 & 0.02 & 0.08 & 0.08 & 0.08 & 0.08 & 0.14 \\ 
  300 & 0.01 & 0.06 & 0.06 & 0.06 & 0.06 & 0.11 \\ 
  400 & 0.01 & 0.05 & 0.05 & 0.05 & 0.05 & 0.09 \\ 
  500 & 0.01 & 0.04 & 0.04 & 0.04 & 0.04 & 0.08 \\ 
  750 & 0.01 & 0.03 & 0.03 & 0.03 & 0.03 & 0.06 \\ 
  1000 & 0.01 & 0.02 & 0.02 & 0.02 & 0.02 & 0.04 \\ 
  1500 & 0.01 & 0.01 & 0.01 & 0.01 & 0.01 & 0.04 \\ 
   \hline
\end{tabular}
\endgroup
\caption{Mean penalty/shrinkage parameter by length of time series. The statistics is based on a total of 5 different training / validation sets per length of the time series.} 
\label{tab:Factor:shrink}
\end{table}


\begin{table}[ht]
\centering
\begingroup\footnotesize
\begin{tabular}{|r|R{2cm}|R{2cm}|R{2cm}|R{2cm}|R{2cm}|R{2cm}|}
  \hline
\thead{Series \\ length} & \thead{GLASSO \\ XVAL} & \thead{MFCF \\ FIX} & \thead{MFCF \\ FIX ID} & \thead{MFCF \\ VAR} & \thead{MFCF \\ VAR ID} & SHRINKAGE \\ 
  \hline
25 & 1168 & 1694 & 1661 & 1582 & 1597 & 4950 \\ 
  50 & 1206 & 1677 & 1661 & 1694 & 1694 & 4950 \\ 
  75 & 1308 & 1661 & 1661 & 1678 & 1678 & 4950 \\ 
  100 & 1648 & 1710 & 1710 & 1645 & 1645 & 4950 \\ 
  200 & 2143 & 1710 & 1710 & 1678 & 1694 & 4950 \\ 
  300 & 2426 & 1710 & 1694 & 1694 & 1694 & 4950 \\ 
  400 & 2568 & 1710 & 1710 & 1710 & 1710 & 4950 \\ 
  500 & 2814 & 1694 & 1694 & 1694 & 1694 & 4950 \\ 
  750 & 2760 & 1694 & 1694 & 1694 & 1694 & 4950 \\ 
  1000 & 2755 & 1710 & 1710 & 1678 & 1678 & 4950 \\ 
  1500 & 2761 & 1694 & 1694 & 1710 & 1710 & 4950 \\ 
   \hline
\end{tabular}
\endgroup
\caption{Mean number of non-zero coefficient in the precision matrix by length of time series. The statistics is based on a total of 5 different training / validation sets per length of the time series.} 
\label{tab:Factor:param_no}
\end{table}

\begin{figure} 
\centering
\includegraphics[scale=0.8]{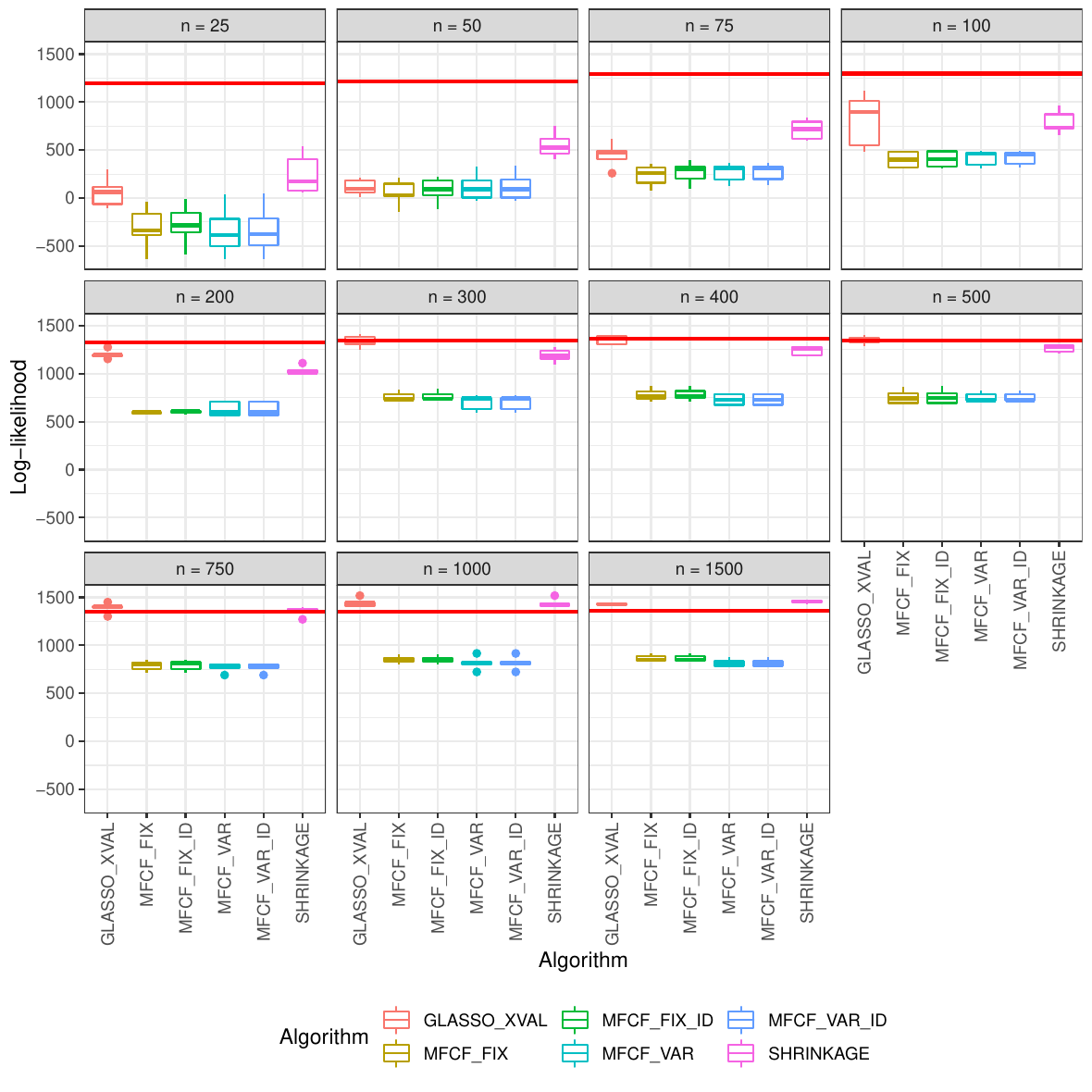}
  \caption{Box plot for the likelihood of the algorithms on synthetic data (factor model) for different lengths of the series. The statistics is based on a total of 50 test sets (10 test sets for each of 5 different training / validation sets) per length of the time series.}
  \label{fig:ex:Factor}
\end{figure}

\begin{figure}
\centering
\includegraphics[scale=0.8]{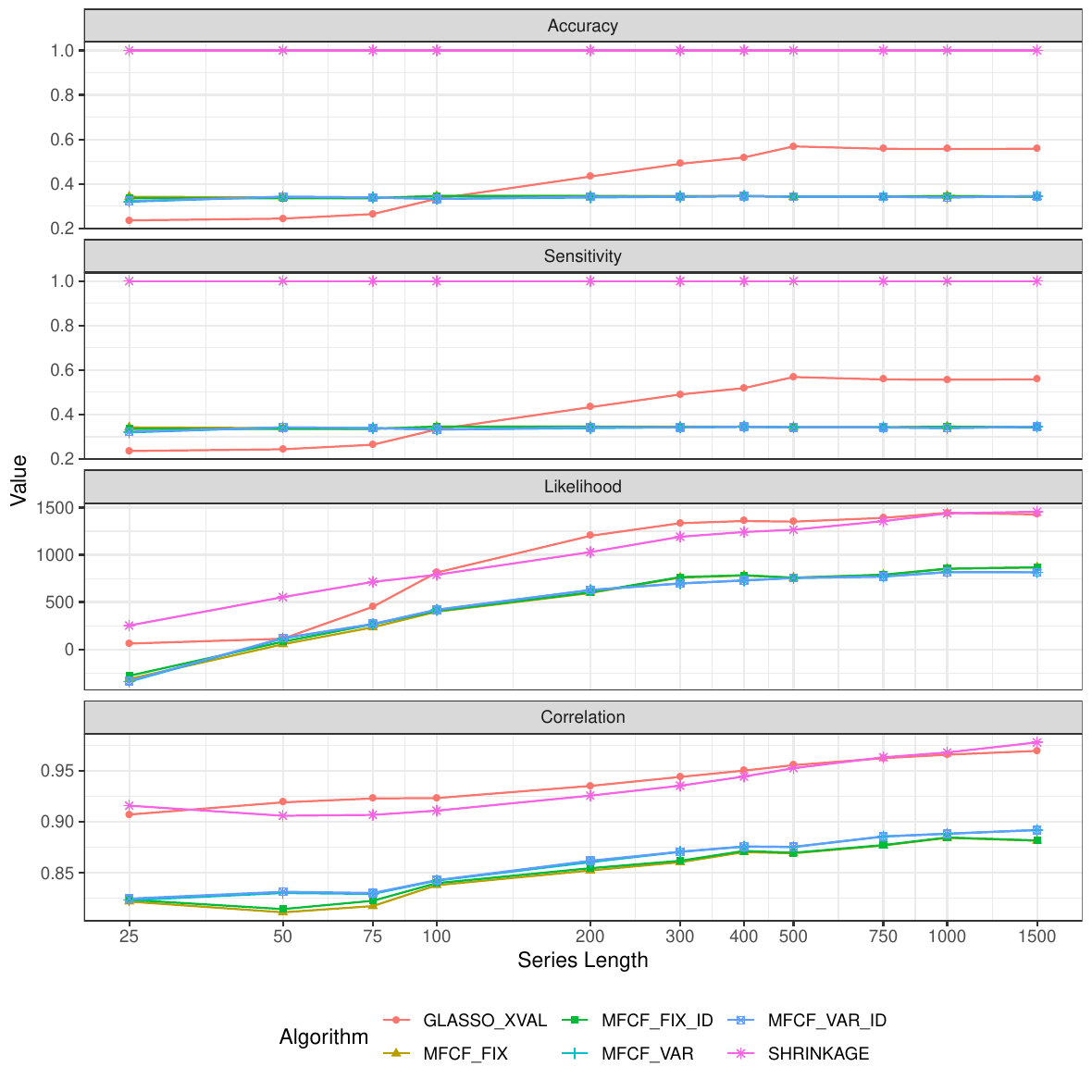}
 \caption{Performance of the algorithms on synthetic data (factor model)  for different lengths of the series. The statistics is based on a total of 50 test sets (10 test sets for each of 5 different training / validation sets) per length of the time series.}
  \label{fig:ex:Factor:1}
\end{figure}

\begin{figure}
\centering
\includegraphics[scale=0.8]{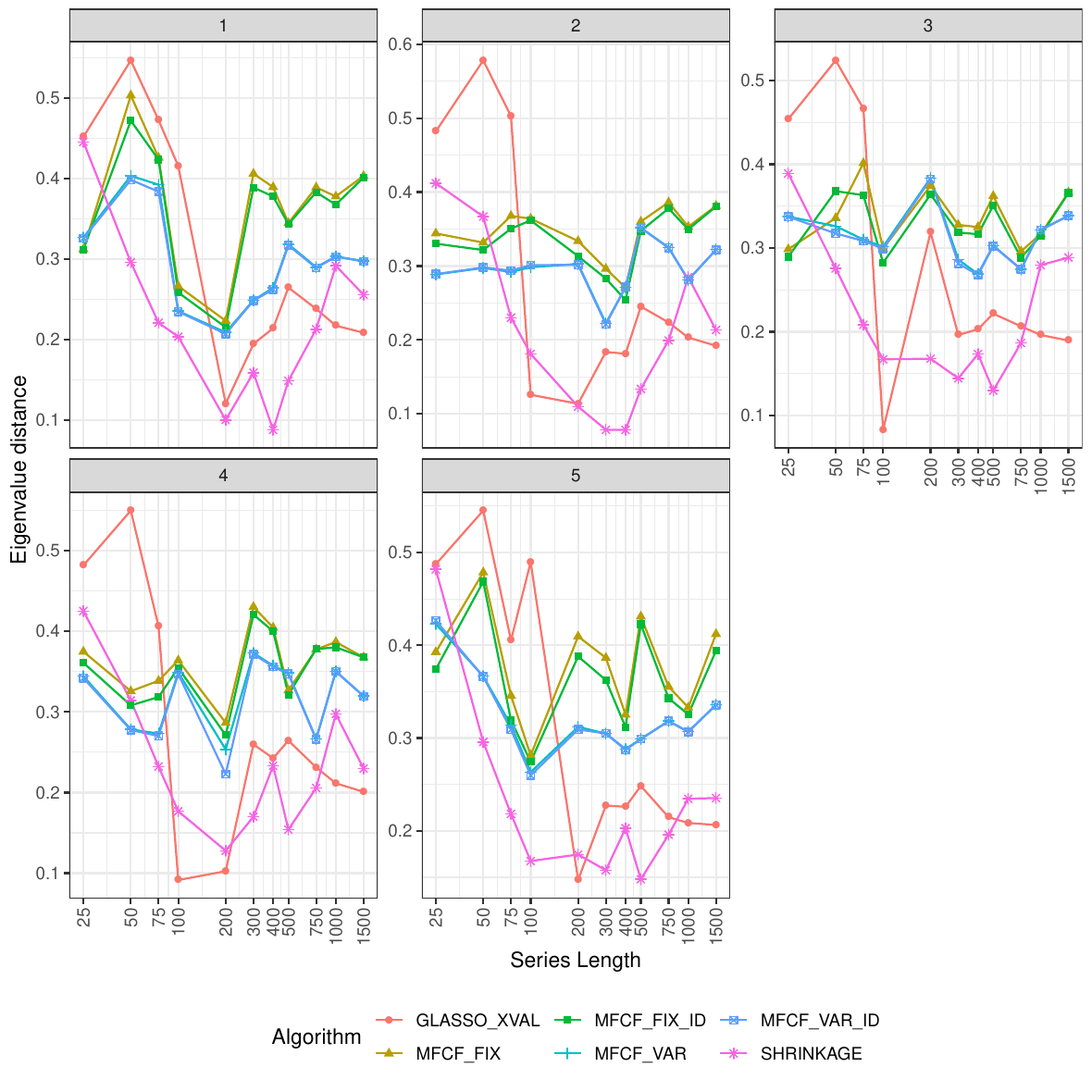}
 \caption{Eigenvalue distance for synthetic data (factor model). The five panels show the values at different time series lengths for 5 training  / validation datasets.}
  \label{fig:ex:Factor:2}
\end{figure}

\begin{figure}
\centering
\includegraphics[scale=0.8]{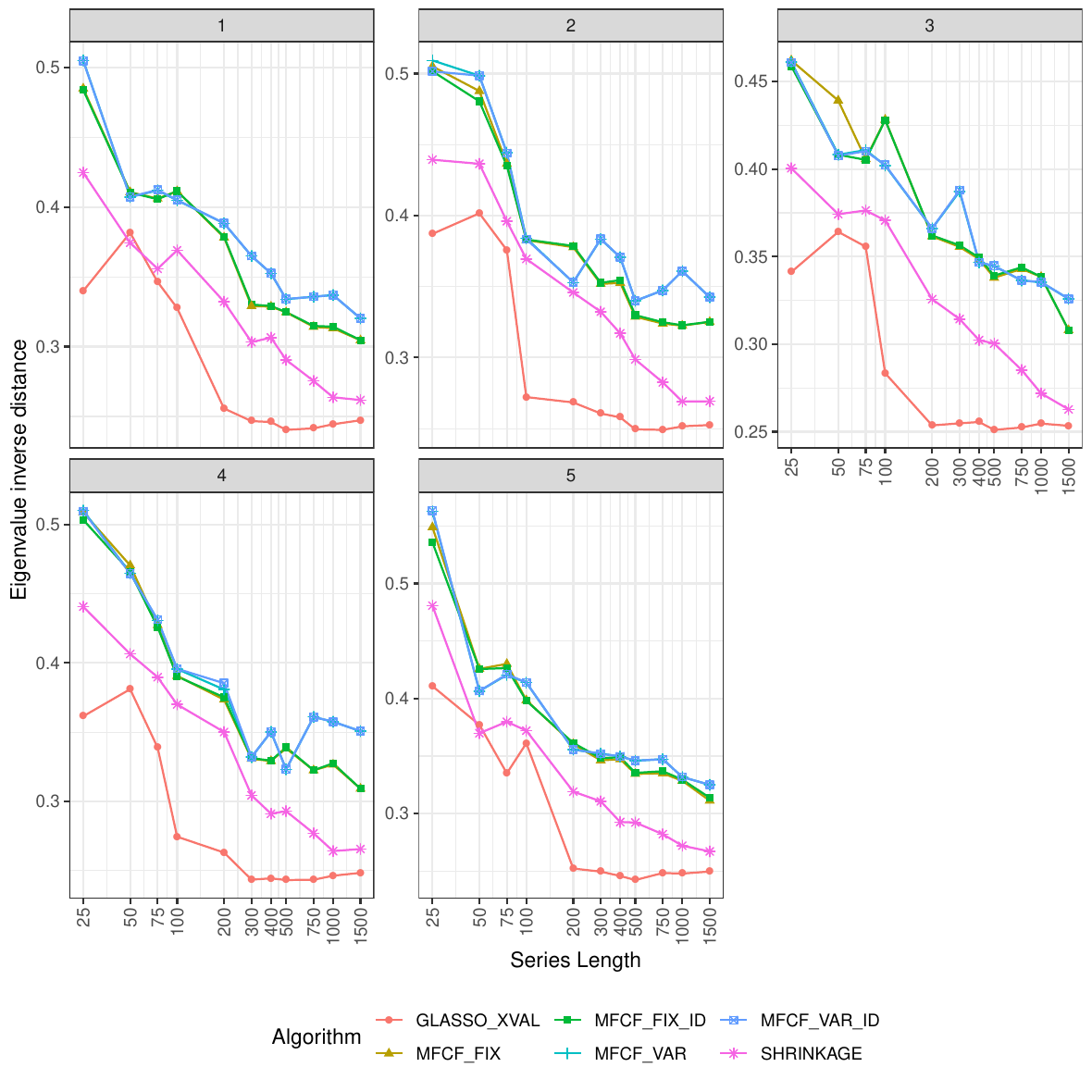}
 \caption{Inverse eigenvalue distance for synthetic data (factor model). The five panels show the values at different time series lengths for 5 training  / validation datasets.}
  \label{fig:ex:Factor:3}
\end{figure}

\begin{figure}
\centering
\includegraphics[scale=0.8]{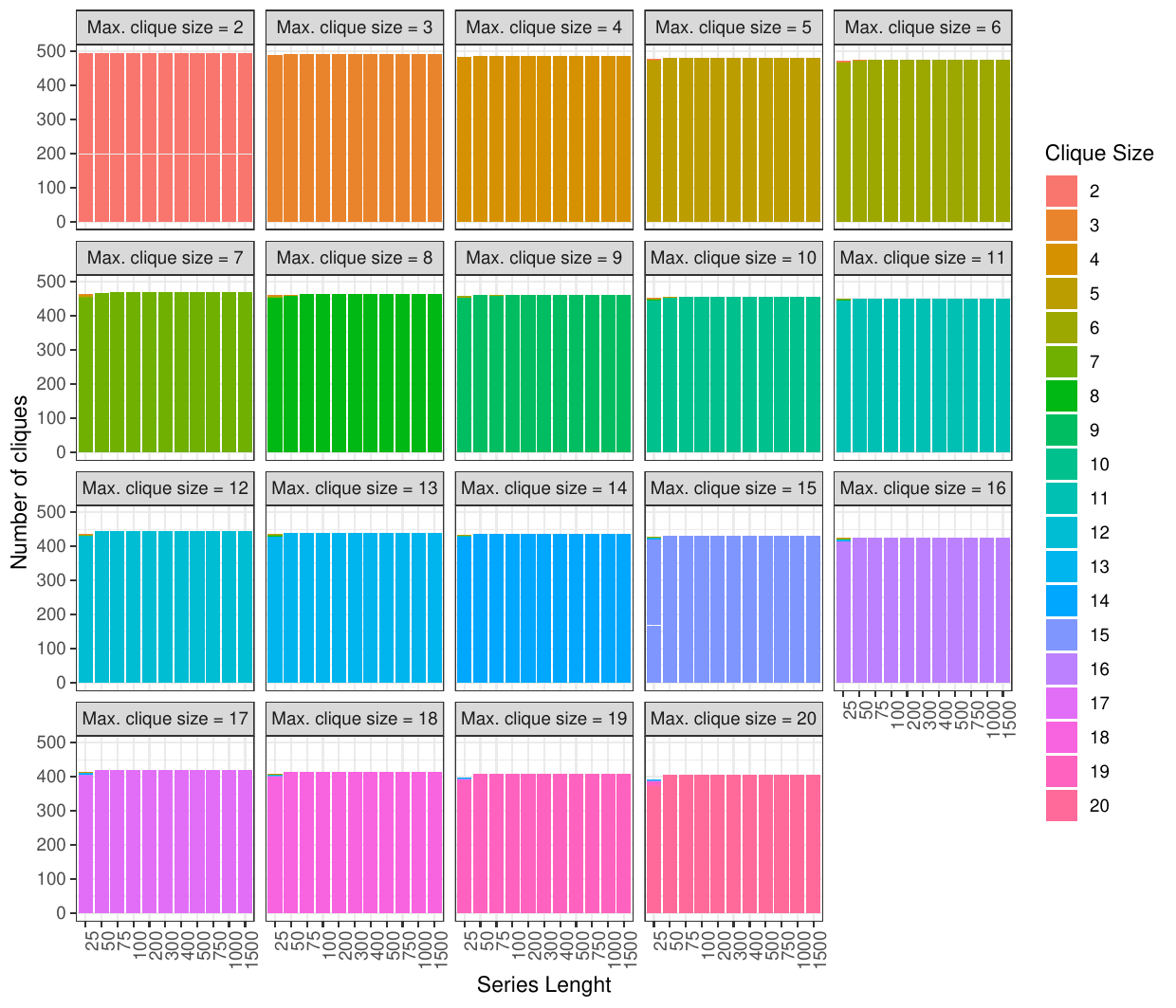}
 \caption{Composition of cliques from \mfcfvar{} for synthetic data (factor model). The statistics is based on a total of 5 different training / validation sets per length of the time series.}
  \label{fig:ex:Factor:4}
\end{figure} 

\FloatBarrier
\subsubsection{Real Data} \label{sec:exp:real:data:res}

Performance measures are reported in Figures \ref{fig:ex:Real}-\ref{fig:ex:Real:4}. We see from the inspection of Figures \ref{fig:ex:Real} and \ref{fig:ex:Real:1} that with real data the log-likelihood is comparable across all models, with slight better values for \mfcffix{} and \mfcfvar{} for shorter time series. It is worth noting that, in the family of the \mfcf{} algorithms, the two that use the clique tree shrinkage target described in Equation \ref{ml:inverse:cost:corr} (\mfcffix{} and \mfcfvar{}) perform significantly better, for short time series, than the models with the same structure but the simpler identity matrix (\mfcffixi{} and \mfcfvari{}) as a shrinkage target.
Table \ref{tab:Real:shrink} shows that the penalty and shrinkage parameters decrease, as expected, with the length of the time series.
Table \ref{tab:Real:param_no} shows that from time series lengths above 500 the \glasso{} produces matrices with a significantly higher number of parameters different from zero, including almost 50\% of the total number of elements. 
The growth in performance for \mfcf{} family of algorithms is in this case constrained by the maximum clique size.
As Figure \ref{fig:ex:Real:2} shows the \mfcf{} algorithms performs better in the approximation of the eigenvalues of the precision matrix, and slightly worse (Figure \ref{fig:ex:Real:3} in the representation of the eigenvalues of the inverse precision matrix. 
In this experiment, since we don't know the ``true'' correlation matrix we have used the maximum likelihood estimate of the correlation matrix over the full time series.
Overall the results for this dataset demonstrate that \mfcffix{} and \mfcfvar{} are the best performer.
The analysis in Figure \ref{fig:ex:Real:4} showing the composition of the cliques of different sizes shows that the synthetic model closest to the real data is the factor model (see \ref{sec:factor:model:noise}). 


\begin{table}[ht]
\centering
\begingroup\footnotesize
\begin{tabular}{|r|C{2cm}|C{2cm}|C{2cm}|C{2cm}|C{2cm}|C{2cm}|}
  \hline
\thead{Series \\ length} & \thead{GLASSO \\ XVAL} & \thead{MFCF \\ FIX} & \thead{MFCF \\ FIX ID} & \thead{MFCF \\ VAR} & \thead{MFCF \\ VAR ID} & SHRINKAGE \\ 
  \hline
 25 & 0.21 & 0.81 & 0.57 & 0.80 & 0.54 & 0.67 \\ 
   50 & 0.15 & 0.73 & 0.47 & 0.70 & 0.47 & 0.63 \\ 
   75 & 0.12 & 0.62 & 0.35 & 0.61 & 0.36 & 0.55 \\ 
  100 & 0.12 & 0.56 & 0.31 & 0.54 & 0.33 & 0.54 \\ 
  200 & 0.12 & 0.43 & 0.26 & 0.44 & 0.26 & 0.44 \\ 
  300 & 0.10 & 0.32 & 0.20 & 0.34 & 0.18 & 0.39 \\ 
  400 & 0.08 & 0.28 & 0.17 & 0.28 & 0.17 & 0.34 \\ 
  500 & 0.07 & 0.25 & 0.15 & 0.24 & 0.15 & 0.30 \\ 
  750 & 0.03 & 0.18 & 0.12 & 0.19 & 0.13 & 0.24 \\ 
  1000 & 0.02 & 0.13 & 0.08 & 0.13 & 0.08 & 0.17 \\ 
  1500 & 0.01 & 0.11 & 0.07 & 0.11 & 0.07 & 0.14 \\ 
   \hline
\end{tabular}
\endgroup
\caption{Mean penalty/shrinkage parameter by length of time series. The statistics is based on a total of 5 different training / validation sets per length of the time series.} 
\label{tab:Real:shrink}
\end{table}


\begin{table}[ht]
\centering
\begingroup\footnotesize
\begin{tabular}{|r|R{2cm}|R{2cm}|R{2cm}|R{2cm}|R{2cm}|R{2cm}|}
  \hline
\thead{Series \\ length} & \thead{GLASSO \\ XVAL} & \thead{MFCF \\ FIX} & \thead{MFCF \\ FIX ID} & \thead{MFCF \\ VAR} & \thead{MFCF \\ VAR ID} & SHRINKAGE \\ 
  \hline
25 & 950 & 1561 & 1478 & 1206 & 1055 & 4950 \\ 
  50 & 1039 & 1512 & 1010 & 1286 & 1018 & 4950 \\ 
  75 & 1099 & 1628 & 798 & 1416 & 850 & 4950 \\ 
  100 & 1063 & 1545 & 1053 & 1310 & 1060 & 4950 \\ 
  200 & 1057 & 1363 & 1152 & 1443 & 1094 & 4950 \\ 
  300 & 1170 & 1202 & 1049 & 1375 & 888 & 4950 \\ 
  400 & 1191 & 1359 & 1154 & 1325 & 1255 & 4950 \\ 
  500 & 1464 & 1578 & 1358 & 1477 & 1360 & 4950 \\ 
  750 & 1994 & 1495 & 1495 & 1562 & 1562 & 4950 \\ 
  1000 & 2404 & 1595 & 1595 & 1645 & 1645 & 4950 \\ 
  1500 & 2752 & 1710 & 1677 & 1645 & 1629 & 4950 \\ 
   \hline
\end{tabular}
\endgroup
\caption{Mean number of non-zero coefficient in the precision matrix by length of time series. The statistics is based on a total of 5 different training / validation sets per length of the time series.} 
\label{tab:Real:param_no}
\end{table}

\begin{figure} 
\centering
\includegraphics[scale=0.8]{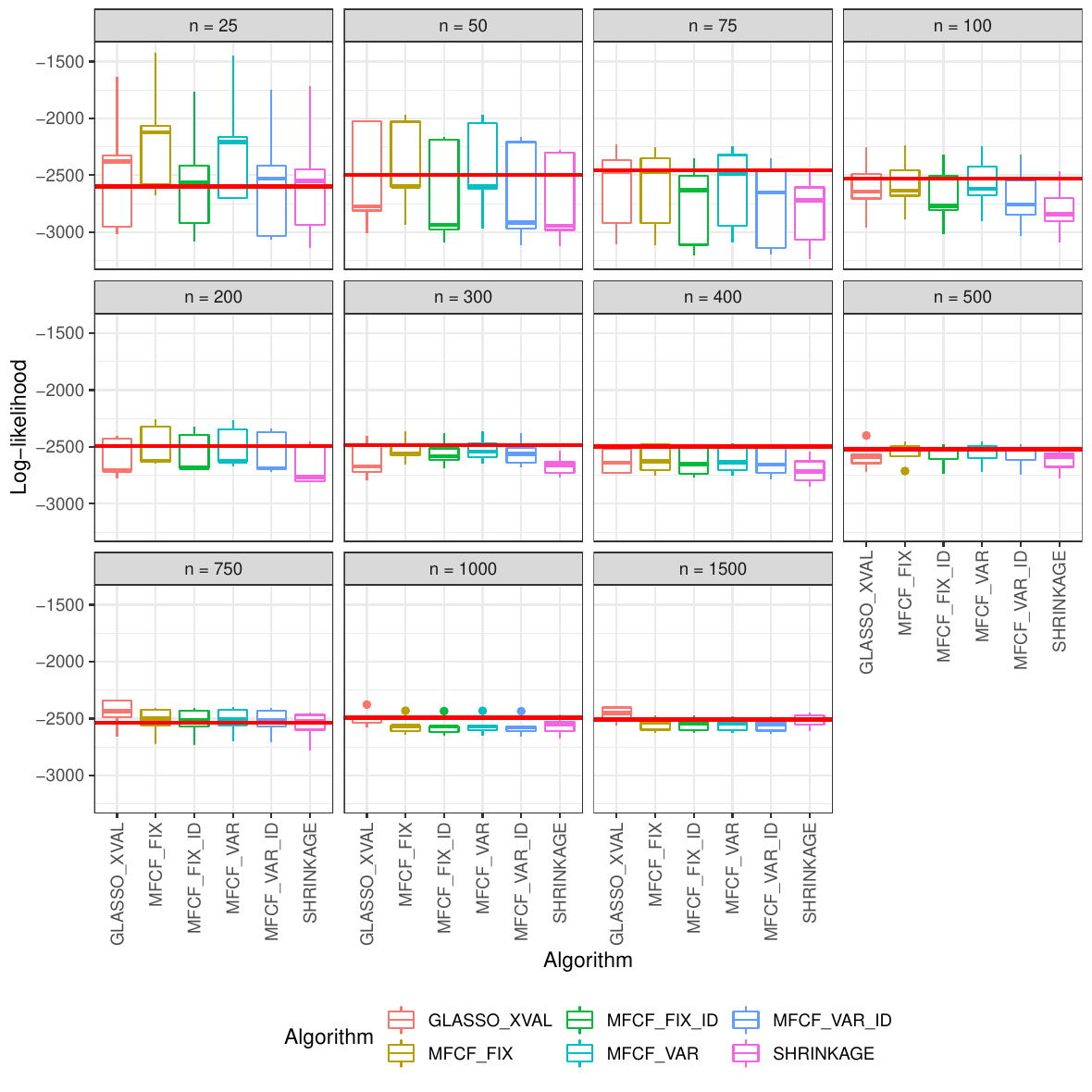}
  \caption{Log-likelihood of the algorithms on real data (stock returns) for different length of the series. The statistics is based on a total of 50 test sets (10 test sets for each of 5 different training / validation sets) per length of the time series.}
  \label{fig:ex:Real}
\end{figure}

\begin{figure}
\centering
\includegraphics[scale=0.8]{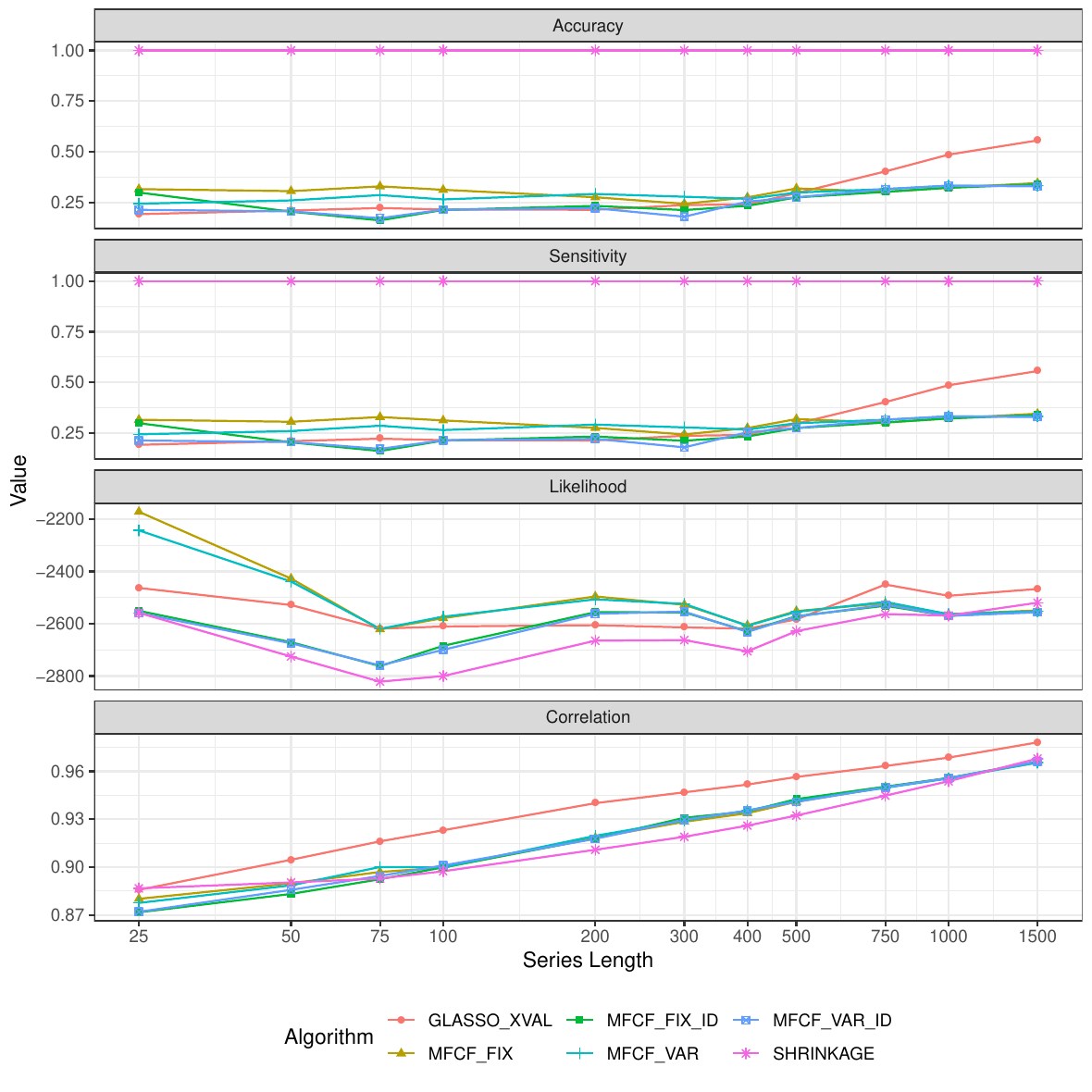}
 \caption{Performance of the algorithms on real data (stock returns) for different lengths of the series. The statistics is based on a total of 50 test sets (10 test sets for each of 5 different training / validation sets) per length of the time series.}
  \label{fig:ex:Real:1}
\end{figure}

\begin{figure}
\centering
\includegraphics[scale=0.8]{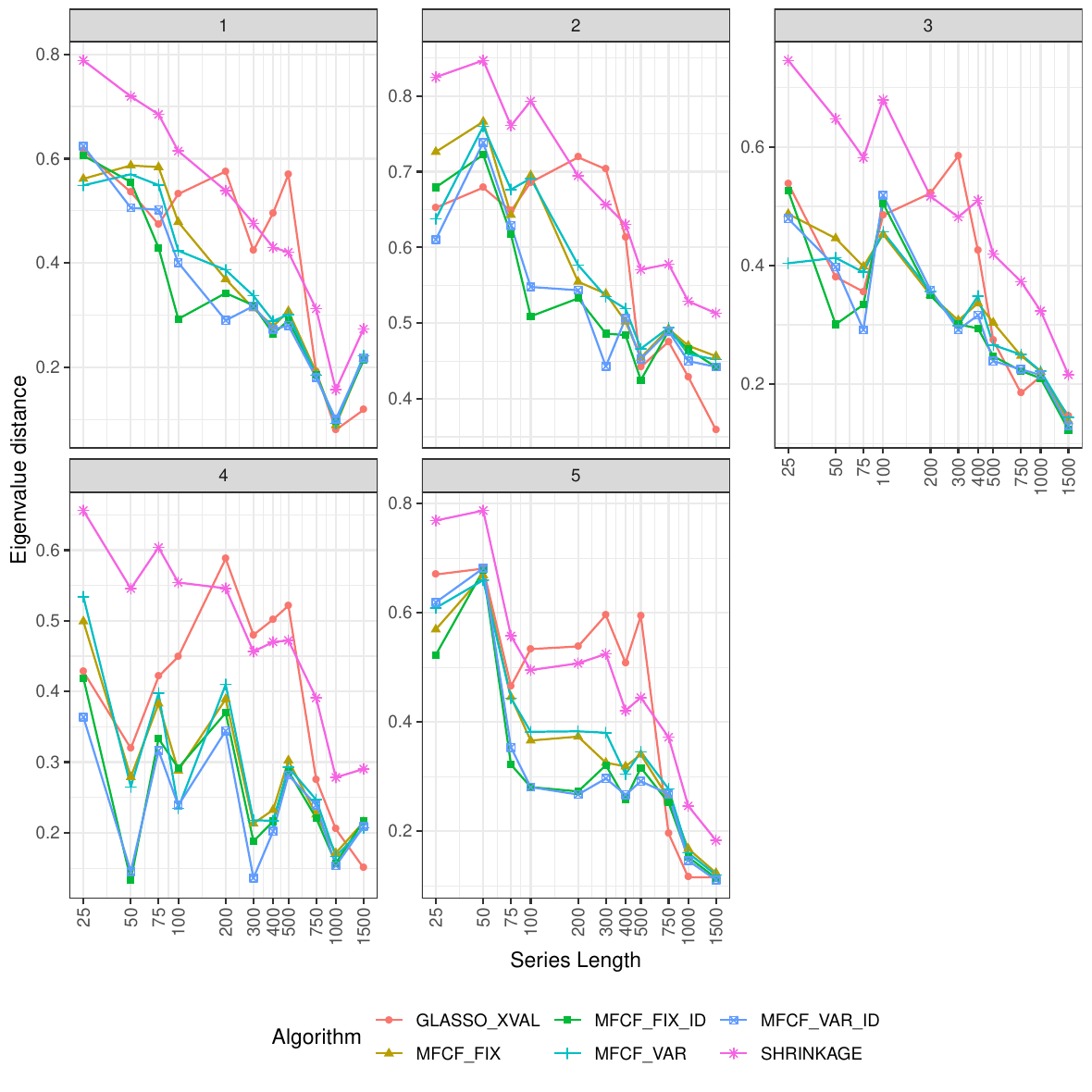}
 \caption{Eigenvalue distance for real data (stock returns) for 5 training sets. The five panels show the values at different time series lengths for 5 training  / validation datasets.}
  \label{fig:ex:Real:2}
\end{figure}

\begin{figure}
\centering
\includegraphics[scale=0.8]{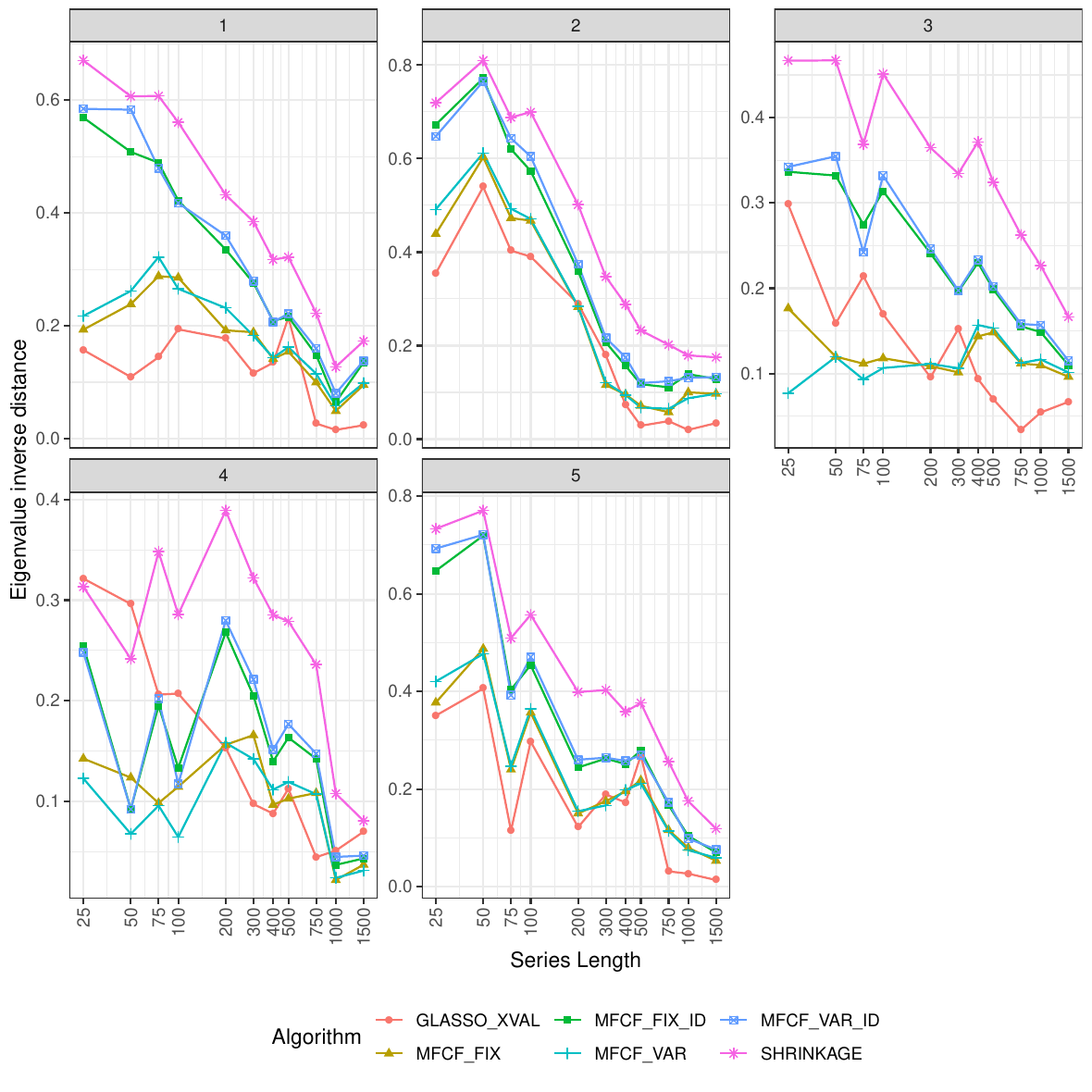}
 \caption{Inverse eigenvalue distance for real data (stock returns) for 5 training sets. The five panels show the values at different time series lengths for 5 training  / validation datasets.}
  \label{fig:ex:Real:3}
\end{figure} 

Figure \ref{fig:ex:Real:4} shows that, excepting for the shortest time series, the \mfcfvar{} algorithms almost always use the largest allowed clique size.

\begin{figure}
\centering
\includegraphics[scale=0.8]{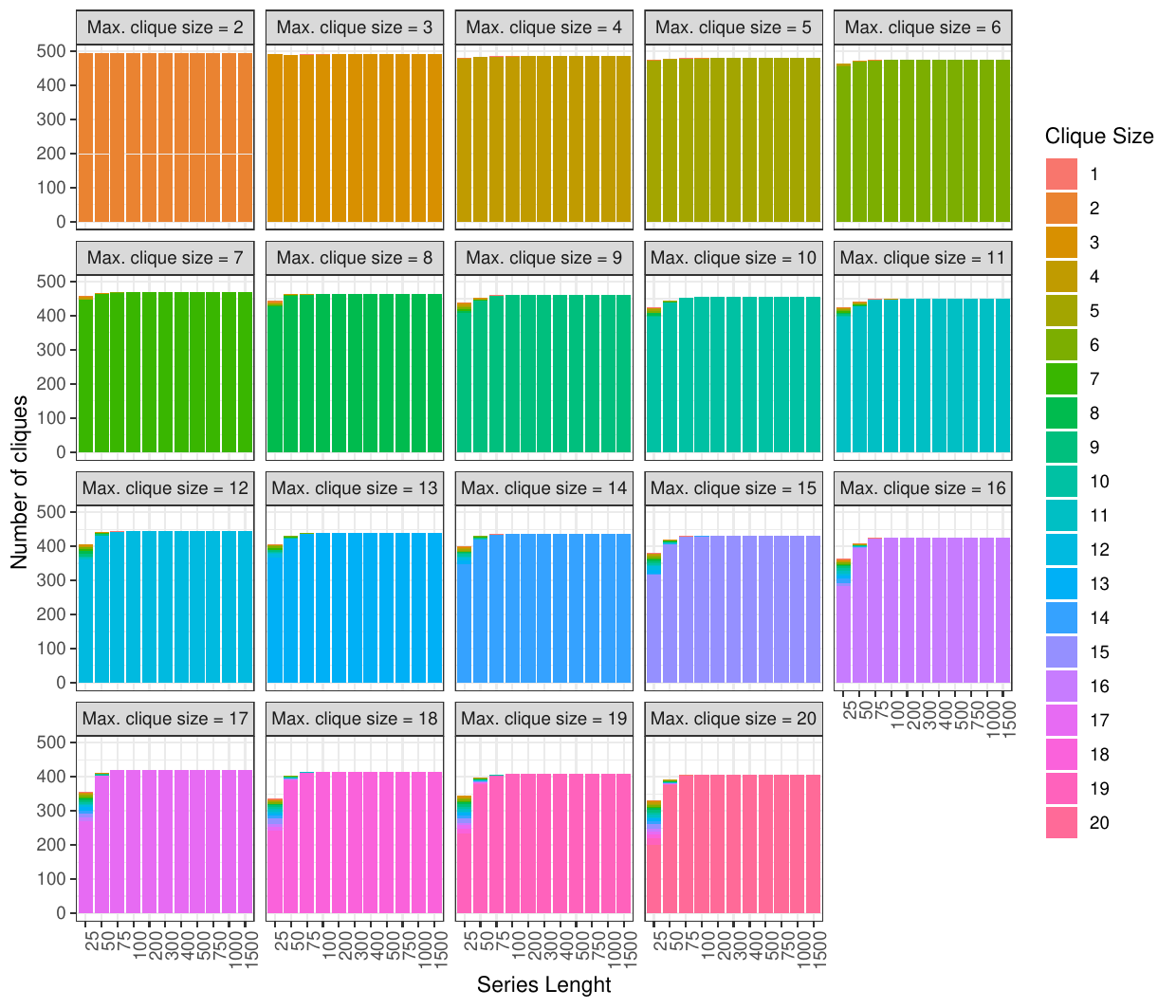}
 \caption{Composition of cliques from \mfcfvar{} for real data (stock returns). The statistics is based on a total of 5 different training / validation sets per length of the time series.}
  \label{fig:ex:Real:4}
\end{figure}

\end{document}